\documentclass[default]{sn-jnl}
\usepackage{amsmath,amssymb,amsthm,mathrsfs}
\usepackage{caption}
\usepackage{subcaption}
\usepackage{booktabs}
\usepackage{bm}
\usepackage{listings}
\usepackage{graphicx}
\usepackage{adjustbox}
\usepackage{enumitem}
\usepackage{pdfpages} %to include pdf
\usepackage{float}
\usepackage{multirow}
\usepackage{diagbox}
\usepackage{commath}
\usepackage{multicol,ragged2e,siunitx}

\usepackage{algorithm}
\usepackage{algpseudocode}
\algrenewcommand\algorithmicrequire{\textbf{Input:}}
\algrenewcommand\algorithmicensure{\textbf{Output:}}

\newcommand*{\argmin}{\operatornamewithlimits{argmin}\limits}

\jyear{2022}%

\theoremstyle{thmstyleone}%
 \newtheorem{theorem}{Theorem}
\theoremstyle{thmstylethree}%
\newtheorem{definition}{Definition}[section]

\begin{document}

% --------------------------------------------------------------
%                         Start here
% --------------------------------------------------------------
 
%\renewcommand{\qedsymbol}{\filledbox}
\title[Physics-informed deep collocation method]{Analysis of three dimensional potential problems in non-homogeneous media with physics-informed deep collocation method using material transfer learning and sensitivity analysis}

\author[1,2,4]{\fnm{Hongwei}  \sur{Guo}}\email{hongwei.guo@iop.uni-hannover.de}

\author[1,2,4]{\fnm{Xiaoying} \sur{Zhuang}}\email{zhuang@iop.uni-hannover.de}

\author[5]{\fnm{Pengwan} \sur{Chen}}\email{pwchen@bit.edu.cn}

\author[3]{\fnm{Naif} \sur{Alajlan}}\email{najlan@ksu.edu.sa}

\author*[3]{\fnm{Timon} \sur{Rabczuk}}\email{timon.rabczuk@uni-weimar.de}

\affil[1]{\orgdiv{Department of Geotechnical Engineering},\orgname{Tongji University}, \orgaddress{\street{1239 Siping Road}, \city{Shanghai}, \postcode{200092}, \country{P.R. China}}}

\affil[2]{\orgdiv{Key Laboratory of Geotechnical and Underground Engineering of Ministry of Education},\orgname{Tongji University}, \orgaddress{\street{1239 Siping Road}, \city{Shanghai}, \postcode{200092}, \country{P.R. China}}}

\affil*[3]{\orgdiv{ALISR Laboratory, College of Computer and Information Sciences},\orgname{King Saud University}, \orgaddress{\street{P. O. Box 51178}, \city{Riyadh} \postcode{11543}, \country{Saudi Arabia}}}

\affil[4]{\orgdiv{Computational Science and Simulation Technology, Institute of Photonics},\orgname{Leibniz Universität Hannover}, \orgaddress{\street{Appelstr. 11}, \city{Hannover}, \postcode{30167}, \country{Germany}}}

\affil[5]{\orgdiv{State Key Laboratory of Explosion Science and Technology},\orgname{Beijing Institute of Technology}, \orgaddress{\street{No. 5, South Street, Zhongguancun, Haidian District}, \city{Beijing}, \postcode{100081}, \country{P.R. China}}}

\abstract{In this work, we present a deep collocation method (DCM) for three dimensional potential problems in non-homogeneous media. This approach utilizes a physics-informed neural network with material transfer learning reducing the solution of the non-homogeneous partial differential equations to an optimization problem. We tested different configurations of the physics-informed neural network including smooth activation functions, sampling methods for collocation points generation and combined optimizers. A material transfer learning technique is utilized for non-homogeneous media with different material gradations and parameters, which enhance the generality and robustness of the proposed method. In order to identify the most influential parameters of the network configuration, we carried out a global sensitivity analysis. Finally, we provide a convergence proof of our DCM. The approach is validated through several benchmark problems, also testing different material variations.}


\maketitle

\noindent\fbox{%
\hfill\parbox{\dimexpr\textwidth-15pt}{%
\vspace{-\topskip}\small
\newcommand{\entry}[3][\>]{#2 #1 \parbox[t]{.4\textwidth}{\RaggedRight
#3\strut\par}\\}%
\begin{multicols}{2}
  \begin{tabbing}
    \textbf{\large Nomenclature}\\[2ex]
    \entry[\quad\=]{$k\left(\textit{\textbf{x}}\right)$}{Position oriented material function}
    \entry{$\phi$}{Potential function}
    \entry{$q$}{Flux of potential field}
    \entry{$\bm{n}$}{Unit normal vector to a surface}
    \entry{$w_{jk}^{l}$}{Weight between neuron $k$ in hidden layer $l-1$ and neuron $j$ in hidden layer $l$}
    \entry{$b^l_j$}{Bias of neuron $j$ in layer $l$}
    \entry{$\sigma$}{Activation function}
    \entry{$\theta$}{Hyperparameters including all weights and biases}
    \entry{$Loss\left(\theta\right)$}{Loss function for training}
    \entry{$\textit{\textbf{x}}\,_\Omega$}{Collocation points to discretize the physical domain}
    \entry{$\textit{\textbf{x}}\,_\Gamma$}{Collocation points to discretize the boundaries}
    \entry{$MSE$}{\quad Mean square error loss form}
    \entry{$\phi^h (\bm{x};\theta)$}{\quad \quad Potential function approximated by Neural networks}
    \entry{$G\left(\textit{\textbf{x}}\right)$}{Governing equation}
    \entry{$\tilde{\phi}, \tilde{ \textit{q}}$}{Potential field and flux prescribed at boundaries}
	\entry{$\eta_i$}{Learning rate}
	\entry{$EE_i$}{Elementary effect for each input factor}	
	\entry{$\mu_{i}^{*}$}{Mean of the distribution of the elementary effects of each input}	
	\entry{$\sigma_{i}$}{Standard deviation of the distribution of the elementary effects of each input}	
	\entry{$\Lambda_j$}{Spectral curve of the Fourier progression}	
	\entry{$S_{i}^{FAST}$}{First order FAST sensitivity indices}	
	\entry{$S_{T_{i}}$}{\quad Total order FAST sensitivity indices}	
	\entry{$e$}{Relative error to measure the model accuracy}		
	\entry{$E_a$}{Analytical solution}			
	\entry{$E_a$}{Predicted solution}	
	\entry{$ \left \| \cdot \right \|$}{$L_2$-norm}		
  \end{tabbing}
\end{multicols}
}%
\hfill}

\section{Introduction}
\label{section 1:Introduction}
Recent years have witnessed the vast growing application of neural networks in physics, this is partly due to the fact that by training the neural network, high-dimensional raw data can be converted to low-dimensional codes \cite{hinton2006reducing}, and thus the high-dimensional PDEs can be directly solved using a `meshfree' deep learning algorithm, which improves computing efficiency and reduces the complexity of problems. The deep learning method deploys a deep neural network architecture with nonlinear activation functions that introduces the nonlinearity that the system as a whole needs for learning nonlinear patterns. This lends some credence to the application of a physics informed machine learning method in discovering the physics behind the potential problems in non-homogeneous media, which is a wide range of problems in physics and engineering. 

The current wave of deep learning started around 2006, when Hinton et al. \cite{hinton2006fast, bengio2007greedy} introduced deep belief nets and unsupervised learning procedures that could create layers of feature detectors without needs of labelled data. Equipped with deep learning model, information can be extracted from complicated raw input data with multiple levels of abstraction through a layer-by-layer process \cite{goodfellow2016deep}. Various variants such as multilayer perceptron (MLP), convolutional neural networks (CNN) and recurrent/recursive neural networks (RNN) \cite{patterson2017deep} have been developed and applied to e.g. image processing \cite{yang2018visually,kermany2018identifying}, object detection \cite{ouyang2015deepid,zhao2019object}, speech recognition \cite{amodei2016deep,nassif2019speech}, biology \cite{yue2018deep,ching2018opportunities} and even finance \cite{heaton2017deep,fischer2018deep}. Over the past decade it has been widely used in applications due to high performance demonstrated. Deep learning can learn features from data automatically, and the features can be used to get the approximation of solutions to differential equations \cite{gyryamachine}, which cast light on the possibility of using deep learning as functional approximators.

Artificial neural networks (ANN) stands at the center of the deep learning revolution, it can be traced back to the 1940's \cite{mcculloch1943logical} but they became especially popular in the past few decades due to the vast development in computational power and sophisticated machine learning algorithms such as backpropagation technique and advances in deep neural networks. Due to the simplicity and feasibility of ANNs to deal with nonlinear and multi-dimensional problems, they were applied in inference and identification by data scientists \cite{dias2004artificial}. They were also adopted to solve partial differential equations (PDEs) \cite{lagaris1998artificial,lagaris2000neural,mcfall2009artificial} but shallow ANNs are unable to learn the complex nonlinear patterns effectively. With improved theories incorporating unsupervised pre-training, stacks of auto-encoder variants, and deep belief nets, deep learning with enhanced learning abilities can also sever as an interesting alternative to classical methods such as FEM. 

According to the universal approximation theorem \cite{FUNAHASHI1989183,HORNIK1989359}, any continuous function can be approximated by a feedforward neural network with one single hidden layer. However, the number of neurons of the hidden layer tends to increase exponentially with increasing complexity and non-linearity of a model. Recent studies show that DNNs render better approximations for nonlinear functions \cite{mhaskar2016deep}. Some researchers employed deep learning for the solution of PDEs. E et al. developed a deep learning-based numerical method for high-dimensional parabolic PDEs and back-forward stochastic differential equations \cite{weinan2017deep,han2018solving}. Raissi et al. \cite{RAISSI2019686} introduced physics-informed neural networks for supervised learning of nonlinear partial differential equations. Beck et al. \cite{Beck_2019}  employed deep learning to solve nonlinear stochastic differential equations and Kolmogorov equations. Sirignano and Spiliopoulos \cite{sirignano2018dgm} provided a theoretical proof for deep neural networks as PDE approximators, and concluded that it converged as the number of hidden layers tend to infinity. Karniadakis et al. give an overview of physics-informed machine learning and introduced application of physics-informed neural networks in various fields such as fluid mechanics \cite{karniadakis2021physics}. For modelling problems in solid mechanics, we presented a Deep Collocation Method (DCM) \cite{anitescu2019artificial,guo2019deep}. Based on DCM, we have proposed a stochastic deep collocation method with neural architecture search strategy for stochastic flow analysis in heterogeneous media and found that physics-informed deep learning model can account for stochastic disturbance/uncertainties efficiently and stably \cite{guo2022stochastic}. Rather than the strong form of the boundary value problem with higher order derivatives, we presented a Deep Energy Method (DEM) \cite{samaniego2020energy, nguyen2020deep, goswami2020transfer, zhuang2021deep} by constraining total potential energy in the loss instead of a BVP. With the maturity of physics-informed deep learning model, it can be further applied to practical engineering problems. 

The problems of potential represent a category of physical and engineering problems. For some physical parameters in potential problems, for example, heat conductivity, permeability, permittivity, resistivity, magnetic permeability, tends to have a spatial distribution, and they can vary with respect to one or more coordinates. In order to deal with these problems, the non-homogeneous problems is translated into homogeneous problems with some classes of material variations. The steady state heat conduction analysis of FGMs analysis is a representative of potential problems. Due to the inherent mathematical difficulties, closed-form solutions exist in a few simple cases. Some traditional powerful methods, such as the finite element method(FEM), the boundary element method(BEM), and the method of fundamental solutions (MFS) and the dual reciprocity method (DRM) were used to solve the potential problems \cite{qu2015solutions,alves2005new}. The `meshfree' physics-informed neural networks offered a novel and robust approach in discovering the nonlinear patterns behind the potential patterns, especially for higher dimensions.  

In practice the learning ability of deep neural networks can strongly rely on the neural network configurations and the optimization algorithms, such as the activation forms, number of neurons and layers, weight initialization methods, number of iterations, and so on. In this paper, we therefore compare  different parameters to offer suggestions on the choice of a favourable configuration for the physics-informed neural network. Moreover, to increase the generality and robustness of the physics-informed deep learning based collocation method, the material transfer learning technique is integrated in the model, which will reduce the computation costs for different material variation types and help to improve the numerical results. Further, to unveil those influencing parameters for the proposed model, a global sensitive analysis is supplemented in the paper, which will be instructive for setting up physics-informed neural networks.

The paper is organised as follows: 
First, the three dimensional potential problem with in-homogeneous media is presented. Then we introduce the physics-informed deep learning based collocation method, which includes the neural network architecture, activation functions, sampling methods, a convergence proof, the material transfer learning and sensitivity analysis. Subsequently, a sufficient survey of numerical examples are presented, which investigated different neural network configurations, material transfer learning and model sensitivity analysis. Finally, the effectiveness of the deep learning method is demonstrated for solving three dimensional potential problems in non-homogeneous media.

\section{The governing equation for 3D problems of potential}
\label{pde3Dp} 
The general partial differential equation for potential function $\phi$ defined on a region $\Omega$ bounded by surface $\tau$, with an outward normal $\bm{n}$, can be written as:
\begin{equation}\label{eq:gvquation}
(k(\bm{x})\phi_{,i})_{,i}=k(\bm{x})\phi_{,ii}+k_{,i}(\bm{x})\phi_{,i}=0
\end{equation}
where $k$ is a position-oriented material function. Equation~\eqref{eq:gvquation} is the field equation for a wide range of problems in physics and engineering such as heat transfer, incompressible flow, gravity field, shaft torsion, electrostatics and magnetostatics, some of which are shown in Table~\ref{tab:Table1} \cite{paris1997boundary}.

\begin{table}[!h] % Add the following just after the closing bracket on this line to specify a position for the table on the page: [!htb], [t], [b] or [p] - these mean: here, top, bottom and on a separate page, respectively
\captionsetup{width=0.9\columnwidth}
\caption{Problems belong to the category of problems of potential} % Table caption, can be commented out if no caption is required
\vspace{-0.1cm}
\centering % Centres the table on the page, comment out to left-justify
\resizebox{1\columnwidth}{!}{%
\begin{tabular}{l c c c c} % The final bracket specifies the number of columns in the table along with left and right borders which are specified using vertical pipes (|); each column can be left, right or center-justified using l, r or c. Columns will widen to hold the content in them by default, to specify a precise width, use p{width}, e.g. p{5cm}
\toprule % Top horizontal line
\toprule % Top horizontal line
& & & \multicolumn{2}{c}{\textbf{Boundary condition}} \\ % Amalgamating several columns into one cell is done 
\cmidrule(l){4-5} % Horizontal line spanning less than the full width of the table - you can add (r) or (l) just before the opening curly bracket to shorten the rule on the left or right side
Problems & Scalar function $\phi$& $k(\bm{x})$& Dirichlet & Neumann\\ % Column names row
\midrule % In-table horizontal line
Heat transfer &Temperature T  & Thermal conductivity (k)&$T=\bar{ T}$ & Heat flow $q=-k\frac{\partial T}{\partial n}$\\ % Content row 1
\midrule
Ground water flow &Hydraulic head H&Permeability (k)&$H=\bar H$&Velocity flow $q=-k\frac{\partial H}{\partial n}$\\ % Content row 2
\midrule
Electrostatic& Electrostatic potential V&Permittivity ($\varepsilon$)&$V=\bar V$&Electric flow $q=-k\frac{\partial V}{\partial n}$\\ % Content row 3
\midrule
Electric conduction&Electropotential E&Resistivity (k)&$E=\bar E$&Electric current $q=-k\frac{\partial E}{\partial n}$\\ % Content row 4
\midrule
Magnetostatic&Magnetic potential M&Magnetic permeability ($\mu$)&$M=\bar M$&Magnetic flux density $q=-k\frac{\partial M}{\partial n}$\\ % Content row 4
\bottomrule % Bottom horizontal line
\end{tabular}
}
\label{tab:Table1} % A label for referencing this table elsewhere, references are used in text as \ref{label}
\end{table}

The Dirichlet $\tau_D$ and Neumann boundary $\tau_N$ conditions are given as:
\begin{equation}\label{eq:boundarycondition}
\begin{aligned}
\phi(\bm{x},t)=\bar{\phi}, \bm{x} \in \tau_D,\\
-k(\bm{x})\frac{\partial \phi(\bm{x},t)}{\partial \bm{n}}=\bar{q}, \bm{x} \in \tau_N
\end{aligned}
\end{equation}
where $\bm{n}$ is the unit outward normal to $\tau_N$.   The material properties of functionally graded materials (FGMs) vary gradually in space. Classical variations of $k(\bm{x})$ take the form $k(\bm{x})=k_0f(\bm{x})$, $k_0$ denoting a reference value and $f(\bm{x})$ is the material property variation function. Among the most common variation functions are the quadratic, exponential and trigonometric:

\begin{equation}\label{eq:materialvar}
\begin{aligned}
& \textrm{Parabolic}:\quad f(\bm{x})=(a_1+a_2\bm{x})^{2}\\
& \textrm{Exponential}:\quad f(\bm{x})=(a_1e^{\beta \bm{x}}+a_2e^{-\beta \bm{x}})^{2}\\
& \textrm{Trigonometric}:\quad f(\bm{x})=(a_1cos\beta \bm{x}+a_2sin\beta \bm{x})^{2}
\end{aligned}
\end{equation}

The governing equations for different material variations in the $z-$direction are summarized in Table \ref{tab:Table2}:
\begin{table}[!h] 
\captionsetup{width=0.9\columnwidth}
\caption{Governing equation deduced by considering various forms of $k(\bm{x})$} % Table caption, can be commented out if no caption is required
\vspace{-0.1cm}
\centering % Centres the table on the page, comment out to left-justify
\resizebox{1\columnwidth}{!}{%
\begin{tabular}{c| c} % The final bracket specifies the number of columns in the table along with left and right borders which are specified using vertical pipes (|); each column can be left, right or center-justified using l, r or c. Columns will widen to hold the content in them by default, to specify a precise width, use p{width}, e.g. p{5cm}
\toprule % Top horizontal line
\toprule % Top horizontal line
$k(\bm{x})$& Differential equation\\ % Column names row
\midrule % In-table horizontal line
$k_0(a_1+a_2z)^{2}$ & $(a_1+a_2z)\nabla^2\phi+2a_2\phi_z=0$\\ % Content row 1
\midrule
$k_0(a_1e^{\beta z}+a_2e^{-\beta z})^{2}$&$(a_1e^{\beta z}+a_2e^{-\beta z})^{2}\nabla^2\phi+2\beta(a_1^2e^{2\beta z}+a_2^2e^{-2\beta z})\phi_z=0$\\ % Content row 2
\midrule
$k_0(a_1cos\beta z+a_2sin\beta z)^{2}$&$(a_1cos\beta z+a_2sin\beta z)^{2}\nabla^2\phi+2\beta(0.5(a^2_2-a^1_1)sin2\beta z+a_1a_2cos2\beta z)\phi_z=0$\\ % Content row 3
\bottomrule % Bottom horizontal line
\end{tabular}
}
\label{tab:Table2} % A label for referencing this table elsewhere, references are used in text as \ref{label}
\end{table}

\section{Physics-informed deep learning based collocation method}
\label{section 3: solving problems of potential}

\subsection{Feed forward neural network}
The basic architecture of a fully connected feedforward neural network is shown in Figure~\ref{Figure2:network}. It comprises multiple layers: an input layer, one or more hidden layers and an output layer. Each layer consists of one or more nodes called neurons, shown in Figure~\ref{Figure2:network} by the small colored circles. For an interconnected structure, every two neurons in neighboring layers have a connection, where the weights between neuron $k$ in hidden layer $l-1$ and neuron $j$ in hidden layer $l$ is denoted by $w_{jk}^{l}$, see Figure~\ref{Figure2:network}. No connection exists among neurons in the same layer as well as in the non-neighboring layers. Input data, defined from $x_{1}$ to $x_{N}$, flows through this neural network via connections between neurons, starting from the input layer, through the hidden layers $l-1$, $l$, to the output layer, which eventually outputs data from $y_{1}$ to $y_{M}$.

\begin{figure}[!htb]
	\captionsetup{width=0.9\columnwidth}
	\includegraphics[height=7cm]{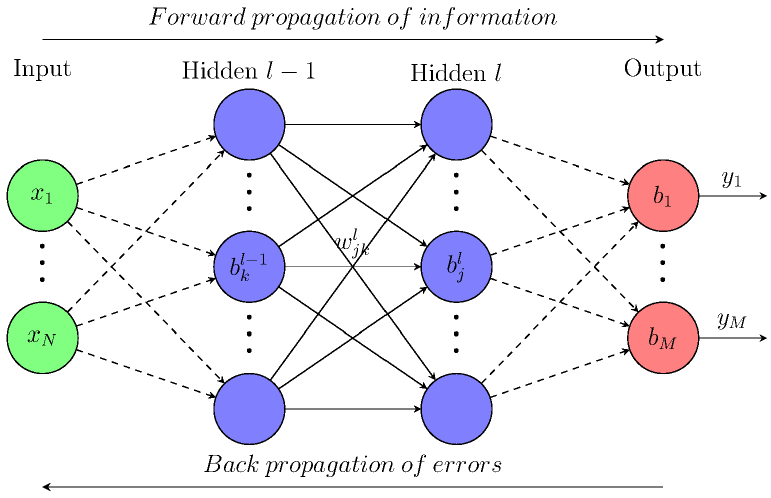}
	\centering
	\caption{Architecture of a fully connected feedforward back-propagation neural network.}
	\label{Figure2:network}
\end{figure}

The activation function is defined for an output of each neuron in order to introduce a non-linearity into the neural network and make the back-propagation possible where gradients are supplied along with an error to update weights and biases. The activation function in layer $l$ will be denoted by $\sigma$ here. 

There are many activation functions $\sigma$ proposed for inference and identification with neural networks such as sigmoids function \cite{dhingraactivation}, hyperbolic tangent function $ \left( Tanh \right)$ \cite{dhingraactivation}, Rectified linear units $ \left( Relu \right)$, to name a few. And some recent smooth activation functions, such as Swish \cite{misra2019mish}, LeCuns Tanh \cite{dhingraactivation}, Bipolar sigmoid \cite{dhingraactivation}, Mish \cite{misra2019mish}, Arctan \cite{zhang2018efficient}, listed in  Appendix~\ref{appendix b} Table~\ref{tab:Tableac} have been studied and compared in the numerical example section. All selected activation functions must be smooth enough in order to avoid gradient vanishing during backpropagation, since the governing equation is introduced in the loss which includes the second order derivatives of the field variable.
%%%
Afterward, the value on each neuron in the hidden layers and output layer can be yielded by adding the weighted sum of values of output values from previous layer to basis. An intermediate quantity for neuron $j$ on hidden layer $l$ is defined as
\begin{equation}
a^l_j = \sum_k w^l_{jk}y^{l-1}_k + b^l_j,
\label{activation}
\end{equation}
and its output is given by the activation of the above weighted input
\begin{equation}
y^{l}_j =\sigma \left( a^l_j \right)=\sigma \left(\sum_k w^l_{jk}y^{l-1}_k + b^l_j \right),
\label{wi} 
\end{equation}
where $y^{l-1}_k$ is the output from previous layer. 

Based on the previous derivation and description, we can draw a definition which will be used in Section \ref{dcm}: 
\theoremstyle{definition}
\begin{definition}{(Feedforward Neural Network)}
A generalized neural networks with activation can be written in a tuple form $\left((f_1,\sigma_1),...,(f_n,\sigma_n)\right)$, $f_i$ referring to an affine-line function $(f_i = W_i\textit{\textbf{x}}+b_i)$ that mapps $R^{i-1} \rightarrow R^{i}$. The tuple formed neural network in all defines a continuous bounded function mapping $R^{d}$ to $R^{n}$:
\begin{equation}
FNN: \mathbb{R}^d \to \mathbb{R}^n, \; \textrm{with}\; \;   F^n\left(\textit{\textbf{x}};\bm{\theta}\right) = \sigma_n\circ f_n \circ\cdots  \circ \sigma_1 \circ f_1
\end{equation}
where $d$ indicates the dimension of the inputs, $n$ the number of field variables, $\bm{\theta}=\{ \bm{W};\bm{b} \}$ consisting of hyperparameters such as weights and biases and $\circ$ denotes the element-wise operator.
\end{definition}
The universal approximation theorem \cite{FUNAHASHI1989183,HORNIK1989359} states that this continuous bounded function $F$ with nonlinear activation $\sigma$ can be adopted to capture the nonlinear property of the system, in our case the potential problem. With this definition, we can define \cite{hornik1991approximation}:
\begin{theorem}\label{theorem1}
If $\sigma^i \in C^m(R^i)$ is non-constant and bounded, then $F^n$ is uniformly m-$dense$ in $C^m(R^n)$.
\end{theorem}

\subsection{Backpropagation}
Backpropagation  $\left( backward\;propagation \right)$ can be used to train multilayer feed-forward networks by calculating the gradient of a loss function and finding the minimum value of the loss function. The backward (output-to-input) flow determine how to adjust each weight as shown in Figure~\ref{backpropogation}. 

\begin{figure}[!htb]
	\captionsetup{width=0.9\columnwidth}
	\centering\includegraphics[height=8cm]{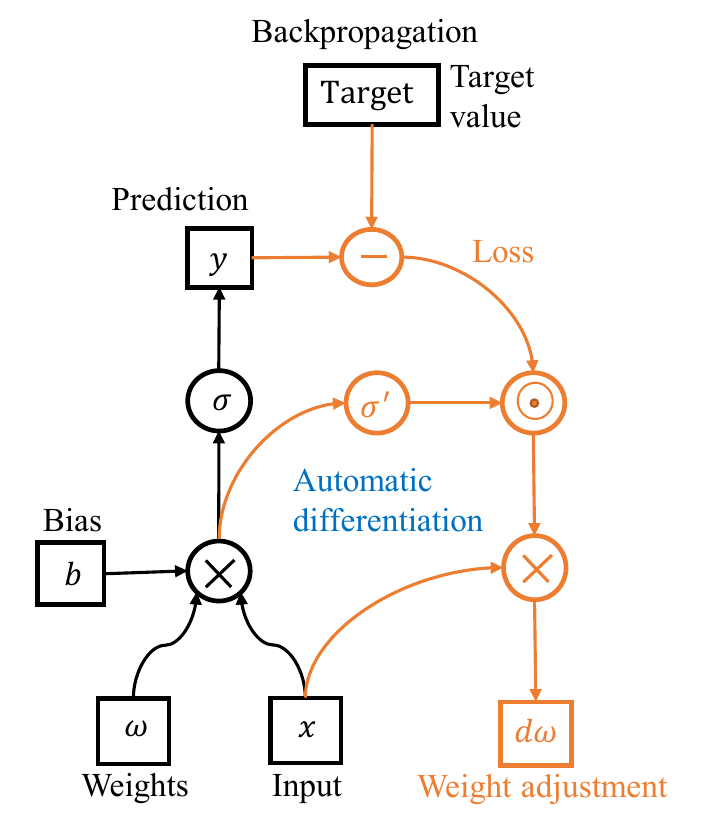}   
	\caption{The ‘compute graph’ for the feedforward neural network}
	\label{backpropogation}
\end{figure}

Backpropagation is based on the chain rule, which is used to calculate the derivative of loss function with regard to the weight in the network. The governing equation in our problem requires the second partial derivatives of the potential function $\phi\left(\textit{\textbf{x}}\right)$. In order to find the weights and biases, a loss function $\textit{\textbf{Loss}}\left(\textit{f},\theta\right)$ is defined. The backpropagation algorithm for computing the gradient of this loss function $\textit{\textbf{Loss}}\left(f,\theta\right)$, the weight coefficients $\textit{\textbf{w}}$ and thresholds of neurons $\textit{\textbf{b}}$ can be written as follow:

\begin{algorithm}[!htb]
\caption{Backpropagation algorithm}
\label{algorithm1}
\textbf{Input:}  Create random sampling points $\bm{x}=(x_1,...,x_{m})^T$inside the physical domain.
\textbf{Output:} The gradient of the loss function.
  \textbf{Set}: the corresponding activation $a^1 = x^i$, $i=1,...,m$ for the input layer
  
  \textbf{Feedforward}: 
  \begin{algorithmic}
  \For{$l$\enspace from 2 to $L$}
    {compute $z^l=w^la^{l-1}+b^l$ and $a^l=\sigma \left(z^l\right)$
   } 
    \EndFor
    \end{algorithmic}
    
    \textbf{Output error}: Compute the output error $\delta^L = \nabla_a\textit{\textbf{Loss}}\odot {\sigma}' \left(z^L\right)$
    
   \textbf{Backpropagate error}: 
    \begin{algorithmic}
  \For{$l$\enspace from L-1 to $2$}
    {compute $\delta^l=\left(\left(w^{l+1}\right)^T\delta^{L+1}\right)\odot{\sigma}' \left(z^l\right)$
   } 
    \EndFor
   \end{algorithmic} 
   \textbf{Output gradient}: The gradient of the loss function $\frac{\partial\bm{Loss}}{\partial w^l_{jk}}=a^{l-1}_k\delta^l_j$ and $\frac{\partial\bm{Loss}}{\partial b^l_{j}}=\delta^l_j$    

  \end{algorithm}

\subsection{Physics-informed deep learning based collocation method}
\label{dcm}

\begin{figure}[!htb]
	\captionsetup{width=0.9\columnwidth}
	\includegraphics[height=7cm]{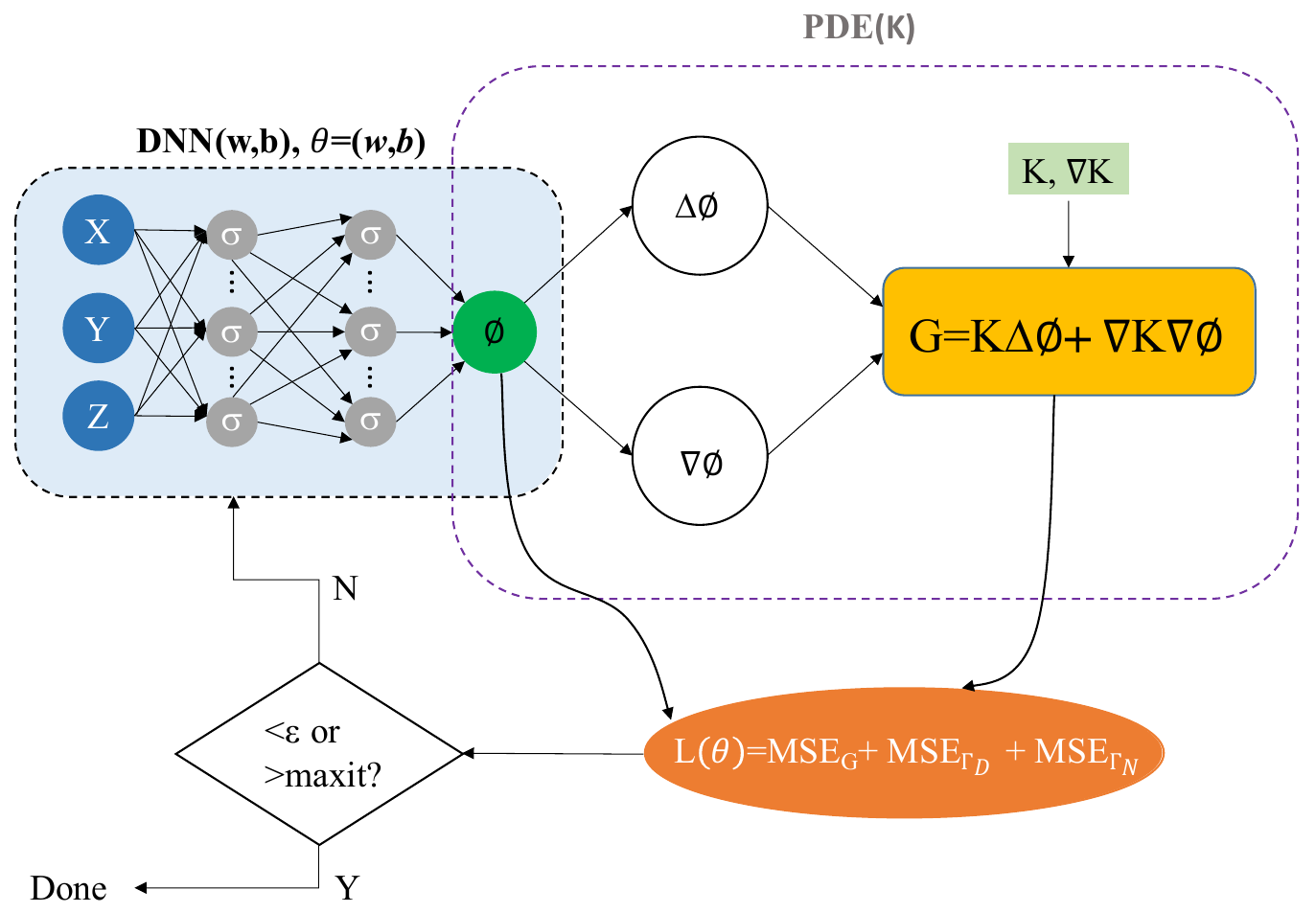}
	\centering
	\caption{Schematic diagram of physics-informed neural networks}
	\label{fig:PINN}
\end{figure}

To train the network, we place collocation points in the physical domain and at the boundaries denoted by $\textit{\textbf{x}}\,_\Omega=(x_1,...,x_{N_\Omega})^T$ and $\textit{\textbf{x}}\,_\Gamma(x_1,...,x_{N_\Gamma})^T$, respectively. Then the potential function $\phi$ is approximated with the aforementioned deep feedforward neural network $\phi^h (\bm{x};\bm{\theta})$. Thus, a loss function related to the underlying BVP is constructed. Substituting $\phi^h \left(\bm{x}\,_\Omega;\bm{\theta}\right)$ into governing equation, we obtain
\begin{equation}
G\left(\textit{\textbf{x}}\,_\Omega;\bm{\theta}\right)=k(\bm{x})\phi_{,ii}^{h}\left(\textit{\textbf{x}}\,_\Omega;\bm{\theta}\right)+k_{,i}(\bm{x})\phi^{h}_{,i}\left(\textit{\textbf{x}}\,_\Omega;\bm{\theta}\right),
\end{equation}
which results in a physics-informed deep neural network $G\left(\textit{\textbf{x}}\,_\Omega;\bm{\theta}\right)$.  The boundary conditions illustrated in Section 2 can also be expressed by the neural network approximation $\phi^h \left(\textit{\textbf{x}}\,_\Gamma;\bm{\theta}\right)$ as: On $\Gamma_{D}$, we have
\begin{equation}
  \phi^h \left(\textit{\textbf{x}}\,_{\Gamma_D};\bm{\theta}\right)=\tilde{\phi},
\end{equation} 

\noindent On $\Gamma_{N}$,
\begin{equation}
  \textit{q}^h \left(\textit{\textbf{x}}\,_{\Gamma_N};\bm{\theta}\right) = \tilde{ \textit{q}}. 
\label{bd2}
\end{equation} 
where $\textit{q}^h \left(\textit{\textbf{x}}\,_{\Gamma_N};\bm{\theta}\right) $ can be obtained from Equation~\eqref{eq:boundarycondition} by combing $\phi^h \left(\textit{\textbf{x}}\,_{\Gamma_N};\bm{\theta}\right)$. Note the induced physics-informed neural network $G\left(\textit{\textbf{x}};\bm{\theta}\right)$, $q\left(\textit{\textbf{x}};\bm{\theta}\right)$ share the same parameters as $\phi^h \left(\textit{\textbf{x}};\bm{\theta}\right)$. Considering the generated collocation points in domain and on boundaries, they can all be learned by minimizing the mean square error loss function \cite{raissi2017physics}:
\begin{equation}\label{lossform}
Loss\left(\bm{\theta}\right)=MSE=MSE_{G}+MSE_{\Gamma_{D}}+MSE_{\Gamma_{N}},
\end{equation}
with
\begin{equation}
\begin{aligned}
&MSE_{G}=\frac{1}{N_d}\sum_{i=1}^{N_d}\begin{Vmatrix}
G\left(\textit{\textbf{x}}\,_\Omega;\bm{\theta}\right)
\end{Vmatrix}^2=\frac{1}{N_\Omega}\sum_{i=1}^{N_\Omega}\begin{Vmatrix}
k(\bm{x}_\Omega)\phi_{,ii}^{h}\left(\textit{\textbf{x}}\,_\Omega;\bm{\theta}\right)+k_{,i}(\bm{x}_\Omega)\phi^{h}_{,i}\left  (\textit{\textbf{x}}\,_\Omega;\bm{\theta}\right)
\end{Vmatrix}^2,\\
&MSE_{\Gamma_{D}}=\frac{1}{N_{\Gamma_D}}\sum_{i=1}^{N_{\Gamma_D}}\begin{Vmatrix}
 \phi^h \left(\textit{\textbf{x}}\,_{\Gamma_D};\bm{\theta}\right)-\bar{\phi}
\end{Vmatrix}^2,\\ 
&MSE_{\Gamma_{N}}=\frac{1}{N_{\Gamma_N}}\sum_{i=1}^{N_{\Gamma_N}}\begin{Vmatrix}
q\left(\textit{\textbf{x}}\,_{\Gamma_N};\bm{\theta}\right)-\bar{q}
\end{Vmatrix}^2=\frac{1}{N_{\Gamma_N}}\sum_{i=1}^{N_{\Gamma_N}}\begin{Vmatrix}
-k(\bm{x}_{\Gamma_N})\frac{\partial \phi\left(\textit{\textbf{x}}_{\Gamma_N};\bm{\theta}\right)}{\partial n}-\bar{q}
\end{Vmatrix}^2.
\end{aligned}
\end{equation}
where $x\,_\Omega \in {R^N} $, $\bm{\theta} \in {R^K}$ are the neural network parameters. $Loss\left(\bm{\theta}\right)=
0$, $\phi^h \left(\textit{\textbf{x}};\bm{\theta}\right)$ is then a solution to potential function. Here, the defined loss function measures how well the approximation satisfies the physical law (governing equation), boundaries conditions. Our goal is to find a set of parameters $\bm{\theta}$ that the  approximated potential $\phi^h \left(\textit{\textbf{x}};\bm{\theta}\right)$ minimizes the loss $Loss$. If $Loss$ is a very small value, the approximation $\phi^h \left(\textit{\textbf{x}};\bm{\theta}\right)$ is very closely satisfying governing equations and boundary conditions, namely
\begin{equation}
\phi^h = \argmin_{\bm{\theta} \in R^K} Loss\left(\bm{\theta}\right)
\end{equation}

The solution of heat conduction problems by deep collocation method can be reduced to an optimization problem. In the deep learning Tensorflow framework, a variety of optimizers are available. One of the most widely used optimization methods is the Adam optimization algorithm, which is also adopted in the numerical study. The idea is to take a descent step at collocation point $\textit{\textbf{x}}_{i}$ with Adam-based learning rate $\eta_i$, 
\begin{equation}
\bm{\theta}_{i+1} = \bm{\theta}_{i} + \eta_i \bigtriangledown_{\bm{\theta} } Loss \left ( \textit{\textbf{x}}_i;\bm{\theta}_i \right )
\label{Adma}
\end{equation}
and then the process in Equation~\eqref{Adma} is repeated until a convergence criterion is satisfied.
\subsection{Convergence of deep collocation method for non-homogeneous PDEs}
With the universal approximation theorem of neural networks, a feedforward neural network is used to approximate the potential function as $\phi^h \left(\textit{\textbf{x}};\bm{\theta}\right)$. The approximation power of neural networks for a quasilinear parabolic PDEs were shown by Sirignano et al. \cite{sirignano2018dgm}. For non-homogeneous elliptic PDEs, the convergence study can be boiled down to:
\begin{equation}
\exists\;\;\phi^h \in F^n, \;\;s.t. \;\;as\;\;n\rightarrow\infty,\;\;Loss(\bm{\theta})\rightarrow0,\;\;\phi^h\rightarrow\phi.
\end{equation}

The non-homogeneous PDEs has a unique solution, s.t. $\phi \in C^2(\Omega)$ with its derivatives uniformly bounded. Also, the conductivity function $k(\textit{\textbf{x}})$ is assumed to be $C^{1,1}$ ($C^1$ with Lipschitz continuous derivative).  

\begin{theorem}\label{theorem2}
Assume that $\Omega$ is compact with measures $\ell_1$, $\ell_2$, and $\ell_3$ whose supports are constrained in $\Omega$, $\Gamma_D$, and $\Gamma_N$. Furthermore, the governing equation~\eqref{eq:gvquation} subject to~\ref{eq:boundarycondition} has a unique classical solution and material function $k(\textit{\textbf{x}})$ being $C^{1,1}$ ($C^1$ with Lipschitz continuous derivative). Then, $\forall \;\; \varepsilon >0 $, $\exists\;\;  K>0$, which may dependent on $sup_{\Omega}\left \|   \phi_{ii}\right \|$ and $sup_{\Omega}\left \|   \phi_{i}\right \|$, s.t. $\exists\;\;  \phi^h\in F^n$, satisfies $Loss(\bm{\theta})\leq K\varepsilon$
\end{theorem}
\begin{proof}
For governing Equation~\eqref{eq:gvquation} subject to \ref{eq:boundarycondition}, according to Theorem \ref{theorem1},
$\forall$ $\varepsilon\;\;  >0$, $\exists\;\;  \phi^h\;\; \in\;\; F^n$, s.t. 
\begin{equation}\label{sup}
\sup_{x\in \Omega}\left \|\phi_{,i}\left(\textit{\textbf{x}}\,_\Omega\right)-  \phi^h_{,i}\left(\textit{\textbf{x}}\,_\Omega\right)\right \|^2+\sup_{x\in \Omega}\left \|\phi_{,ii}\left(\textit{\textbf{x}}\,_\Omega\right)-  \phi^h_{,ii}\left(\textit{\textbf{x}}\,_\Omega\right)\right \|^2<\varepsilon	
\end{equation}
Recalling that the loss is constructed by Equation~\eqref{lossform}, for $MSE_G$ and applying triangle inequality, we obtain:
\begin{equation}
\begin{aligned}
\begin{Vmatrix}
G\left(\textit{\textbf{x}}\,_\Omega;\bm{\theta}\right)
\end{Vmatrix}^2\leqslant\begin{Vmatrix}
k(\bm{x}_\Omega)\phi_{,ii}^{h}\left(\textit{\textbf{x}}\,_\Omega;\bm{\theta}\right)
\end{Vmatrix}^2+\begin{Vmatrix}
k_{,i}(\bm{x}_\Omega)\phi^{h}_{,i}\left  (\textit{\textbf{x}}\,_\Omega;\bm{\theta}\right)
\end{Vmatrix}^2
\end{aligned}
\end{equation}
Let us consider the $C^{1,1}$ conductivity function $k(\textit{\textbf{x}})$, $\exists \;\;M_1>0,\;\;M_2>0$, $\exists \;\; x \in\;\Omega$, $\left \|  k(\textit{\textbf{x}})\right \|\leqslant M_1$, $\left \|  k_{,i}(\textit{\textbf{x}})\right \|\leqslant M_2$. 
From Equation~\eqref{sup}, we can then obtain:
\begin{equation}\label{boundsup1}
\begin{aligned}
\int_{\Omega}k_{,i}^2(\bm{x}_\Omega)\left( \phi_{,i}^h-\phi_{,i} \right )^2d\ell_1\leqslant M_2^2 \varepsilon^2\ell_1(\Omega) \\
\int_{\Omega}k^2(\bm{x}_\Omega)\left ( \phi_{,ii}^h-\phi_{,ii} \right )^2d\ell_1\leqslant M_1^2 \varepsilon^2\ell_1(\Omega)
\end{aligned}
\end{equation}
On boundaries $\Gamma_{D}$ and $\Gamma_{N}$, we can obtain:
\begin{equation}\label{boundsup2}
\begin{aligned}
\int_{\Gamma_{D}}\left (\phi^h \left(\textit{\textbf{x}}\,_{\Gamma_D};\bm{\theta}\right)-\phi\left(\textit{\textbf{x}}\,_{\Gamma_D};\bm{\theta}\right)\right )^2d\ell_2\leqslant \varepsilon^2\ell_2(\Gamma_{D})\\
\int_{\Gamma_{N}}k^2(\bm{x}_{\Gamma_N})\left (\phi^h_{,n} \left(\textit{\textbf{x}}\,_{\Gamma_N};\bm{\theta}\right)-\phi_{,n}\left(\textit{\textbf{x}}\,_{\Gamma_N};\bm{\theta}\right)\right )^2d\ell_3\leqslant M_1^2\varepsilon^2\ell_3(\Gamma_{N})
\end{aligned}
\end{equation}
Therefore, using Equations \ref{boundsup1} and \ref{boundsup2}, as $n\rightarrow\infty$, we obtain
\begin{equation}
\begin{aligned}
&Loss\left(\bm{\theta}\right)=\frac{1}{N_\Omega}\sum_{i=1}^{N_\Omega}\begin{Vmatrix}
k(\bm{x}_\Omega)\phi_{,ii}^{h}\left(\textit{\textbf{x}}\,_\Omega;\bm{\theta}\right)+k_{,i}(\bm{x}_\Omega)\phi^{h}_{,i}\left  (\textit{\textbf{x}}\,_\Omega;\bm{\theta}\right)
\end{Vmatrix}^2+\\ 
&\frac{1}{N_{\Gamma_D}}\sum_{i=1}^{N_{\Gamma_D}}\begin{Vmatrix}
 \phi^h \left(\textit{\textbf{x}}\,_{\Gamma_D};\bm{\theta}\right)-\bar{\phi}
\end{Vmatrix}^2+\frac{1}{N_{\Gamma_N}}\sum_{i=1}^{N_{\Gamma_N}}\begin{Vmatrix}
-k(\bm{x}_{\Gamma_N})\frac{\partial \phi\left(\textit{\textbf{x}}_{\Gamma_N};\bm{\theta}\right)}{\partial n}-\bar{q}
\end{Vmatrix}^2 \\
&\leqslant \frac{1}{N_\Omega}\sum_{i=1}^{N_\Omega}\begin{Vmatrix}
k(\bm{x}_\Omega)\phi_{,ii}^{h}\left(\textit{\textbf{x}}\,_\Omega;\bm{\theta}\right)
\end{Vmatrix}^2+\frac{1}{N_\Omega}\sum_{i=1}^{N_\Omega}\begin{Vmatrix}
k_{,i}(\bm{x}_\Omega)\phi^{h}_{,i}\left  (\textit{\textbf{x}}\,_\Omega;\bm{\theta}\right)
\end{Vmatrix}^2\\ 
&\frac{1}{N_{\Gamma_D}}\sum_{i=1}^{N_{\Gamma_D}}\begin{Vmatrix}
 \phi^h \left(\textit{\textbf{x}}\,_{\Gamma_D};\bm{\theta}\right)-\bar{\phi}
\end{Vmatrix}^2+\frac{1}{N_{\Gamma_N}}\sum_{i=1}^{N_{\Gamma_N}}\begin{Vmatrix}
-k(\bm{x}_{\Gamma_N})\frac{\partial \phi\left(\textit{\textbf{x}}_{\Gamma_N};\bm{\theta}\right)}{\partial n}-\bar{q}
\end{Vmatrix}^2 \\
&\leqslant (M_2^2+M_1^2)\varepsilon^2\ell_1(\Omega)+\varepsilon^2\ell_2(\Gamma_{D})+M_1^2\varepsilon^2\ell_3(\Gamma_{N})=K\varepsilon
\end{aligned}
\end{equation}
\end{proof}
With Theorem \ref{theorem2} and the condition that $\Omega$ is a bounded open subset of R, $\forall n\in N_+$, $\phi^h\in \;F^n \;\in L^2(\Omega)$, it can be concluded from Sirignano et al. \cite{sirignano2018dgm} that:
\begin{theorem}\label{theorem3}
$\forall \;p<2$, $\phi^h\in \;F^n$ converges to $\phi$ strongly in $L^p(\Omega)$ as $n\rightarrow \infty$ with $\phi$ being the unique solution to the potential problems.
\end{theorem}
In summary, for feedforward neural networks $F^n \in L^p$ space ($p<2$), the approximated solution $\phi^h\in F^n$  will converge to the solution to the non-homogeneous PDE. 
\subsection{Collocation points generation}
Model training is an important process in machine learning and the quality of training datasets determines the reliability of the machine learning model to a large extent. The deep collocation method (DCM) utilizes physics-informed neural networks for solving PDEs with randomly generated training points in the physical domain. In order to test the influence of training points on the stability and accuracy, different sampling methods are compared. The Halton and Hammersley sequences generate random points by a constructing the radical inverse \cite{rafajlowicz2006halton}. They are both low discrepancy sequences. The method of Korobov Lattice creates samples from Korobov lattice point sets \cite{wang2004korobov}. Sobol Sequence is a quasi-random low-discrepancy sequence to generate sampling points \cite{dick2007construction}. Latin hypercube sampling (LHS) is a statistical method, where a near-random sample of parameter values is generated from a multidimensional distribution \cite{shields2016generalization}. Monte Carlo methods can create points by  repeated random sampling \cite{shapiro2003monte}. The distribution plots of different sampling points inside a cube is listed in  Appendix~\ref{appendix b} Table~\ref{tab:Tablesp}. 

\subsection{Material transfer learning}
In order to improve the generality and robustness of the DCM, transfer learning is exploited, which will make use of the information from an already trained model yielding to training with less data and a reduced training time. The basic idea can be found in Figure \ref{fig:trans}. For different material variations in  nonhomogeneous media, the 'knowledge' of one material model can be set up as the pretrained model resulting in a two-stage paradigm. The material transfer learning model is divided into two parts, i.e. pretraining, where the network is trained on a large dataset and longer iterations for one material variation type. The remaining part is the fine-tuning, where the pretrained model is trained on other material variations with less data and number of epochs. Consequently, the weights and biases and network configurations from a trained model are passed to other relevant models.
\begin{figure}[!htb]
	\centering
	\includegraphics[width=11.5cm]{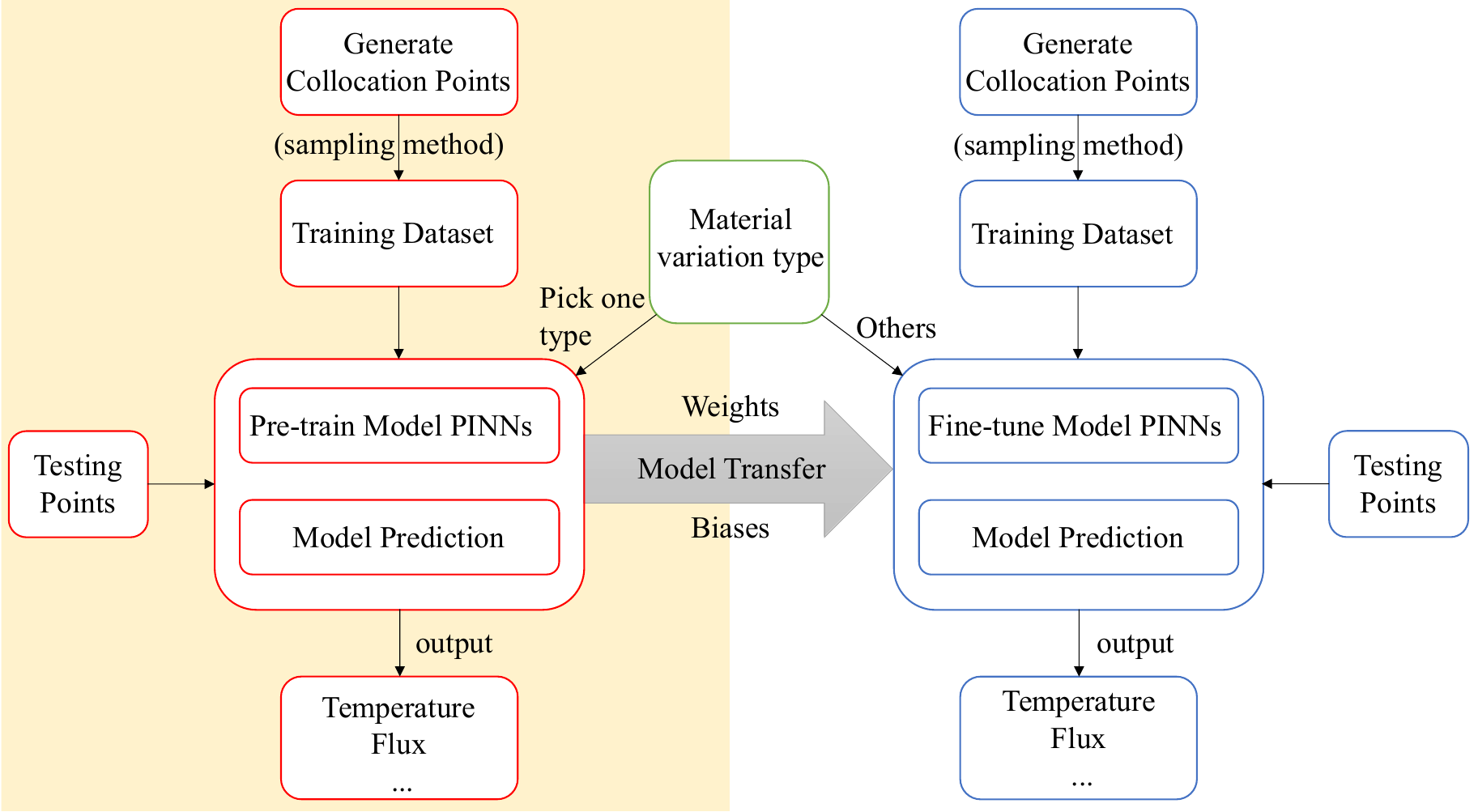}
	\caption{Transfer learning}
	\label{fig:trans}
\end{figure}

\section{Sensitivity analysis}
\label{section 4:sensal}
Algorithm-specific parameters, such as the neural architecture configurations, parameters related to optimizers and number of collocation points  significantly influence the model's accuracy. To quantify their influence on the accuracy, a global sensitivity analysis (GSA)  is performed. Classical GSA including regression methods, screening approaches such as Morris method \cite{iooss2015review}, the variance-based measures, such as Sobol's method \cite{sobol2001global}, and  the Fourier amplitude sensitivity test (FAST) \cite{cukier1973study}, or the extended FAST (EFAST) \cite{saltelli1999quantitative}. 

Variance-based methods are usually more computationally expensive than the derivative-based methods as well as the regression methods. If the model or the parameters in analysis is large, the use of variance-based method can be costly. The method of Morris is generally robust to correctly screen the most and least sensitive parameters for a highly parameterized model with 300 times fewer model evaluations than the Sobol' method \cite{herman2013method}. Therefore, the computational cost of a sensitivity analysis can potentially be reduced by first performing parameter screening using the Morris method to identify non-influential parameters, reducing the dimension of the parameter space to be studied in further analysis, then filter them again, but with the eFAST method. In this way, we can quantifying the effects of inputs more accurately with a relatively small amount of time.

\subsection{Method of Morris}
\label{subsubsection 4.1:morris}
The method of Morris \cite{morris1991factorial} is a screening technique used to rank the importance of parameters by averaging coarse difference relations termed elementary effects. Given a model with $n$ parameters,  $\bm{X}={X_1,X_2,...X_n}$ denoting a vector of parameter values, we can specify an objective function $y(x)=f(X_1,X_2,...X_n)$, change the variables $X_i$ by specific ranges and then calculate the distribution of elementary effects (EE) of each input factor with respect to the model outputs, i.e.
\begin{equation}\label{eq:morris}
EE_i=\frac{f(x_1,...,x_i+\Delta_i,...,x_n)-f(x)}{\Delta_i}
\end{equation}
where $f(x)$ represents the prior point in the trajectory. Using the single trajectory shown in Equation~\eqref{eq:morris}, the elementary effects of each parameter can be calculated with $p+1$ model evaluations. After sampling the trajectories, the resulting set of elementary effects are then averaged to obtain the total-order sensitivity of the $i$-th parameter $\mu_{i}^{*}$:
\begin{equation}\label{eq:mu}
\mu_{i}^{*}=\frac{1}{n}\sum_{j=1}^{n}\abs{EE_{i}^{j}}
\end{equation}
Similarly, the variance of the set of EEs can be calculated as
\begin{equation}\label{eq:sigma}
\sigma_{i}^{2}=\frac{1}{n-1}\sum_{j=1}^{n}( EE_{i}^{j}-\mu_i)^2
\end{equation}

The mean value $\mu^*$ quantifies the individual effect of the parameters on an output while the variance $\sigma^2$ indicates the influence of parameter interactions. We rank the parameters according to $\sqrt{\sigma^2+{\mu^*}^2}$. 

\subsection{eFAST method}
\label{subsubsection 4.2:eFAST}
The eFAST method \cite{saltelli1999quantitative} is based on Fourier transformations. The spectrum is obtained by each parameter and the output variance of model results due to interactions. Employing a suitable search function, the model $y(x)=f(X_1,X_2,...X_n)$ can be transformed by the Fourier transform into $y= f(s)$ 
\begin{equation}\label{eq:FT}
y=f(s)=\sum_{j=-\infty}^{+\infty}\big (A_j cos(js)+B_j sin(js)\big ),
\end{equation}
with
\begin{equation}\label{eq:FT1}
A_j=\frac{\pi}{2}\int_{\frac{\pi}{2}}^{-\frac{\pi}{2}}f(s)cos(js)\mathrm{d}s,
\end{equation}
\begin{equation}\label{eq:FT2}
B_j=\frac{\pi}{2}\int_{\frac{\pi}{2}}^{-\frac{\pi}{2}}f(s)sin(js)\mathrm{d}s.
\end{equation}

The spectral curve of the Fourier progression is defined as $\Lambda_j=A_j^2+B_j^2$. The variance of the model results due to the uncertainty in the parameter $X_i$ is given by
\begin{equation}\label{eq:D}
D_i=\sum_{p\in Z_0}\Lambda_p\omega_i,
\end{equation}
with the parametric frequency $\omega_1$, the spectrum of the Fourier transform $\Lambda$,  and the non-zero integers $Z_0$. The total variance can be obtained by cumulatively summing the spectra at all frequencies
\begin{equation}\label{eq:Dsum}
D=2\sum_{j=1}^{\infty}\Lambda_j.
\end{equation}

The fraction of the total output variance caused by each parameter apart from interactions with other parameters is measured by the first-order index
\begin{equation}\label{eq:S}
S_{i}^{FAST}=\frac{D_i}{D}.
\end{equation}

To find the total sensitivity of $X_i$, the frequency of $X_i$ is set to $\omega_i$, while a different frequency $\omega'$ is set for all other parameters. By calculating the frequency $\omega_i$ and its higher harmonics $p\omega_i$ spectra, the output variance $D_{-i}$ due to the influence of all parameters except $X_i$ and their interrelationships can be obtained. Thus, the total-order sensitivity indices can be obtained:
\begin{equation}\label{eq:Ssum}
S_{T_{i}}=\frac{D-D_{-i}}{D}.
\end{equation}

\section{Numerical examples}
In this section, several cases are considered testing the accuracy and efficiency of our DCM including the influence of suitable NN configurations, sampling methods and optimizers taking advantage of GSA. Also, different material variations using material transfer learning are studied. The accuracy is measured in the relative error between the predicted solution and the analytical solution:
\begin{equation}\label{relerr}
e=\frac{\left \| E_{pred} - E_{a} \right \|}{\left \| E_{a} \right \|}
\end{equation}
where $E_{a}$ is the analytical solution and  $E_{pred}$ is the predicted solution while $ \left \| \cdot \right \|$ refers to the $L_2$-norm. All simulations are done on a 64-bit macOS Catalina computer with Intel(R) Core(TM) i7-8850H CPU, 32GB memory. The parametric settings for training is summarised in Table \ref{tab:hyp}.
\begin{table}[!htb] 
	\captionsetup{width=0.85\columnwidth}
	\caption{Hyper-parameters settings in training}
	\centering 
		\resizebox{0.95\columnwidth}{!}{
		\begin{tabular}{c| c| c}
			\toprule
			\toprule 
			Model & Hyper-parameters & Values\\
			\midrule
			Adam optimizer&Learning rate & 0.001\\ 
			\midrule
			 L-BFGS-B optimizer&Maximum number of iterations to perform & 50000\\ 	
			 &Maximum number of function evaluations& 50000\\ 		
			 &Maximum number of variable metric corrections  & 50\\ 		 		
			 &Maximum number of line search steps (per iteration) & 50\\ 		 					
			\bottomrule 
		\end{tabular}
		}
	\label{tab:hyp} 
\end{table}

\subsection{Case 1: Sensitivity analysis}
First, we perform a SA to determine the key parameters of the deep collocation method.
\subsubsection{Parameters screening with Morris method}
The sensitivity indices computed by the Morris screening method with 30 trajectories and 4 grid levels are listed in Figures \ref{mu_star1} and \ref{covariance1}, showing the effect of the numbers of neurons, layers, iterations and collocation points on the loss values. Figure \ref{mu_star1} depicts the horizontal barplot of the GSA measure $\mu^*$. The highest $\mu^*$ value is found for the numbers of layers and neurons. The numbers of collocation points barely have an effect on the loss value. According to a classification scheme proposed by Garcia Sanchez et al. \cite{sanchez2014application}, the ratio $\sigma/\mu^*$ allows the characterisation of the model parameters in terms of (non-)linearity $(\sigma/\mu^*< 0.1)$, (non-) monotony $(0.1<\sigma/\mu^*< 0.5)$ or possible parameter interactions $(1<\sigma/\mu^*)$, see also Figure \ref{covariance1}. For our test models, all parameters are in the range $\sigma/\mu^*>1$ suggesting that most parameters exhibit either non-linear behaviour, interaction effects with each other or both. The plot of the mean value and standard deviation $(\sigma, \mu^*)$ in Figure \ref{covariance1} reveals that the most influential parameter with largest $\sqrt{\sigma^2+{\mu^*}^2}$ is the numbers of layers. The number of neurons and iterations are less important. The collocation points inside the physical domain and on the surface do not have a significant impact neither. Thus while tuning the parameters of the model, more attention should be paid on the numbers of layers, neurons and iterations.  

\begin{figure}[!htb]
	\captionsetup{width=0.9\columnwidth}
	\centering\includegraphics[height=6cm,width=8cm]{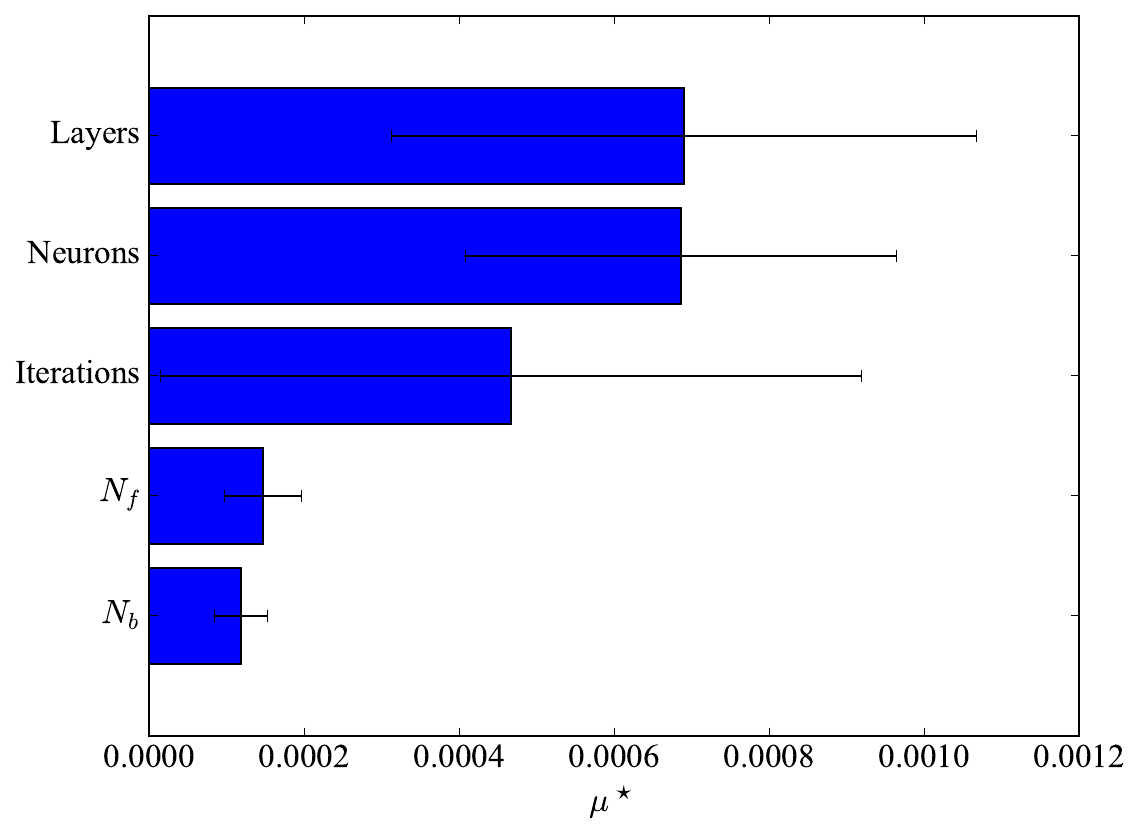}   
	\caption{Horizontal barplot of the mean absolute elementary effects $\mu^*$}
	\label{mu_star1}
\end{figure}
\begin{figure}[!htb]
	\captionsetup{width=0.9\columnwidth}
	\centering\includegraphics[height=6cm,width=8cm]{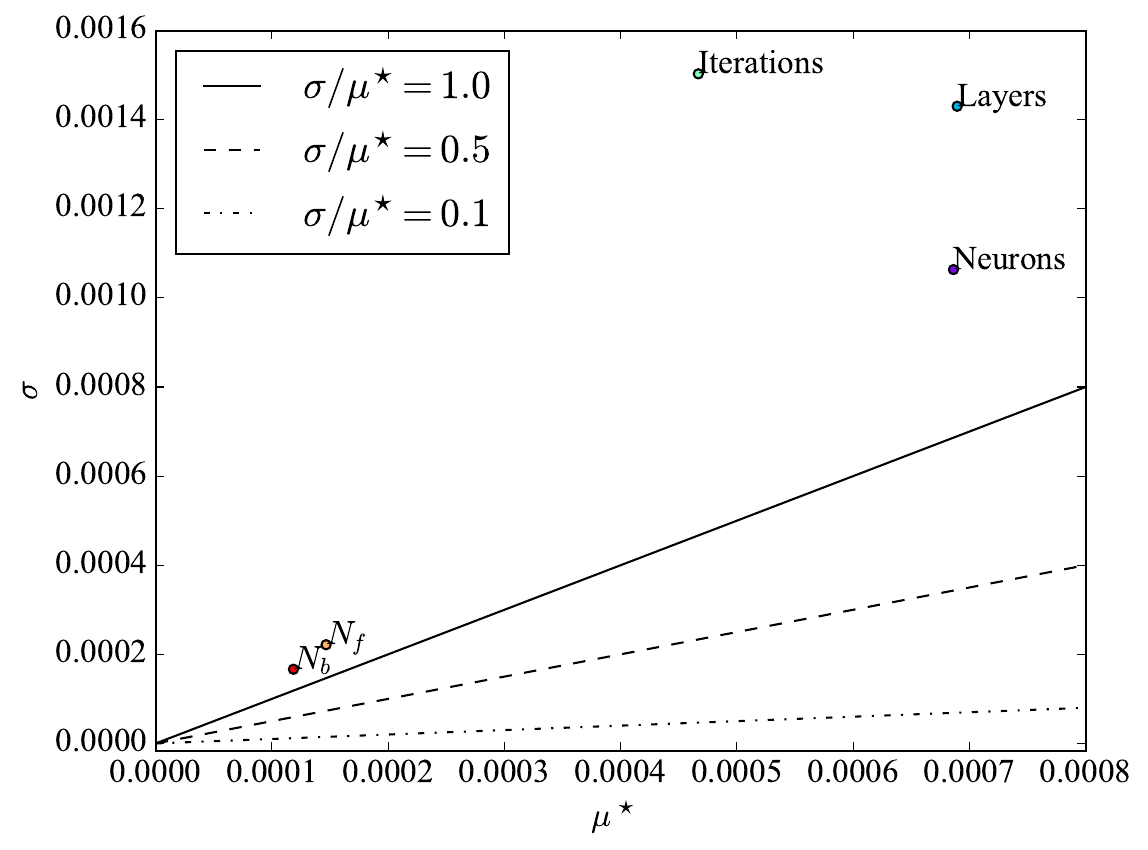}   
	\caption{ $\sigma$ versus $\mu^*$ for parameter screening with Morris method}
	\label{covariance1}
\end{figure}

\subsubsection{Variance-based sensitivity indices}
We now take advantage of the variance-based eFAST method  to compute  sensitivity indices. The independent first-order sensitivity indices $S_i$ and dependent total order sensitivity indices $S_{T_i}$ can be found in Figure \ref{sensitivityindices}. Due to the high computational costs, with 3000 simulations run and 1000 generated samples, no analysis concerning the variation of in $S_i$ and $S_{T_i}$ with different sample sizes were performed.

The associated scatter plots are shown in Figure \ref{scatterply}. The more randomly the loss values are distributed, the less sensitive the parameters is. According to Figure \ref{scatterply}, the number of layers is the most influential parameter, followed by the number of neurons and number of iterations.
\begin{figure}[!htb]
	\captionsetup{width=0.9\columnwidth}
	\centering\includegraphics[height=5cm,width=11cm]{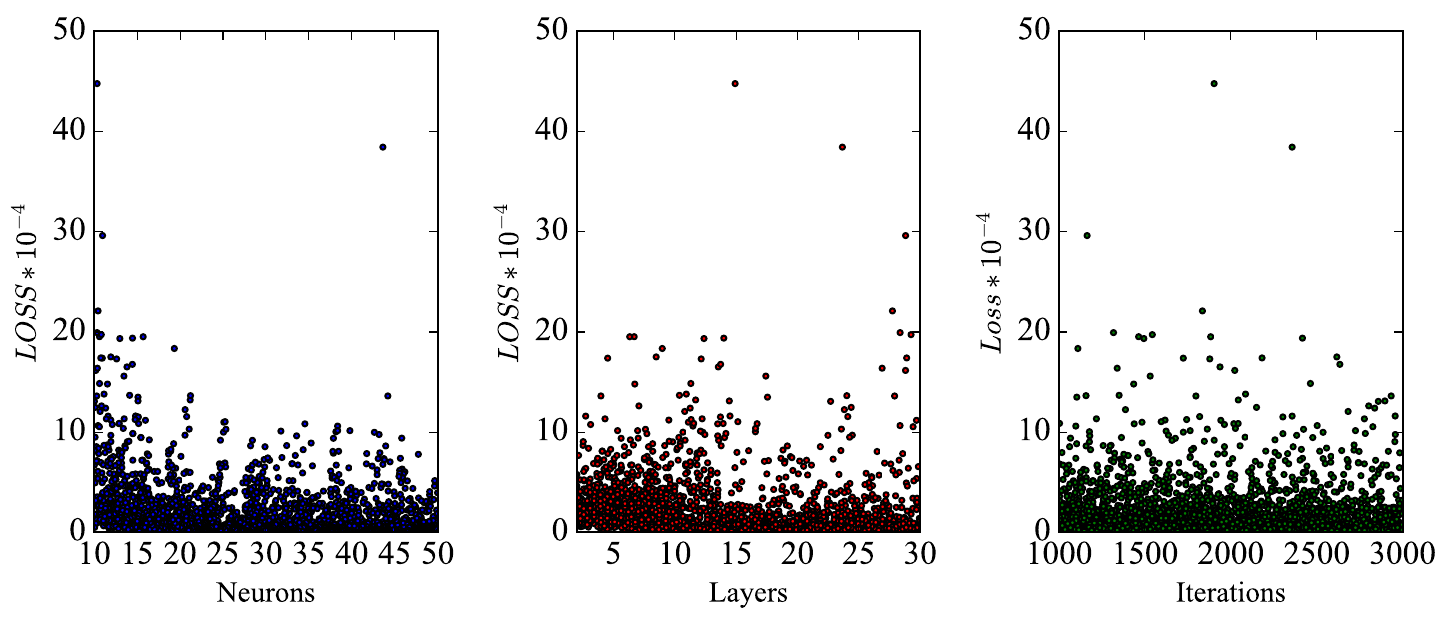}   
	\caption{Scatter plot of loss value against parameter values with 3,000 runs}
	\label{scatterply}
\end{figure}

The first order sensitivity index $S_i$ represents the parameter importance. The number of layers affects the model most, followed by the numbers of neurons and the least influential parameter is the number of iterations, which agrees well with the results of Morris Method. However, the first-order indices are all beyond 0.01, which manifest that those algorithm-specific parameters individually do not have too much influence on the loss value of the model. The total effects index $S_{T_i}$ greater than 0.8 can be regarded very important parameters. Again, the number of layers and neurons are greater than 0.8. For the number of iterations it is between 0.5 to 0.8. However, there is a big difference between the value of total and first-order sensitivity indices, which quantifies the effects of the parameter’s interactions. It can be concluded that the output variance can be attributed to their interactions with other parameters rather than their nonlinear effects and all interactions between these three parameters are noteworthy. 
\begin{figure}[!htb]
	\captionsetup{width=0.9\columnwidth}
	\centering\includegraphics[height=6cm,width=8cm]{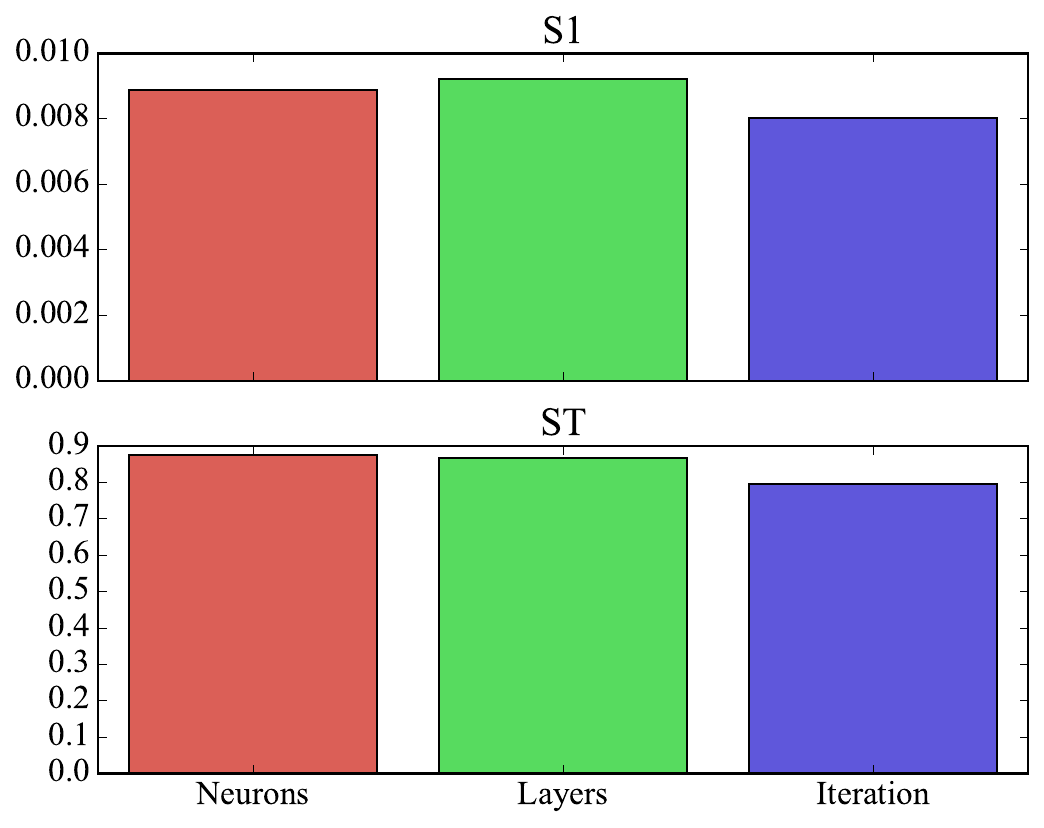}   
	\caption{Results for first-order and total indices with eFAST method}
	\label{sensitivityindices}
\end{figure}
 
\subsection{Case 2. Cube with material gradation along the z-axis}
\label{Case1}
Let us consider a unit cube (L=1) with prescribed constant temperature on two sides. The top surface of the cube at z=1 is maintained at a temperature of T = 100 while the bottom temperature at z=0 is zero. The remaining four faces are insulated (zero normal flux). Three different classes of variations shown in Table \ref{tab:Table5} are considered \cite{sutradhar2004simple}. The profiles of the thermal conductivity k(z) of the three material variation cases are illustrated in Figure~\ref{thermal conductivity}, and the boundary conditions of the unit cube can be found in Figure~\ref{conditions of cube}. For each nonhomogeneous thermal conductivity, the analytical solution is summarized in Table~\ref{tab:Table5}.
\begin{figure}[!htb]
\captionsetup{width=0.9\columnwidth}
\centering\includegraphics[height=6cm,width=8.0cm]{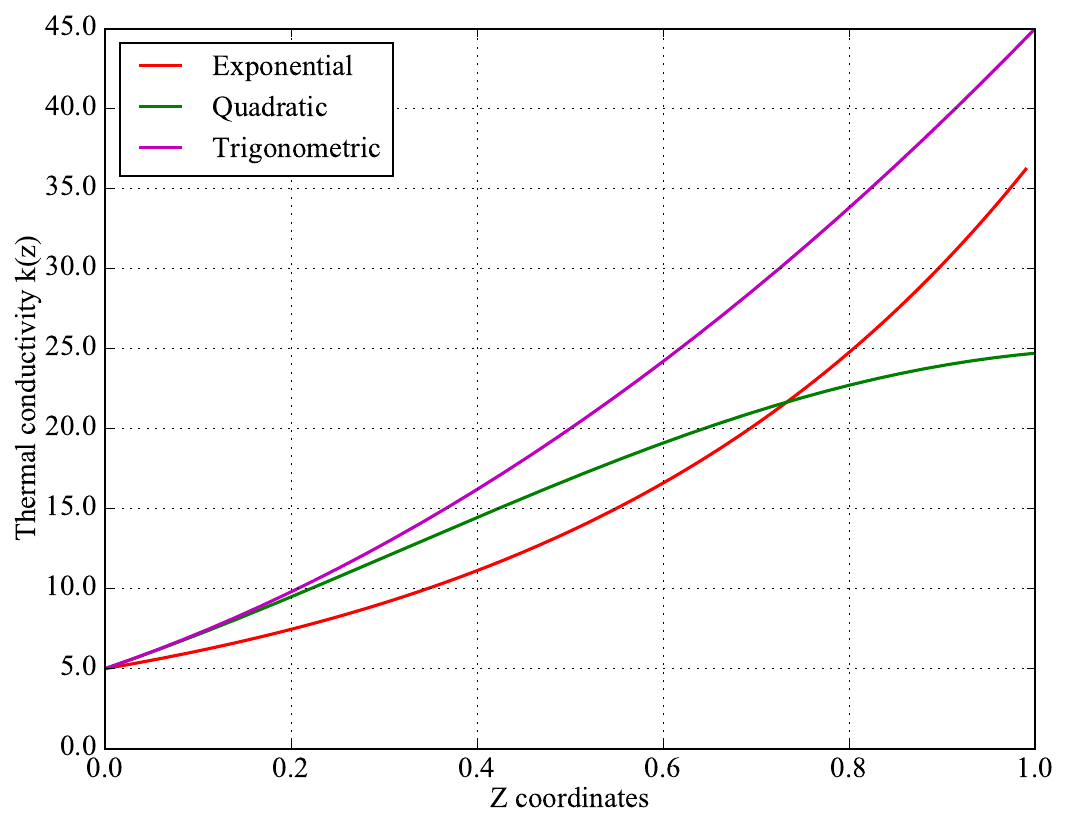}   
 \caption{Thermal conductivity variation along the z direction}
\label{thermal conductivity}
\end{figure}
\begin{figure}[!htb]
	\captionsetup{width=0.9\columnwidth}
	\centering\includegraphics[height=6cm,width=8.0cm]{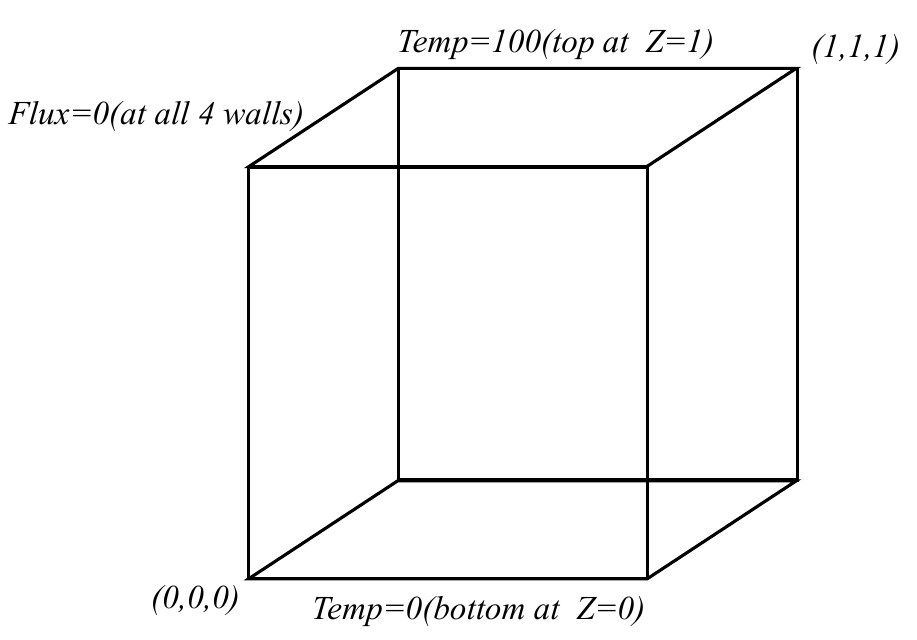}   
	\caption{The boundary conditons of the unit cube}
	\label{conditions of cube}
\end{figure}

\begin{table}[!h] 
\captionsetup{width=0.9\columnwidth}
\caption{Analytical solutions for various forms of thermal conductivity $k(\bm{x})$} % Table caption, can be commented out if no caption is required
\vspace{-0.1cm}
\centering % Centres the table on the page, comment out to left-justify
\resizebox{0.65\columnwidth}{!}{%
\begin{tabular}{c| c} % The final bracket specifies the number of columns in the table along with left and right borders which are specified using vertical pipes (|); each column can be left, right or center-justified using l, r or c. Columns will widen to hold the content in them by default, to specify a precise width, use p{width}, e.g. p{5cm}
\toprule % Top horizontal line
\toprule % Top horizontal line
$k(\bm{x})$& Analytical solution for potential function\\ % Column names row
\midrule % In-table horizontal line
$5(1+2z)^{2}$ & $\phi=\frac{300z}{1+2z}$\\ % Content row 1
\midrule
$5e^{2z}$&$\phi=100\frac{1-e^{-2 z}}{1-e^{-2 L}}$\\ % Content row 2
\midrule
$5(cosz+2sin z)^{2}$&$\phi=100\frac{(cot(L)+2)*sinz}{(cos z+2sinz)}$\\ % Content row 3
\bottomrule % Bottom horizontal line
\end{tabular}
}
\label{tab:Table5} % A label for referencing this table elsewhere, references are used in text as \ref{label}
\end{table}

\subsubsection{Deep collocation method configurations}
First, different NN configurations are investigated. Figure~\ref{Comparison of activation functions} shows the relative error for various activation functions and layers. The $arctan$ function yields the most stable and accurate results. Both $arctan$ and $Tanh$ function outperform the other activation functions. Figure~\ref{Comparison of sampling} depicts the influence of different sampling methods on the relative error. Random sampling method obtained most stable and accurate potentials with increasing layers.  $Korobov$, $Hammersley$,  $Latin Hypercube$ sampling methods also provide reasonable results. 

Next, we focus on various material variations, see Figure~\ref{Comparison of material}. All material variations can be predicted accurately, but the most accurate results are obtained for the exponential conductivity. The results from Figure~\ref{Comparison of activation functions} to \ref{Comparison of material} suggest that 2 hidden layers are a good choice for the underlying problem. 

\begin{figure}[!htb]
\captionsetup{width=0.9\columnwidth}
\centering\includegraphics[height=6cm,width=8.0cm]{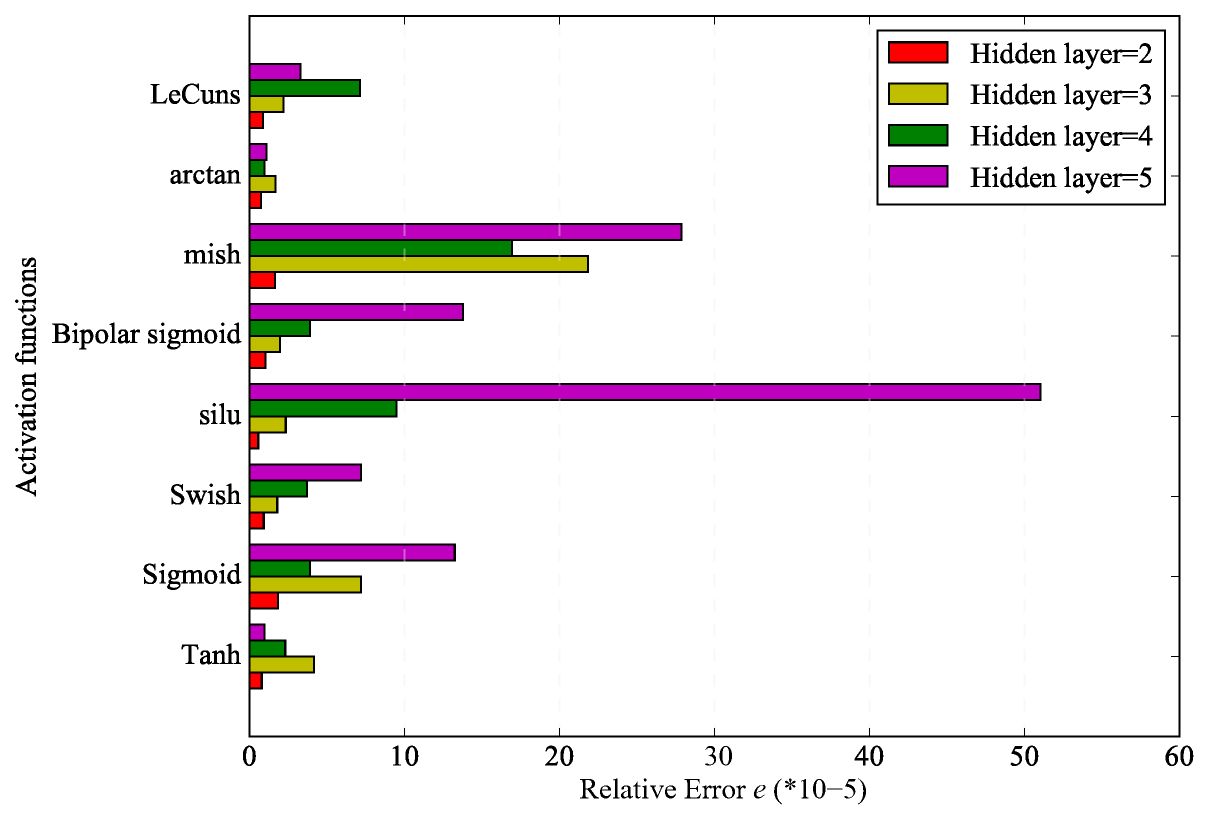}   
 \caption{Comparison of results predicted by DCM with different activation functions}
\label{Comparison of activation functions}
\end{figure}

\begin{figure}[!htb]
\captionsetup{width=0.9\columnwidth}
\centering\includegraphics[height=6cm,width=8.0cm]{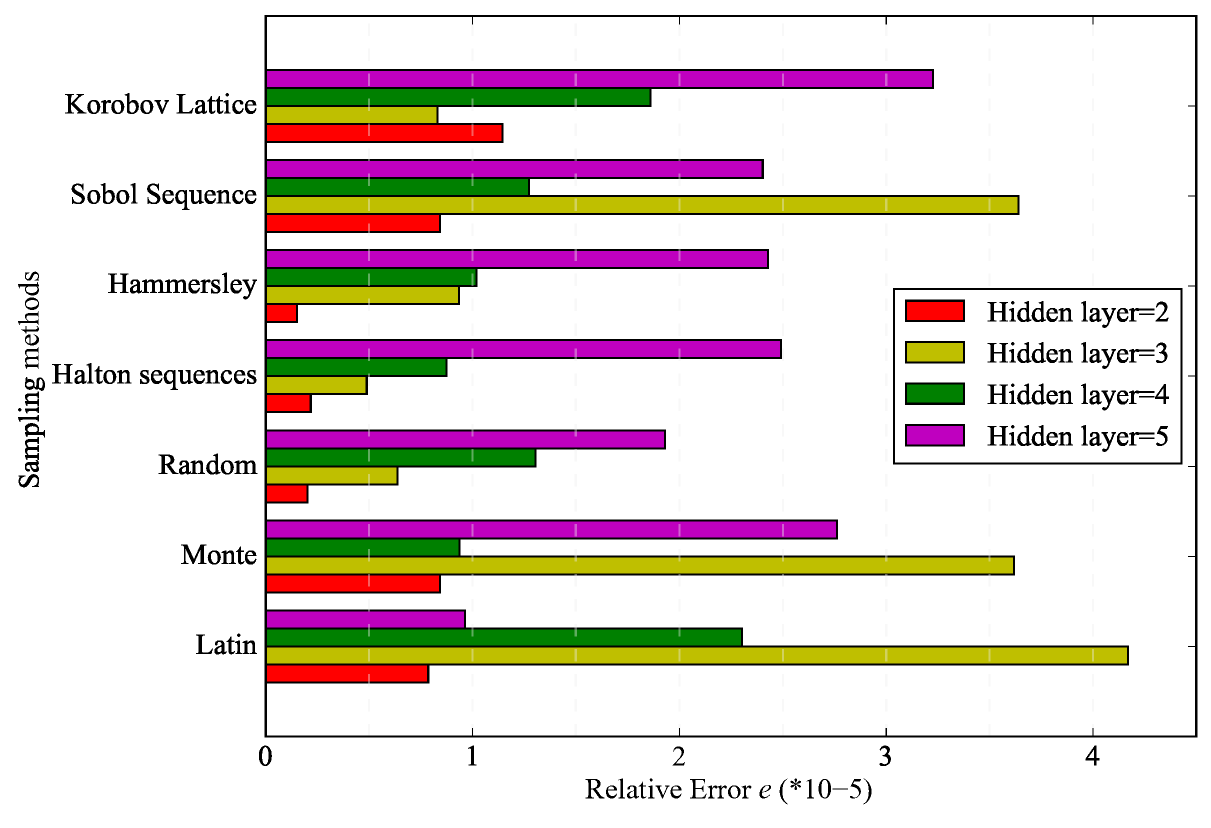}   
 \caption{Comparison of results predicted by DCM with different sampling methods}
\label{Comparison of sampling}
\end{figure}

\begin{figure}[!htb]
\captionsetup{width=0.9\columnwidth}
\centering\includegraphics[height=6cm,width=8.0cm]{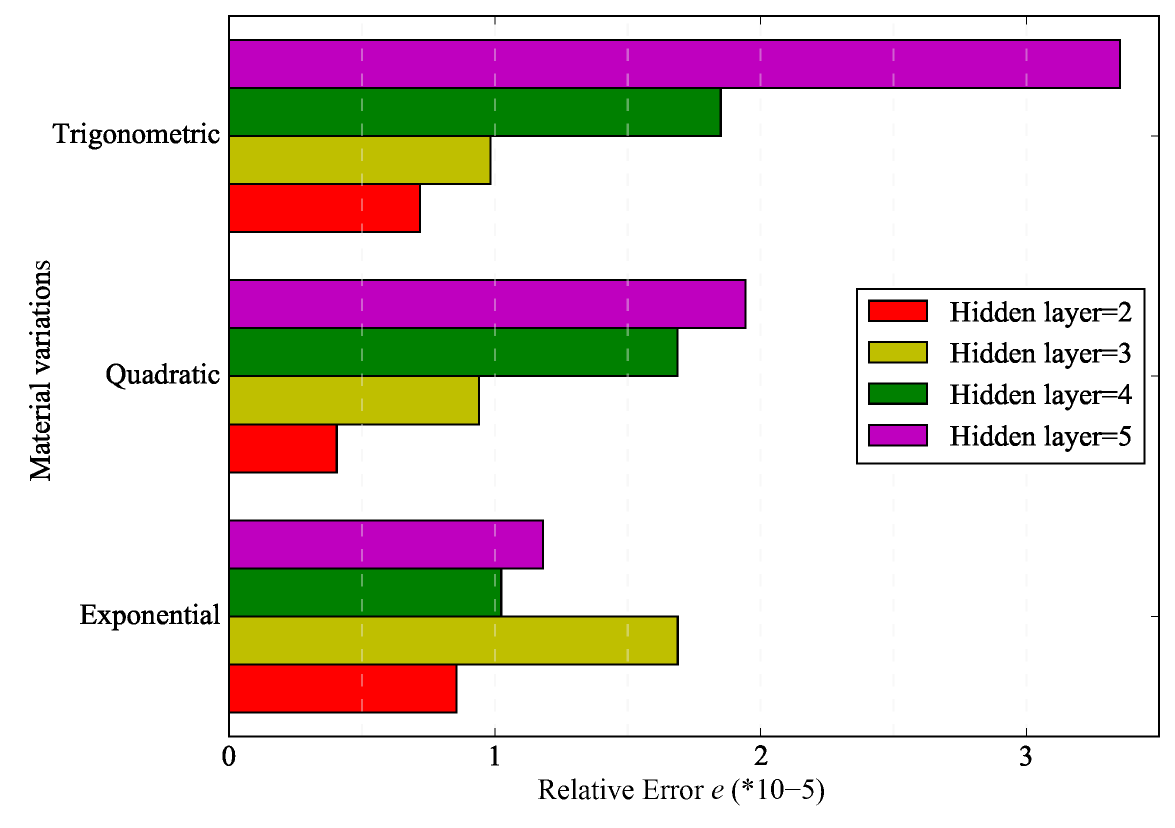}   
 \caption{Comparison of predicted results for different material variations}
\label{Comparison of material}
\end{figure}

\begin{figure}[!htb]
	\captionsetup{width=0.9\columnwidth}
	\centering\includegraphics[height=6cm,width=8.0cm]{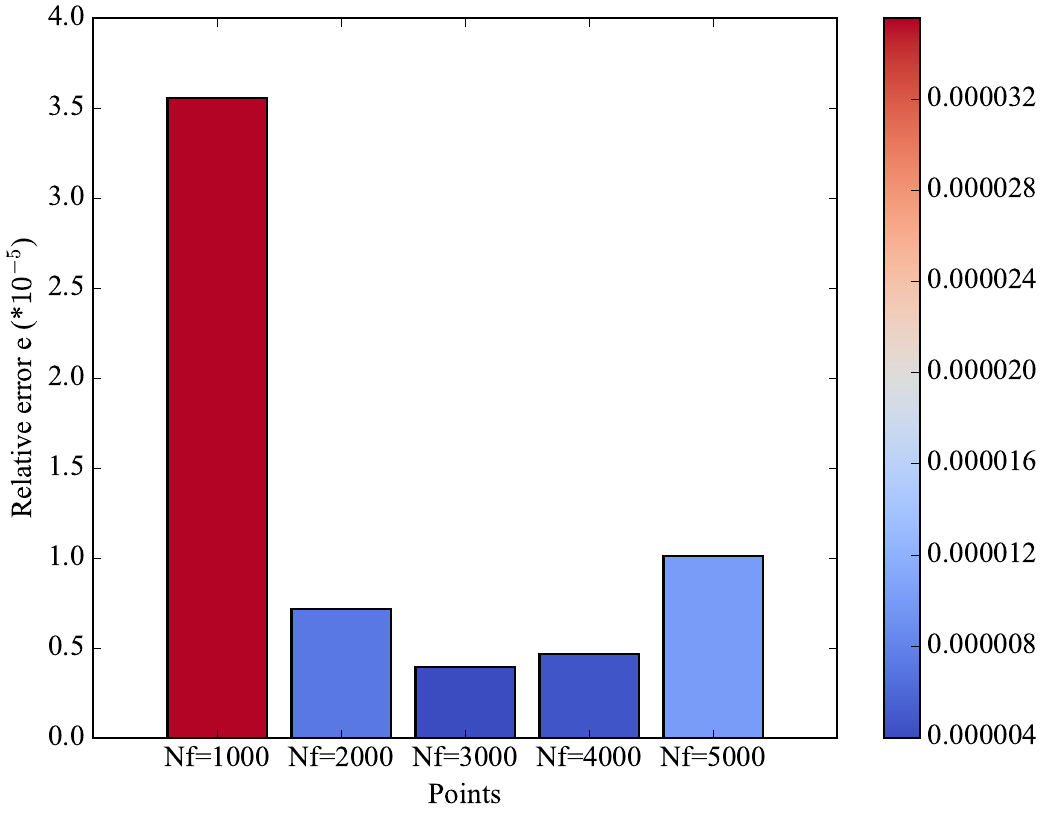}   
	\caption{Comparison of predicted results for different collocation points in cube}
	\label{N_f}
\end{figure}

\begin{figure}[!htb]
	\captionsetup{width=0.9\columnwidth}
	\centering\includegraphics[height=6cm,width=8.0cm]{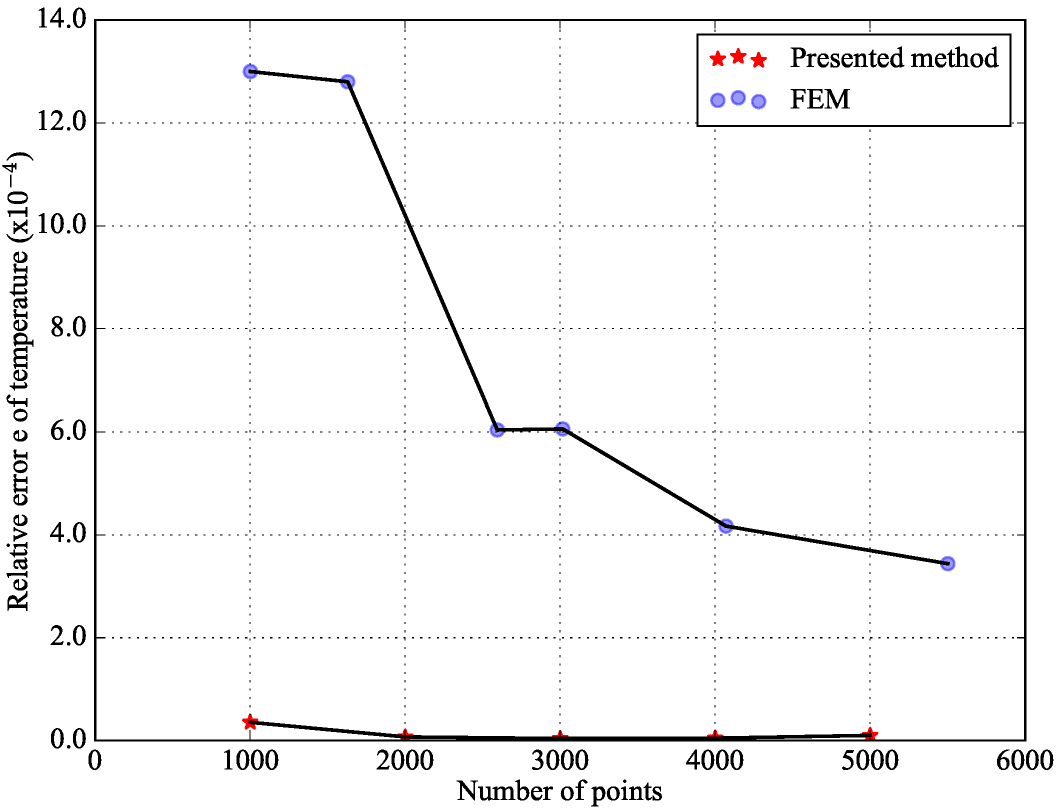}   
	\caption{Comparison of predicted results with FEM versus numbers of points}
	\label{N_f_FEM}
\end{figure}

\begin{figure}[!htb]
	\captionsetup{width=0.9\columnwidth}
	\centering\includegraphics[height=6cm,width=8.0cm]{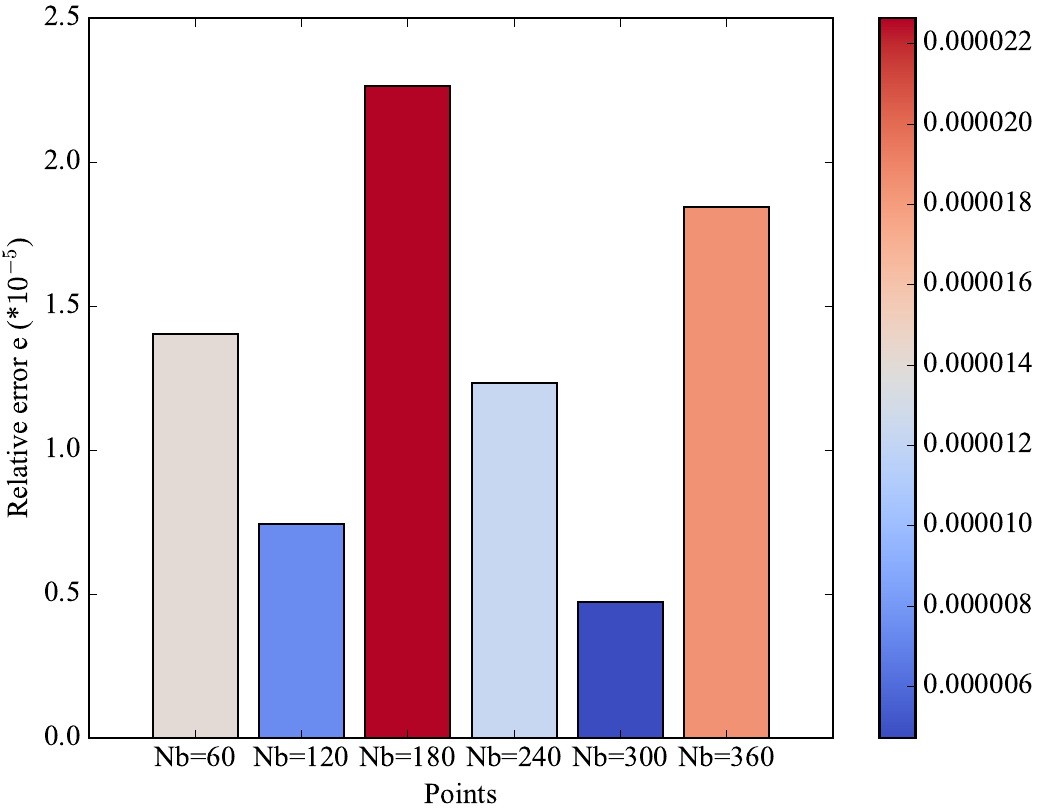}   
	\caption{Comparison of predicted results for different collocation points on each boundary}
	\label{N_b}
\end{figure}
We study now different numbers of collocation points (inside the cube and on its surface). The relative error in the temperature is depicted in Figure~\ref{N_f} and \ref{N_b}. We also compared our results to results from FEM in Figure \ref{N_f_FEM}. The temperature profile along the z-axis for three material variations are plotted  with the corresponding analytical solutions in Figure~\ref{Temperature profile}.   
\begin{figure}[!htb]
	\captionsetup{width=0.9\columnwidth}
	\centering\includegraphics[height=6cm,width=8.0cm]{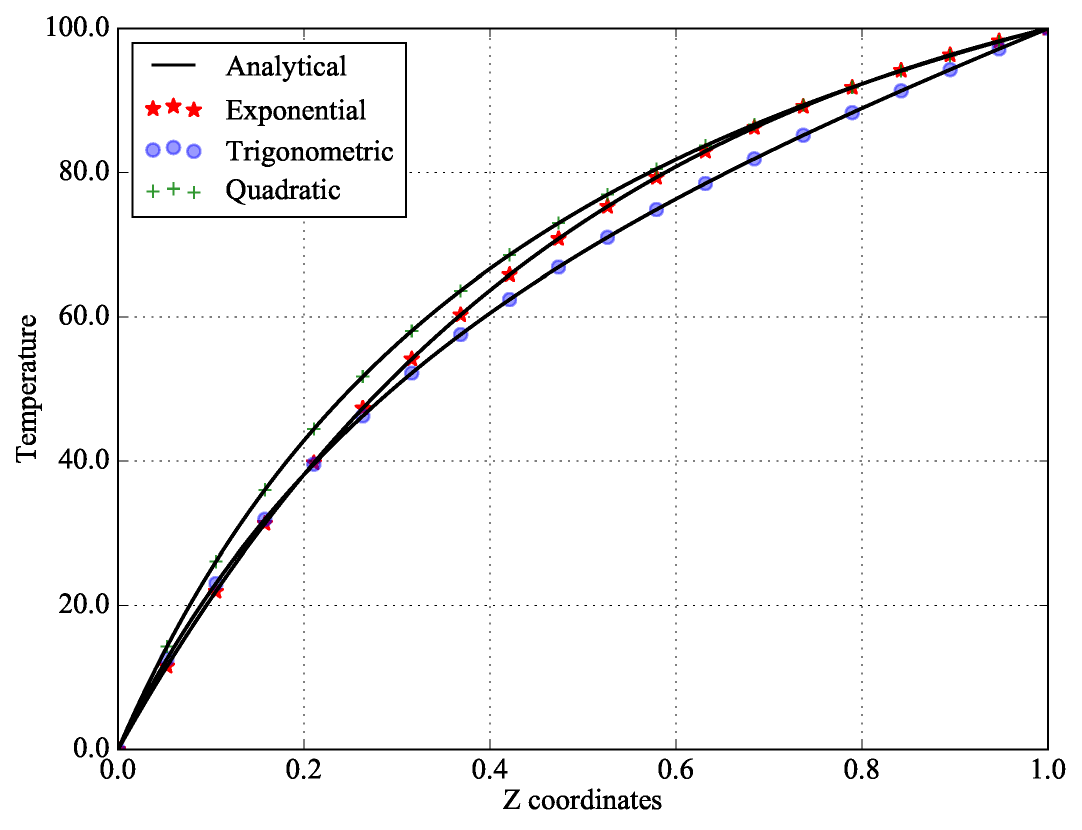}   
	\caption{Temperature profile in the Z direction for different material variations}
	\label{Temperature profile}
\end{figure}

The predicted temperature and flux distributions for three material variations inside the cube is shown in Figures~\ref{flux_1}-\ref{flux_3}. The heat distribution varies with graded variation in the z coordinates which is consistent with the material property of the FGMs.   

\begin{figure}[!htb]
	\captionsetup{width=0.9\columnwidth}
	\centering
	\begin{subfigure}[b]{5.0cm}
		\centering\includegraphics[height=5cm,width=5.0cm]{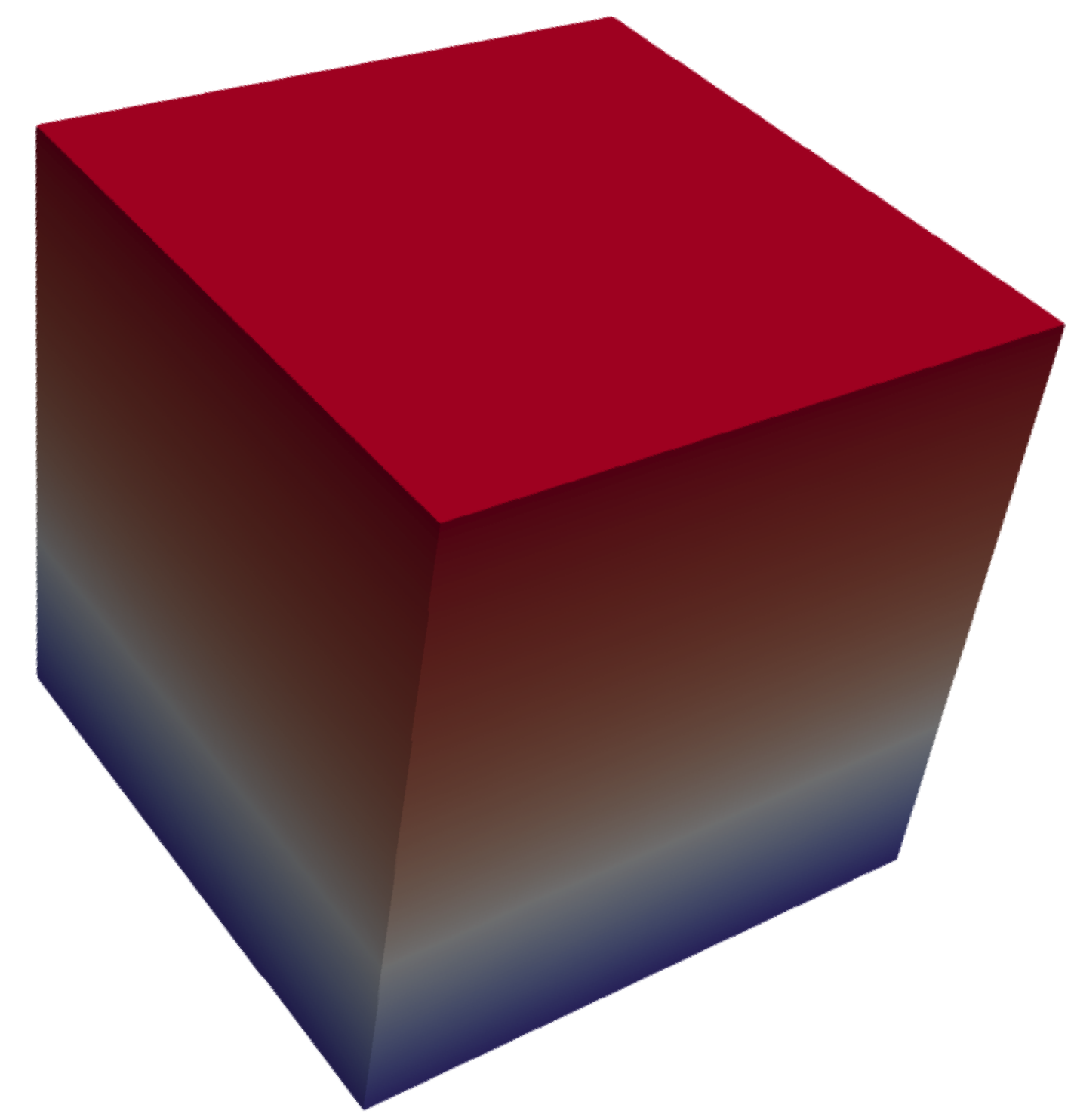}   
		\caption{}\label{k1}
	\end{subfigure}%
	\hspace{0.5cm}
    \begin{subfigure}[b]{5.0cm}
	\centering\includegraphics[height=5cm,width=5.0cm]{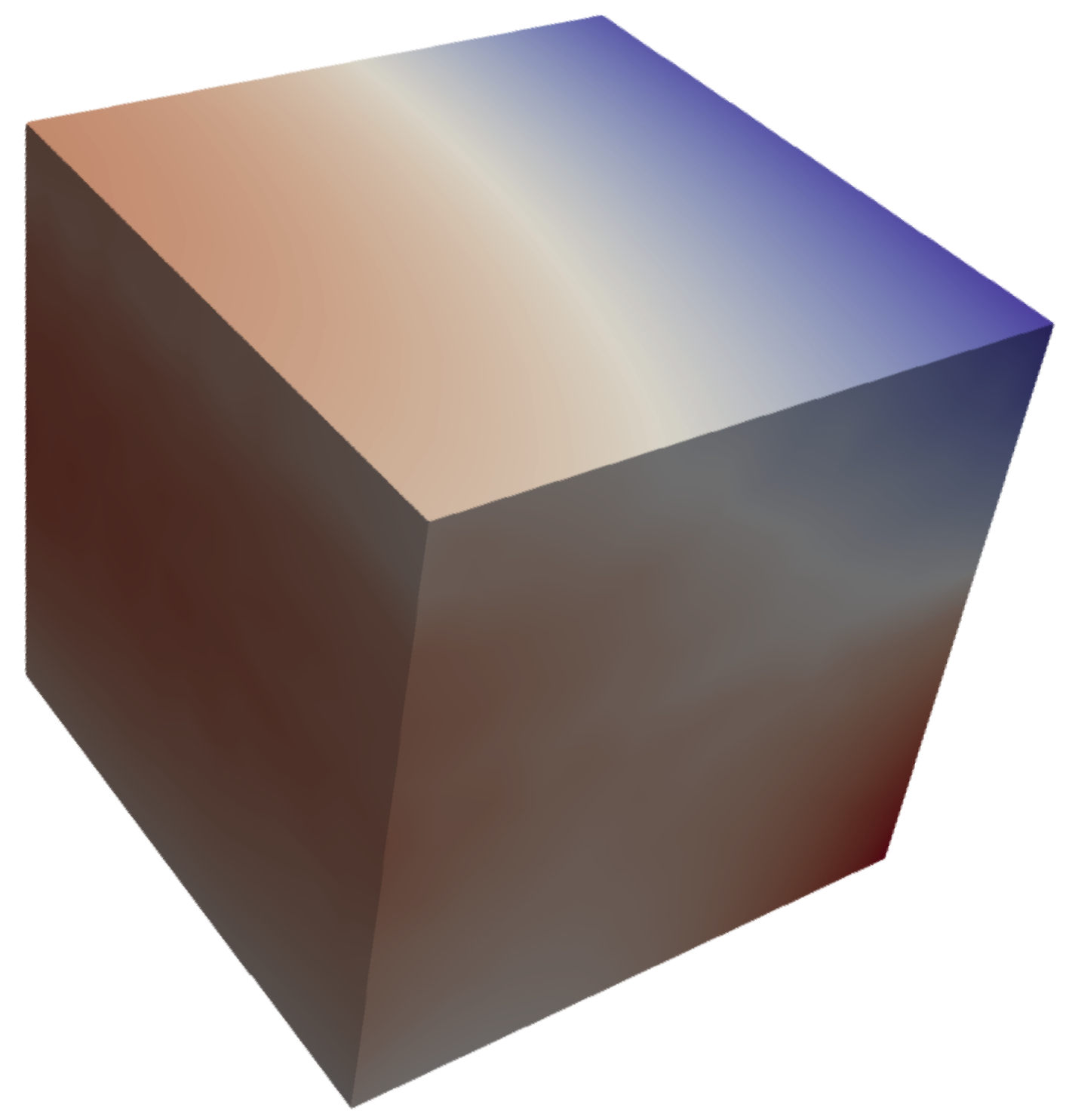}   
	\caption{}\label{k2}
    \end{subfigure}% 
	\caption{$\left(a\right)$ Predicted temperature and $\left(b\right)$ Predicted flux for exponential material variation for the functionally graded unit cubic }
	\label{flux_1}
\end{figure}

\begin{figure}[!htb]
	\captionsetup{width=0.9\columnwidth}
	\centering
	\begin{subfigure}[b]{5.0cm}
		\centering\includegraphics[height=5cm,width=5.0cm]{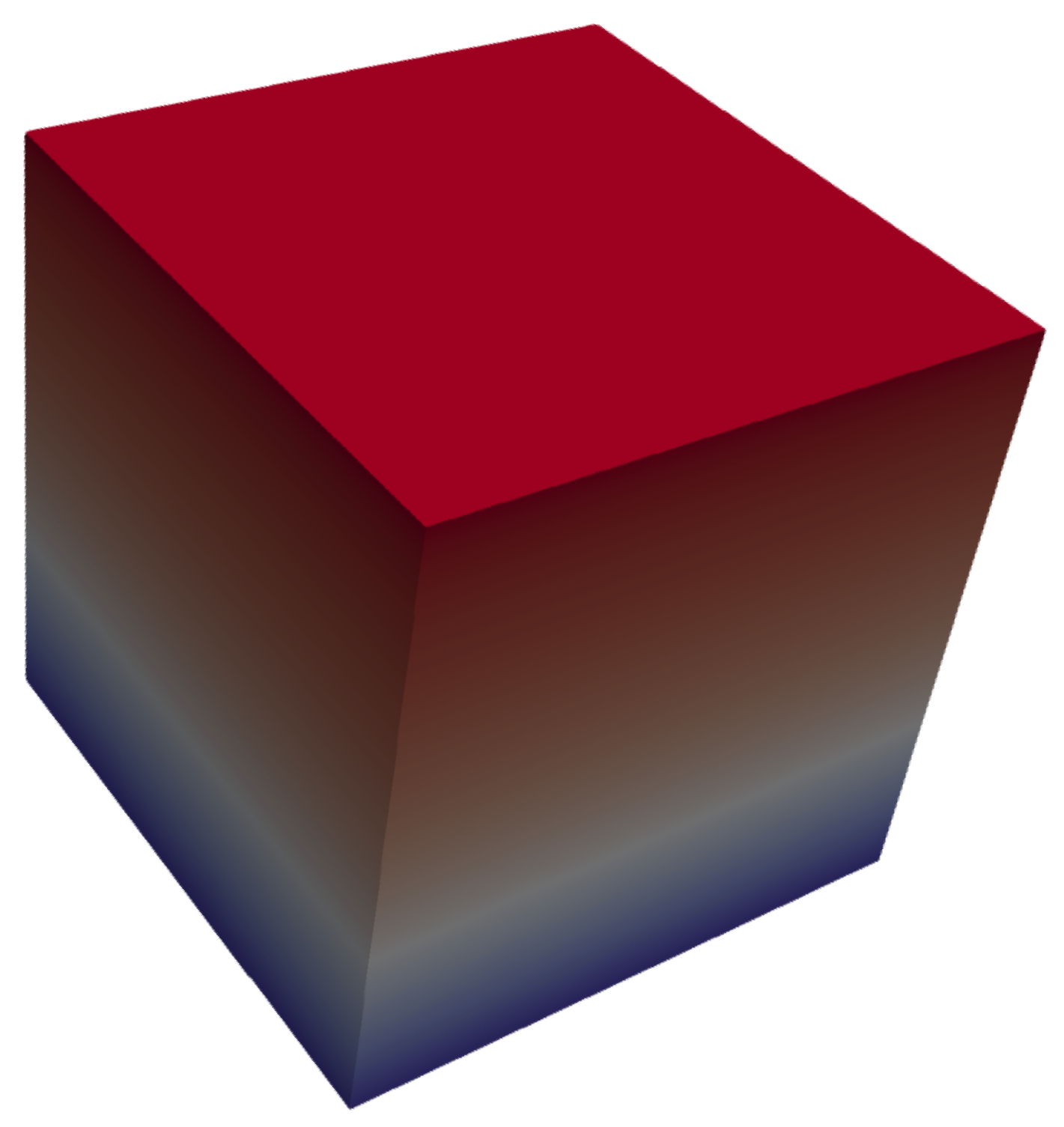}   
		\caption{}\label{k3}
	\end{subfigure}%
	\hspace{0.5cm}
	\begin{subfigure}[b]{5.0cm}
		\centering\includegraphics[height=5cm,width=5.0cm]{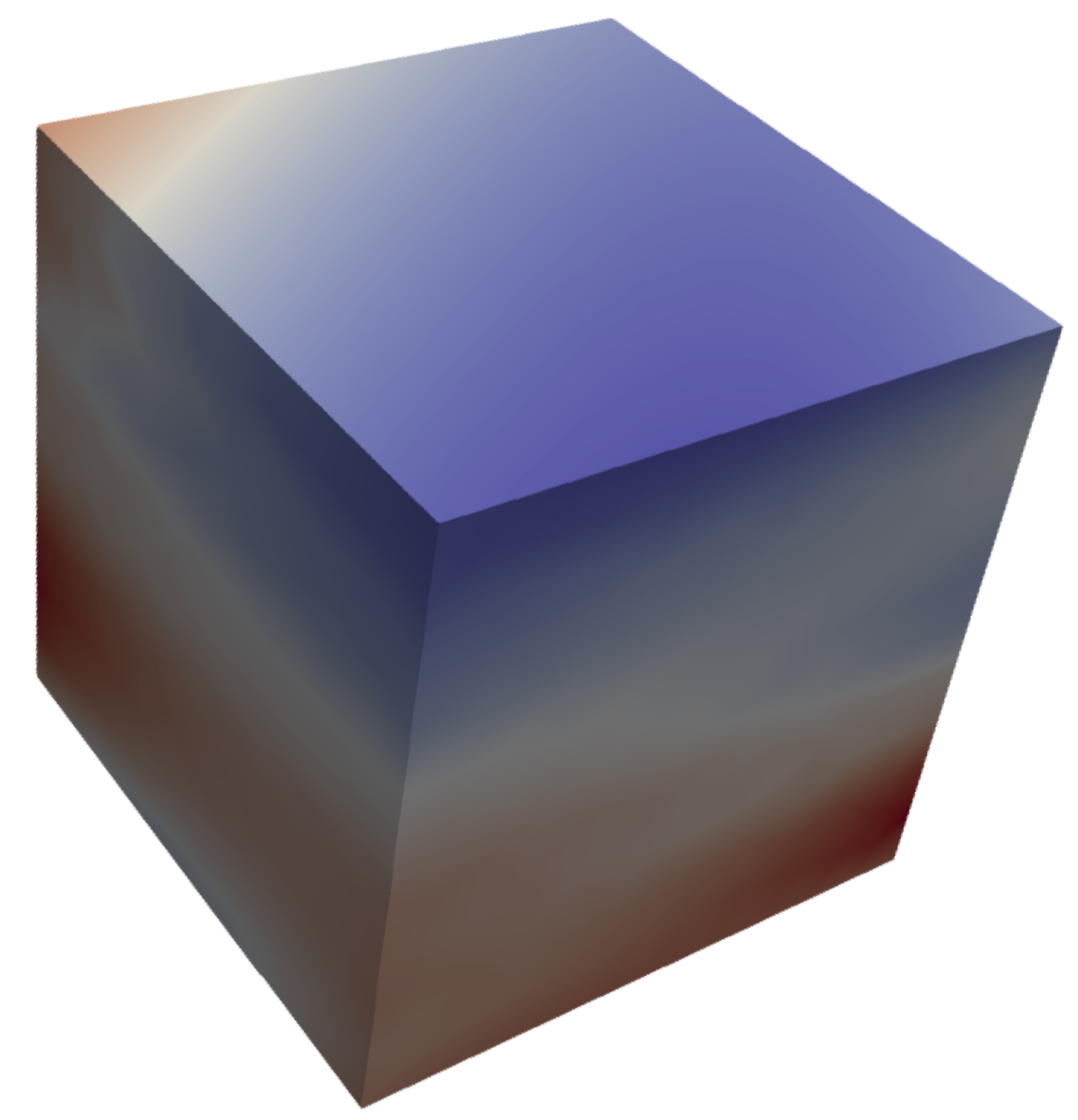}   
		\caption{}\label{k4}
	\end{subfigure}% 
	\caption{$\left(a\right)$ Predicted temperature and $\left(b\right)$ Predicted flux for trigonometric material variation for the functionally graded unit cubic}
	\label{flux_2}
\end{figure}

\begin{figure}[!htb]
	\captionsetup{width=0.9\columnwidth}
	\centering
	\begin{subfigure}[b]{5.0cm}
		\centering\includegraphics[height=5cm,width=5.0cm]{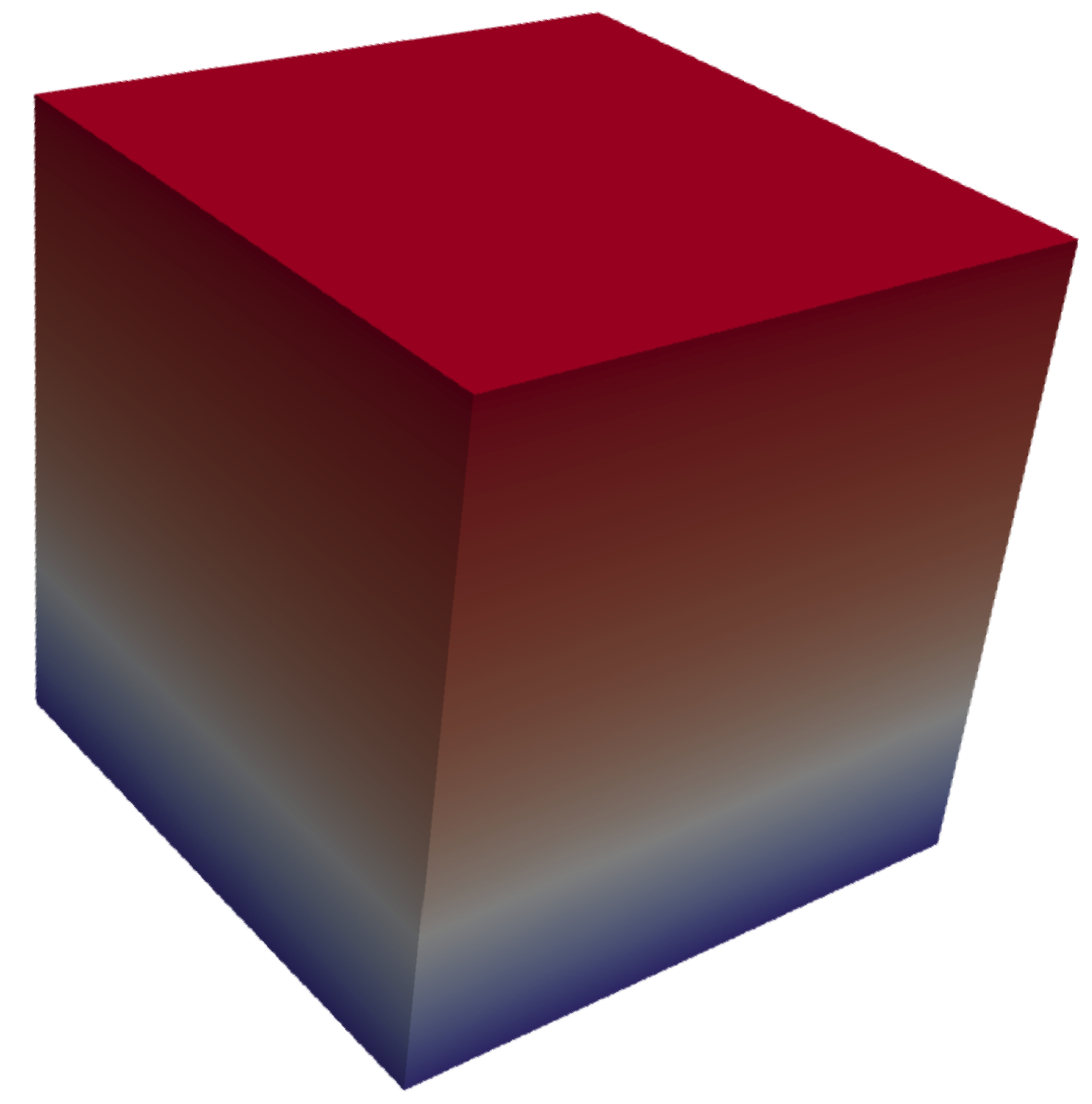}   
		\caption{}\label{k5}
	\end{subfigure}%
	\hspace{0.5cm}
	\begin{subfigure}[b]{5.0cm}
		\centering\includegraphics[height=5cm,width=5.0cm]{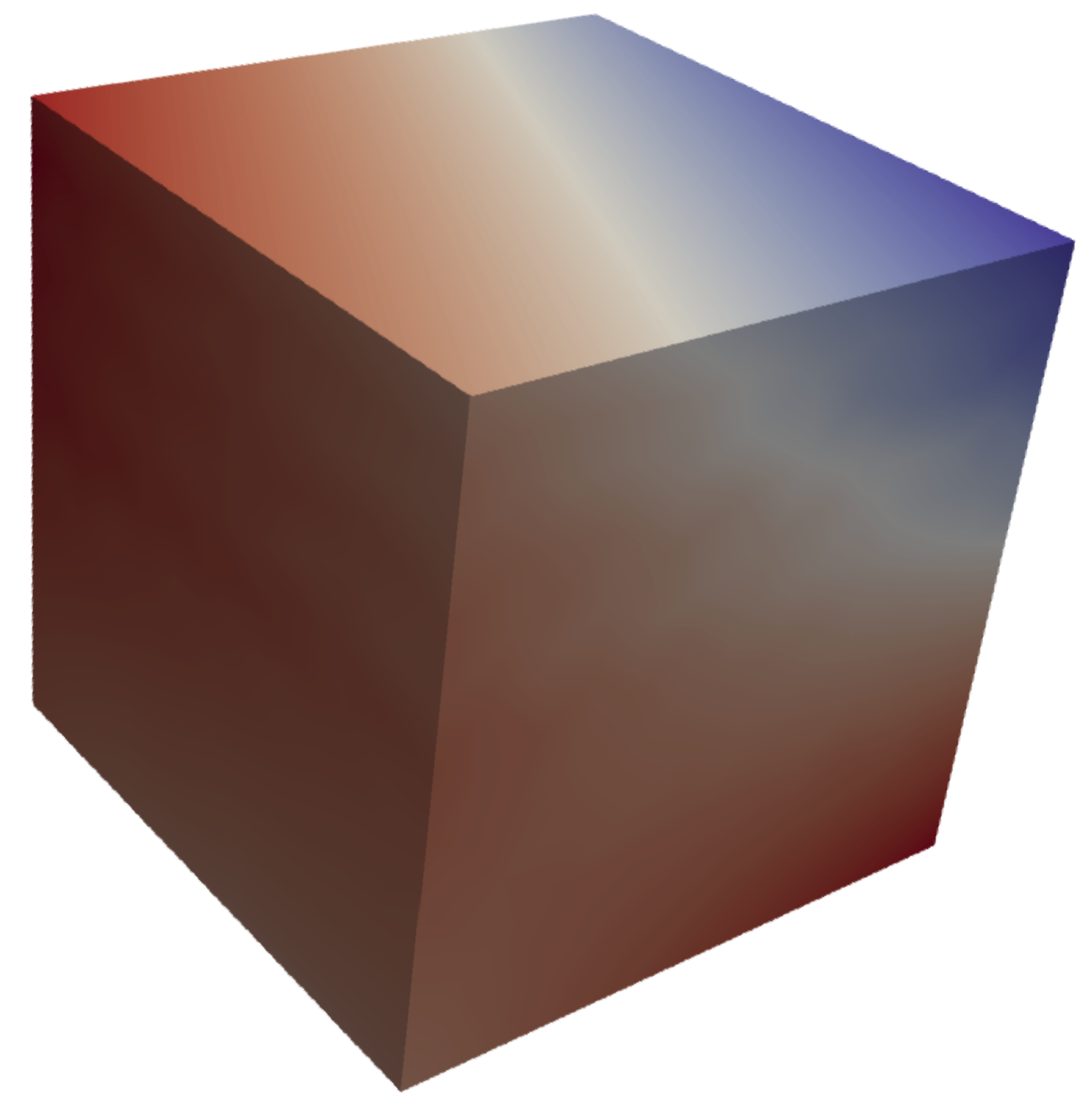}   
		\caption{}\label{k6}
	\end{subfigure}% 
	\caption{$\left(a\right)$ Predicted temperature and $\left(b\right)$ Predicted flux for quadratic material variation for the functionally graded unit cubic}
	\label{flux_3}
\end{figure}
Let us now test the influence of the optimizer on the results. First-order methods minimize the function using its gradient while second-order methods minimize the loss function using the second order derivatives (Hessian information). In this application a combination of these two optimizers is employed. The used first-order method is the Adam algorithm while L-BFGS is the tested second-order method. The convergence history for different optimizers is illustrated in Figure~\ref{cost-op}. Although the first order optimizer can be faster, they require more iterations. The L-BFGS optimizer needs less iterations, but there is the risk in being trapped in local minima. Using the combined optimizers, the loss reaches a significant smaller value with acceptable number of iterations and simultaneously ensures the solution being close to the global minima. 
\begin{figure}[!htb]
	\captionsetup{width=0.9\columnwidth}
	\centering\includegraphics[height=6cm,width=8.0cm]{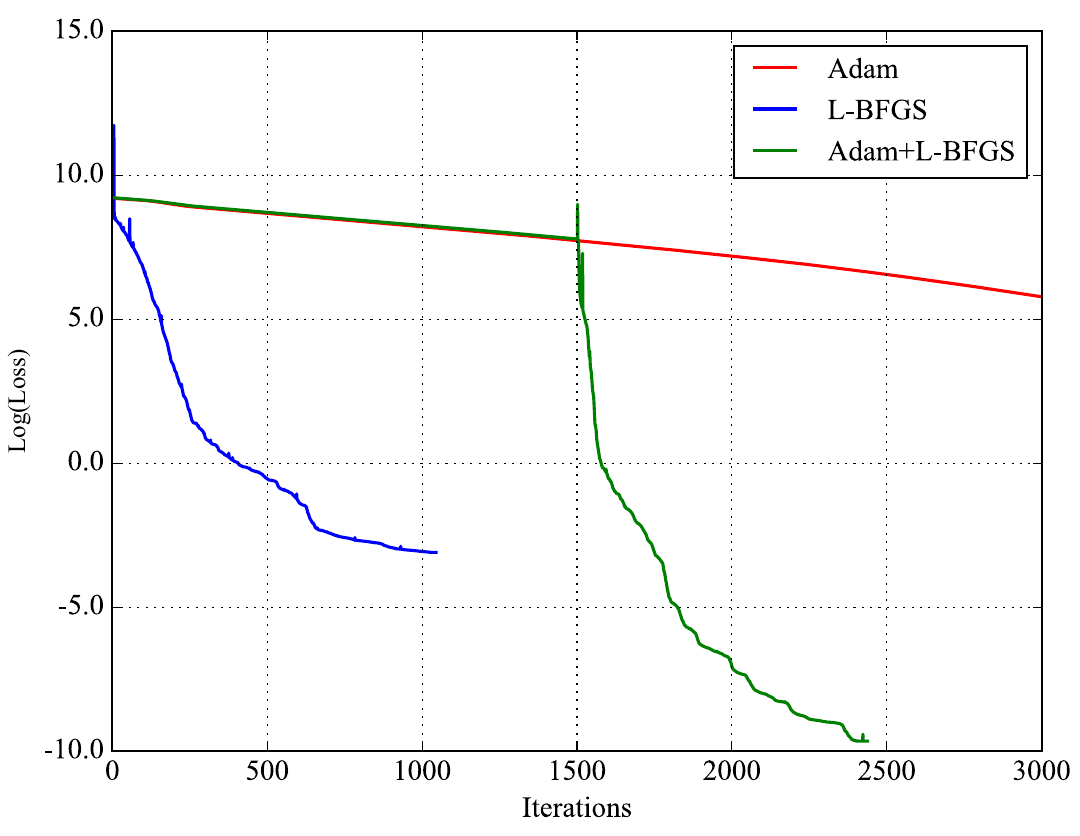}   
	\caption{Comparison of different optimizers for deep collocation method}
	\label{cost-op}
\end{figure}
The results for different number of layers are illustrated in Figure~\ref{cost-layer}.
\begin{figure}[!htb]
	\captionsetup{width=0.9\columnwidth}
	\centering\includegraphics[height=6cm,width=8.0cm]{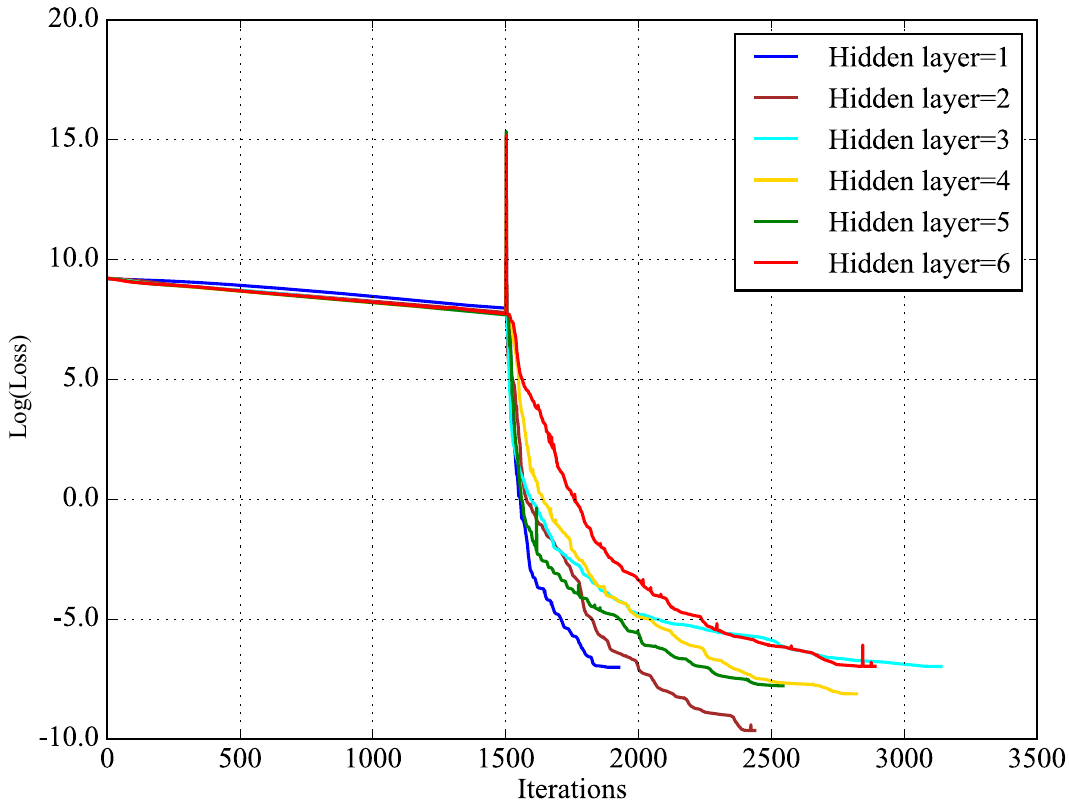}   
	\caption{Convergence graph for DCM with increasing hidden layers}
	\label{cost-layer}
\end{figure}

\subsubsection{Material transfer learning}
The  loss vs number of iterations is shown in Figure \ref{cost-tl}. After funetuning, the loss decreases to a smaller value in less iterations for all three material variations. The numerical results are summarized in Table \ref{tab:Table-tl} demonstrating that the computational effort can be drastically reduced with transfer learning.
\begin{figure}[!htb]
	\captionsetup{width=0.9\columnwidth}
	\centering\includegraphics[height=8cm,width=8.0cm]{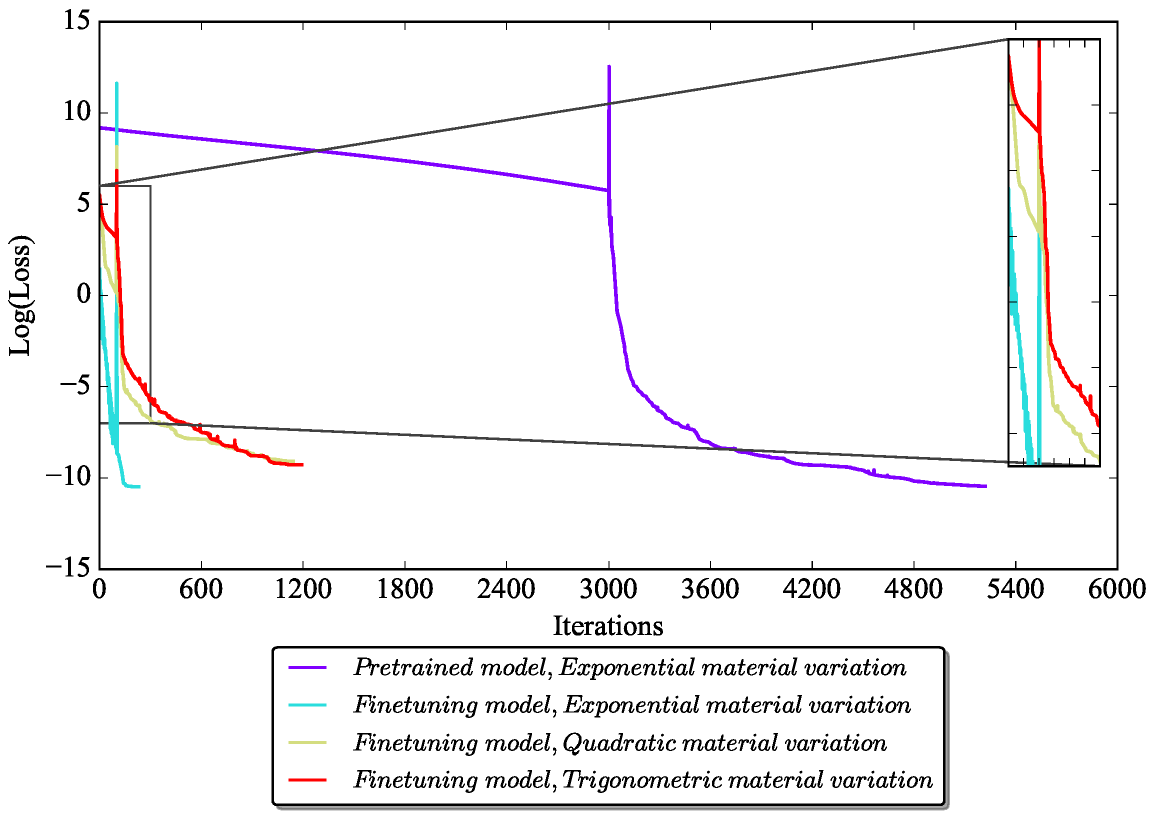} 
	\vspace*{-0.1cm}
	\caption{Loss vs iteration using transfer learning with different material variations}
	\label{cost-tl}
\end{figure}

\begin{table}[!h] 
 \captionsetup{width=0.9\columnwidth}
 \caption{Relative error and training time for material variation with transfer learning} 
 \vspace{-0.1cm}
 \centering 
 \resizebox{0.9\columnwidth}{!}{
  \begin{tabular}{l|c|c|c|c}
   \toprule 
   \toprule 
      \multirow{2}*{\diagbox{Results}{Material variation}}&Exponential &Exponential&Quadratic&Trigonometric\\
        \cline{2-5}
   ~ &without TL&with TL&with TL&with TL\\
   \midrule
   Relative error&4.2846e-06&3.9015e-06&3.7033e-06&3.6562e-06\\ 
   \midrule
   Training time&45.5s&9.1s&22.4s&18.3s\\
   \bottomrule
  \end{tabular}
 }
 \label{tab:Table-tl} % A label for referencing this table elsewhere, references are used in text as \ref{label}
\end{table}

Figure \ref{cost-pr-tl} shows the loss vs iteration using transfer learning for different material parameters while Tables \ref{tab:Table-tl-par} and \ref{tab:Table-tl-par-comp} list the accuracy and CPU time with and without transfer learning.
\begin{figure}[!htb]
	\captionsetup{width=0.9\columnwidth}
	\centering\includegraphics[height=6cm,width=11.0cm]{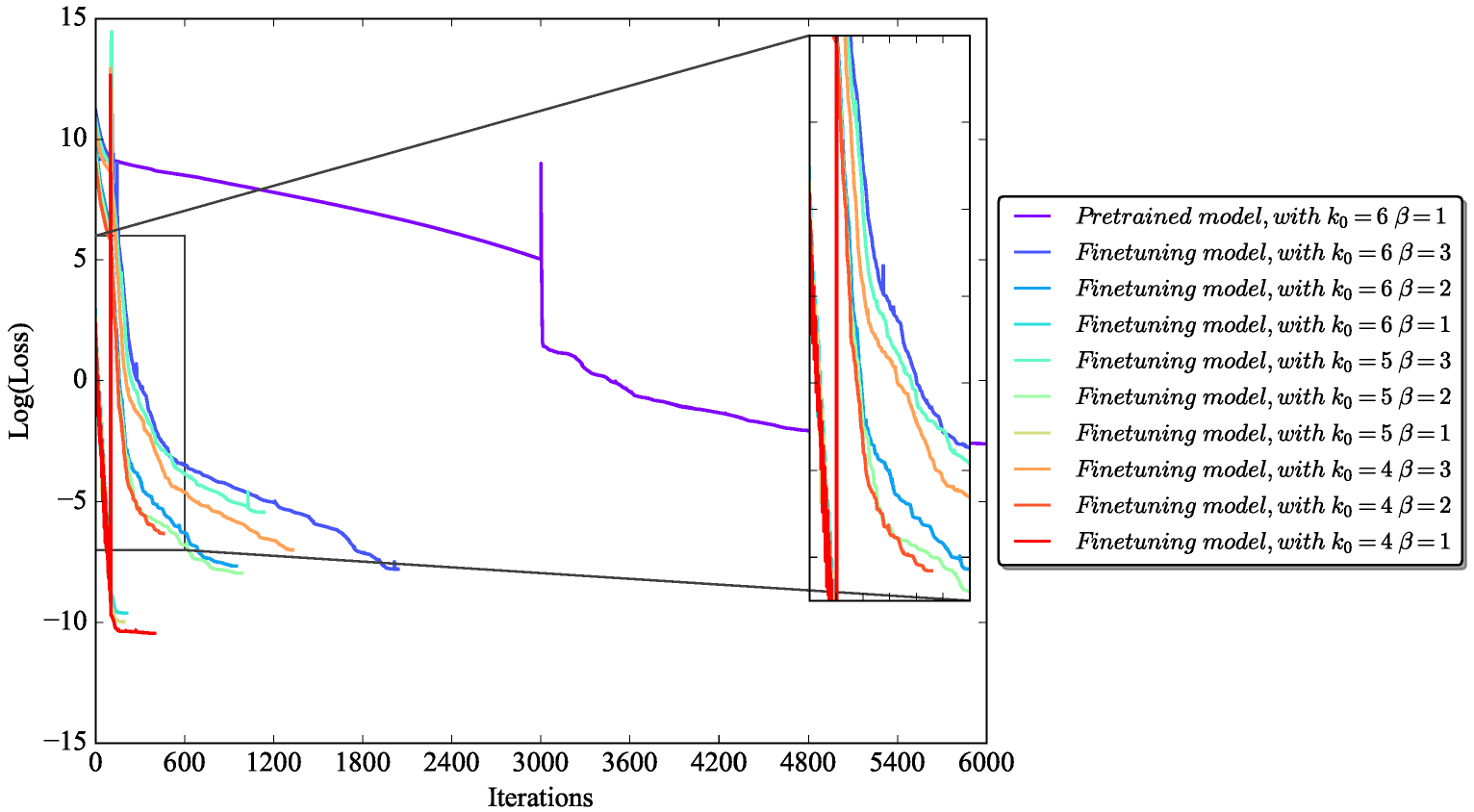} 
	\caption{Loss vs iteration using transfer learning with different material parameters}
	\label{cost-pr-tl}
\end{figure}

\begin{table}[!h]   
	\captionsetup{width=0.9\columnwidth}
	\caption{Relative error of temperature with varying material parameters}
	\vspace{-0.1cm}
	\centering
	\resizebox{0.9\columnwidth}{!}{
		\begin{tabular}{l|c|c|c|c|c|c} 
			\toprule 
			\toprule 
			\multirow{2}*{\diagbox{$k_0$}{$\beta$}}&\multicolumn{2}{c|}{3}&\multicolumn{2}{c|}{2}&\multicolumn{2}{c}{1}\\
			\cline{2-7}
			~ &without TL&with TL&without TL&with TL&without TL&with TL\\  
			\midrule % In-table horizontal line
			6&1.9416e-05&8.6204e-06&1.7244e-05&8.2445e-06&3.8324e-06&3.1974e-06\\
			\midrule
			5&1.8445e-05&1.9075e-05&6.8346e-06&9.9521e-06&1.9358e-06&2.7892e-06\\
			\midrule
			4&1.4026e-05&8.0790e-06&6.8956e-06&1.6358e-05&4.8229e-06&2.3579e-06\\
			\bottomrule
		\end{tabular}
	}
	\label{tab:Table-tl-par}
\end{table}

\begin{table}[!h]   
	\captionsetup{width=0.9\columnwidth}
	\caption{Computation time with varying material parameters (s or sec)}
	\vspace{-0.1cm}
	\centering
	\resizebox{0.9\columnwidth}{!}{
		\begin{tabular}{l|c|c|c|c|c|c} 
			\toprule 
			\toprule 
			\multirow{2}*{\diagbox{$k_0$}{$\beta$}}&\multicolumn{2}{c|}{3}&\multicolumn{2}{c|}{2}&\multicolumn{2}{c}{1}\\
			\cline{2-7}
			~ &without TL&with TL&without TL&with TL&without TL&with TL\\  
			\midrule % In-table horizontal line
			6&6.9715e+01&1.8325e+01&4.6163e+01&1.0890e+01&4.7989e+01&5.4962e+00\\
			\midrule
			5&5.6954e+01&1.2305e+01&4.0428e+01&1.0409e+01&4.3479e+01&4.9966e+00\\
			\midrule
			4&6.4699e+01&1.3876e+01&5.3908e+01&6.9735e+00&3.8583e+01&6.1923e+00\\
			\bottomrule
		\end{tabular}
	}
	\label{tab:Table-tl-par-comp}
\end{table}

\subsection{Case 3: Cube with a 3D material gradation}

Now, we consider a cube with the following three-dimensional thermal conductivity variation:
\begin{equation}
k(x,y,z)={ (5+0.2x+0.4y+0.6z+0.1xy+0.2yz+0.3zx+0.7xyz) }^{ 2 }
\label{conductivity variation}
\end{equation} 

The iso-surfaces of the 3D variation of the thermal conductivity is illustrated in Figure~\ref{thermal conductivity 3D}. The analytical solution for this variation is
\begin{equation}
\phi (x,y,z)={ \frac { xyz }{ (5+0.2x+0.4y+0.6z+0.1xy+0.2yz+0.3zx+0.7xyz) }  }
\label{function of variation}
\end{equation} 
\begin{figure}[!htb]
	\captionsetup{width=0.9\columnwidth}
	\centering\includegraphics[height=6cm,width=6cm]{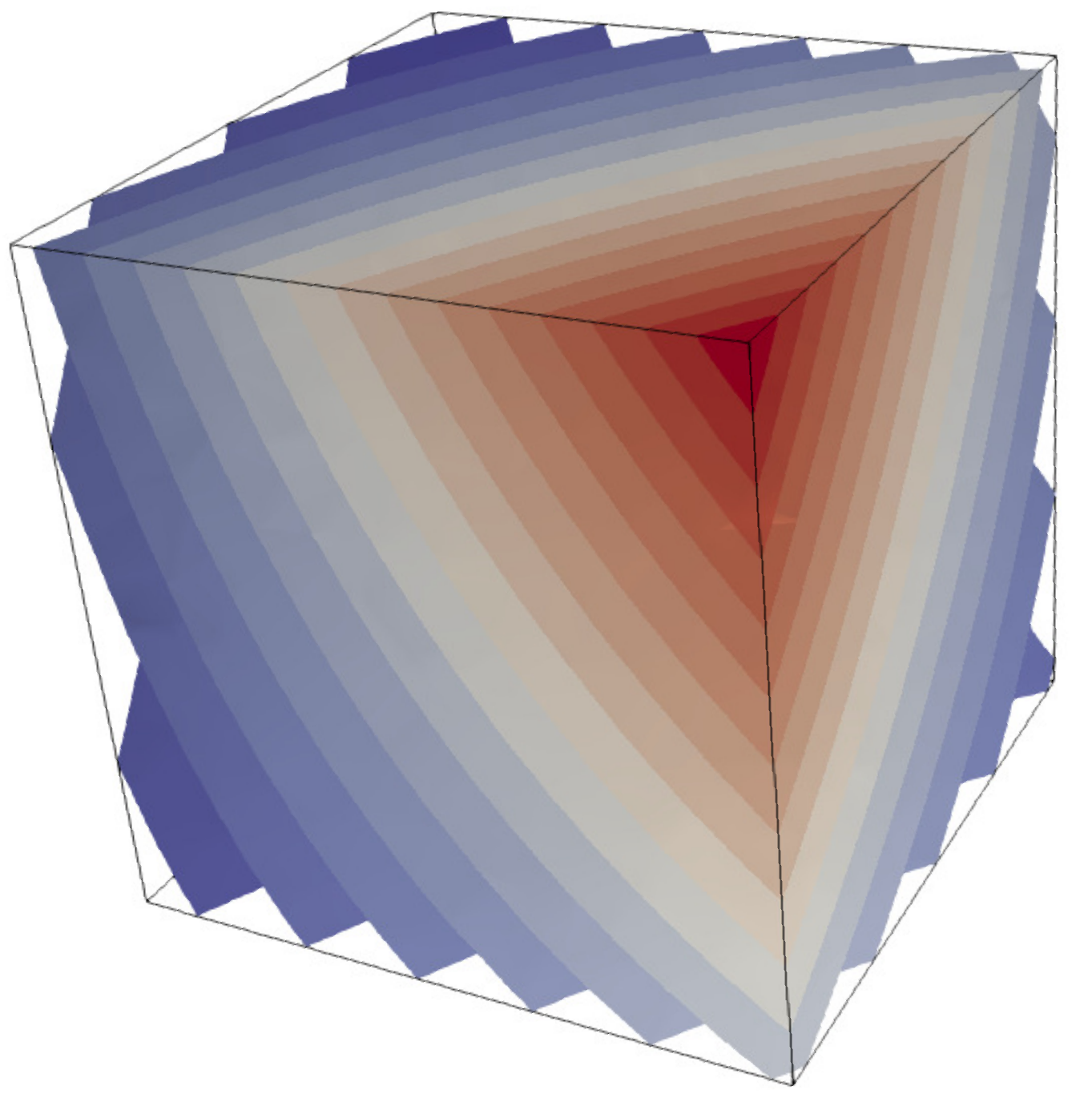}   
	\caption{Representation of iso-surfaces for the three-dimensional variation of thermal conductivity k(x,y,z)}
	\label{thermal conductivity 3D}
\end{figure}
The boundary conditions at the six faces of the cube are listed in Table~\ref{tab:Table6}.
\begin{table}[!h] 
	\captionsetup{width=0.9\columnwidth}
	\caption{The boundary conditons of cube with a 3D material gradation} 
	\vspace{-0.1cm}
	\centering 
	\resizebox{0.5\columnwidth}{!}{%
		\begin{tabular}{ c|c} 
			\toprule % Top horizontal line
			\toprule % Top horizontal line
			\multicolumn{2}{c}{\textbf{Boundary condition}} \\ 
			\midrule
			Dirichlet & Neumann\\ % Column names row
			\midrule % In-table horizontal line
			$\phi (0,y,z)=0$ &  $q(1,y,z)=-0.2zy(25+2y+3z+zy)$\\ % Content row 1
			\midrule
			$\phi (x,0,z)=0$ &  $q(x,1,z)=-0.1xz(50+2x+6z+3xz)$\\ % Content row 2
			\midrule
			$\phi (x,y,0)=0$ &  $q(x,y,1)=-0.1xy(50+2x+4y+xy)$\\ % Content row 3
			\bottomrule % Bottom horizontal line
		\end{tabular}
    }              
	\label{tab:Table6} 
\end{table}

The predicted temperature and flux distributions  are shown in Figure ~\ref{temperature and flux of 3D}. The predicted relative error of the temperature across the cube is 5.215360e-03, see also Figure~\ref{temperature line 3D}. 
\begin{figure}[!htb]
	\captionsetup{width=0.9\columnwidth}
	\centering
	\begin{subfigure}[b]{5.0cm}
		\centering\includegraphics[height=5cm,width=5.0cm]{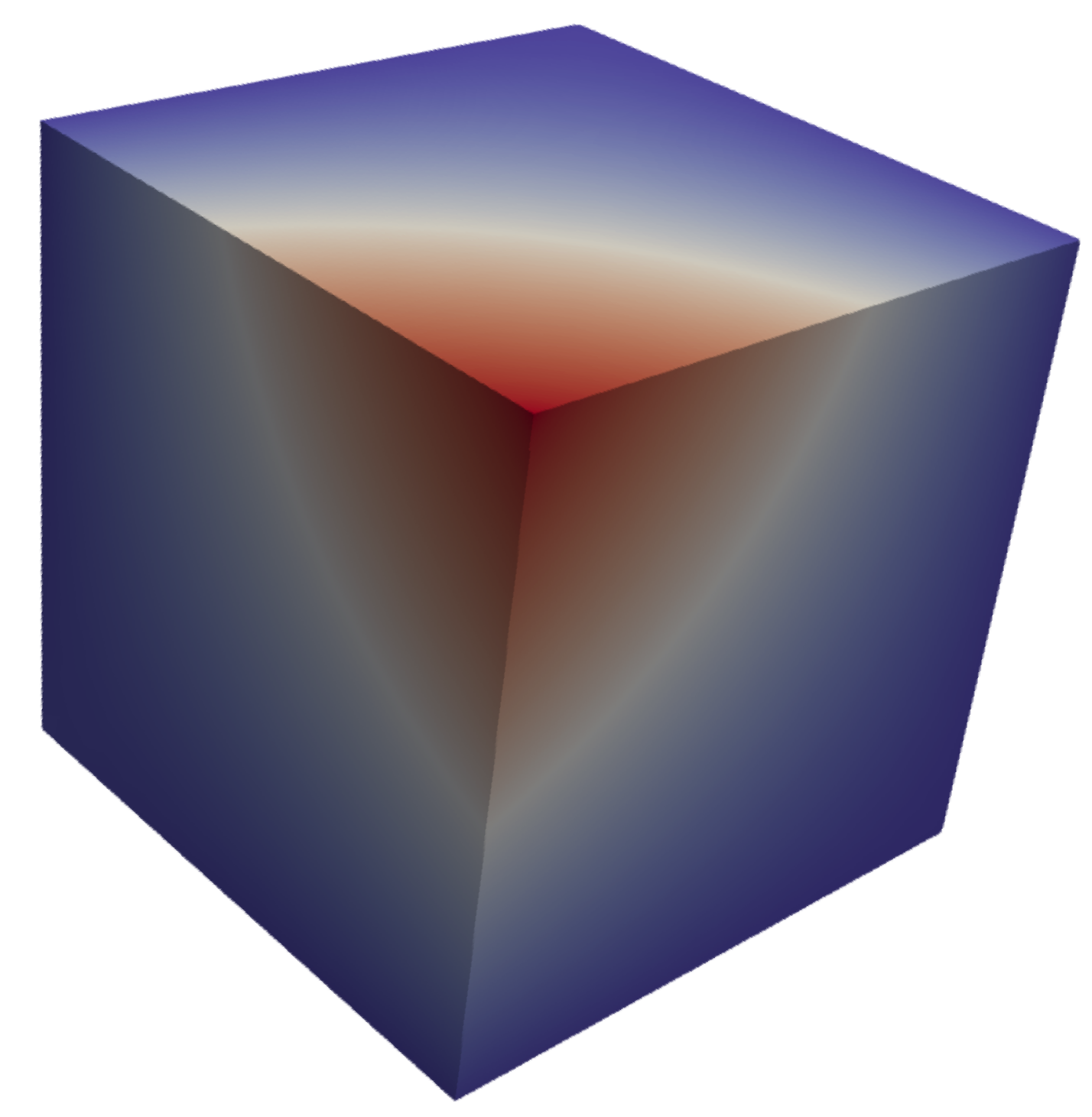}   
		\caption{}\label{d1}
	\end{subfigure}%
	\hspace{0.5cm}
	\begin{subfigure}[b]{5.0cm}
		\centering\includegraphics[height=5cm,width=5.0cm]{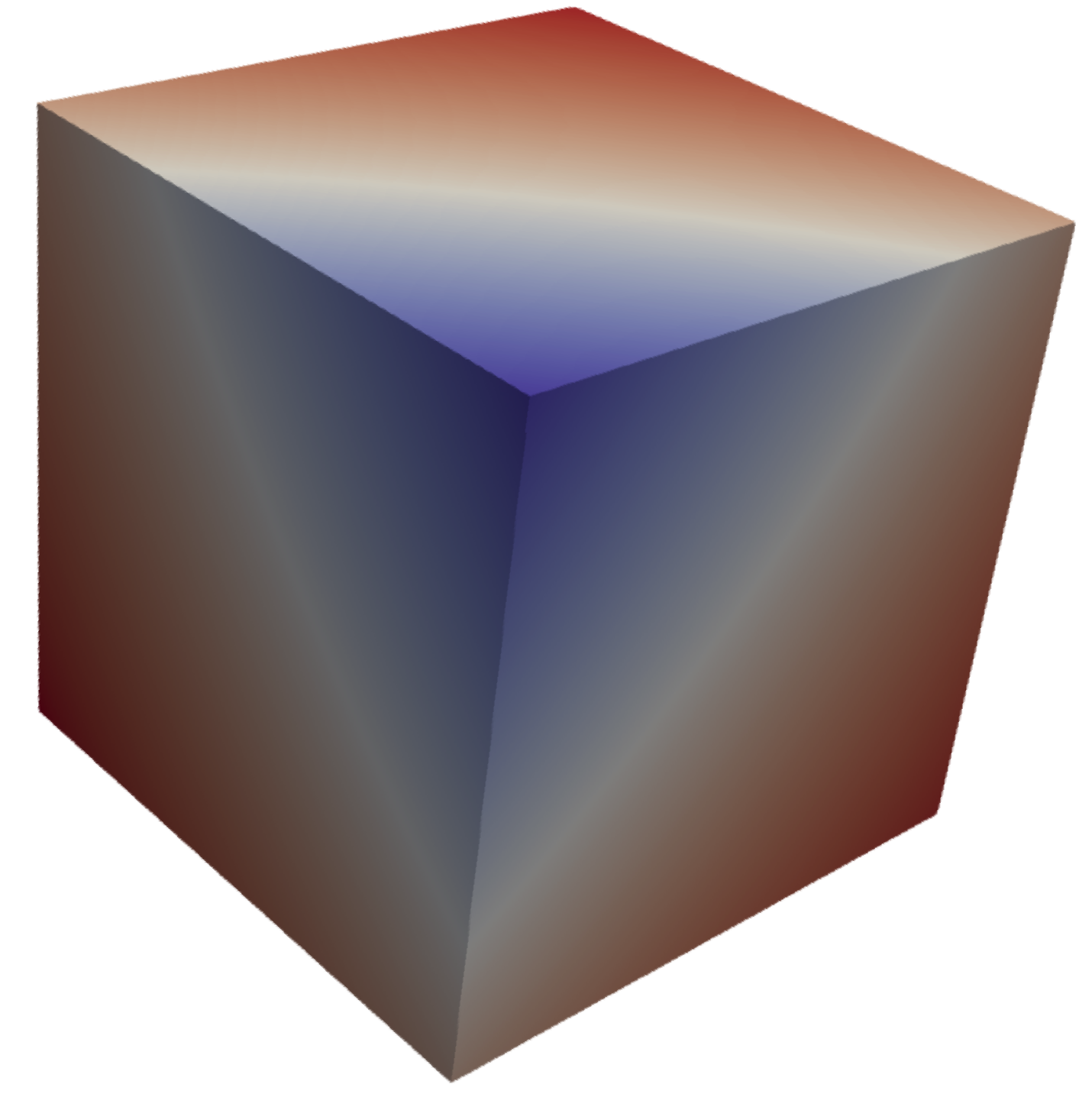}   
		\caption{}\label{d2}
	\end{subfigure}% 
	\caption{Distributions of $\left(a\right)$ temperature; and $\left(b\right) $flux for the cube with 3D material variation }
	\label{temperature and flux of 3D}
\end{figure}

\begin{figure}[!htb]
	\captionsetup{width=0.9\columnwidth}
	\centering\includegraphics[height=6cm,width=8cm]{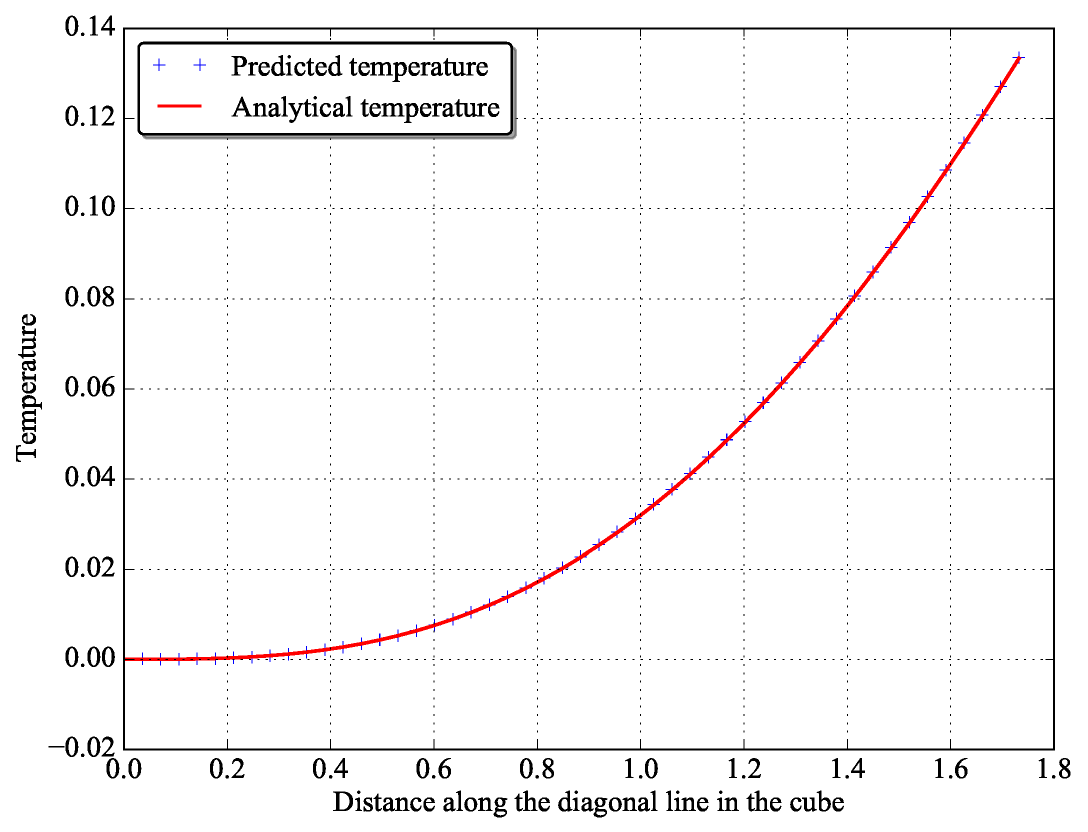}   
	\caption{Temperature profile in diagonal line for the cube with a 3D material gradation}
	\label{temperature line 3D}
\end{figure}
\subsection{Case 4: Irregular shaped annular sector}
Next, we present results for an irregular shaped annular sector as depicted in Figure \ref{thermalconditions}. The inner radius is 0.3, outer radius 0.5, the top surface at Z=0.1  and the thermal conductivity for the geometry varies exponentially according to 
\begin{equation}
k(z)=5{ e }^{ (3z) }
\label{irregular conductivity}
\end{equation} 
The variation of the thermal conductivity k(z)  is illustrated in Figure~\ref{thermalconductivityirregular}.
\begin{figure}[!htb]
	\captionsetup{width=0.9\columnwidth}
	\centering\includegraphics[height=6cm,width=8.0cm]{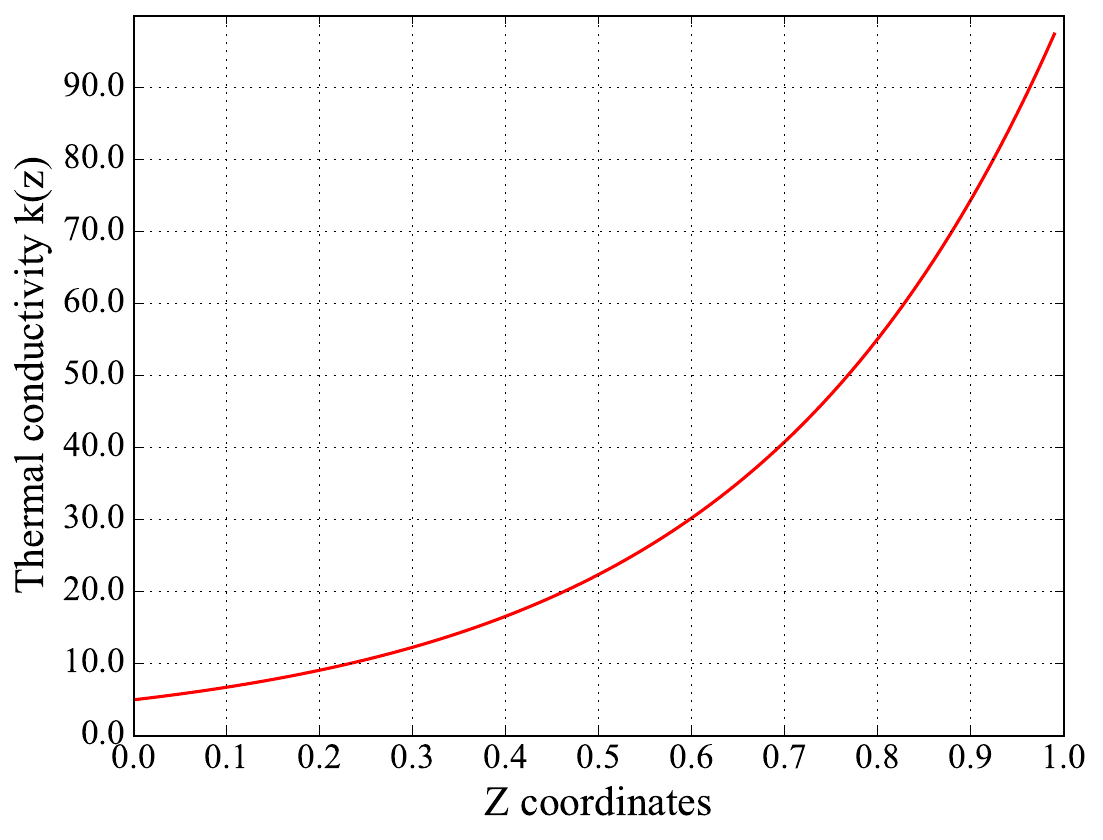}   
	\caption{Profile of thermal conductivity in z direction. The exponential variation of the conductivity is $k(z)=5{ e }^{ (3z) }$}
	\label{thermalconductivityirregular}
\end{figure}
The temperature is specified along the inner radius as ${T}_{inner}=0$, and outer radius as ${T}_{outer}=100$, and all other surfaces are insulated.  The boundary conditions of the geometry are shown in Figure ~\ref{thermalconditions}.
\begin{figure}[!htb]
	\captionsetup{width=0.9\columnwidth}
	\centering\includegraphics[height=6cm,width=8.0cm]{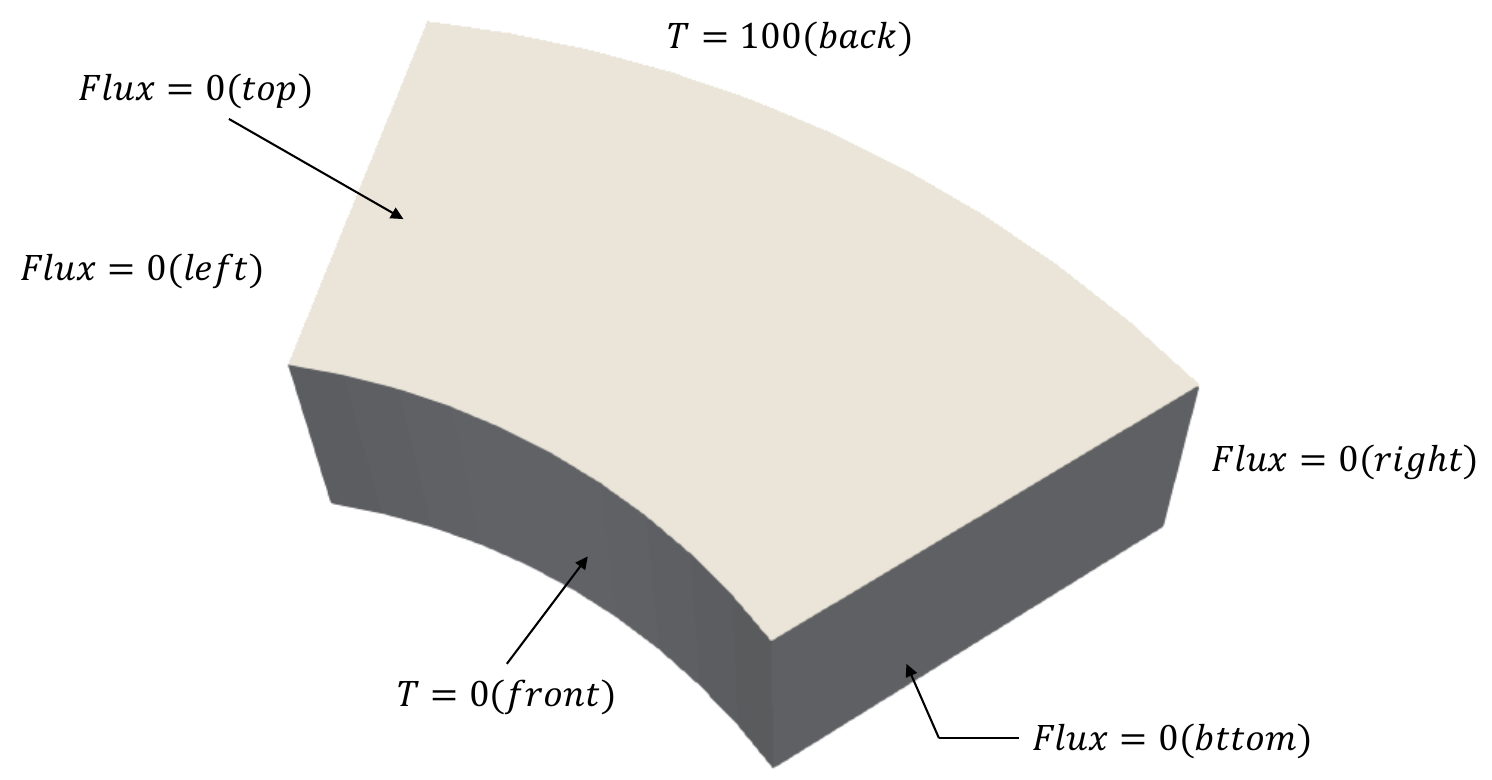}   
	\caption{Annular sector subjected to thermal boundary conditions}
	\label{thermalconditions}
\end{figure}

The results of the predicted temperature are shown in Figure~\ref{rotor-temp}-\ref{rotor-flux} and compared to an FEM solution of the commercial software package ABAQUS,  as no analytical solution is available for this problem. 
\begin{figure}[!htb]
	\captionsetup{width=0.9\columnwidth}
	\centering
	\begin{subfigure}[b]{5.0cm}
		\centering\includegraphics[height=5cm,width=5.0cm]{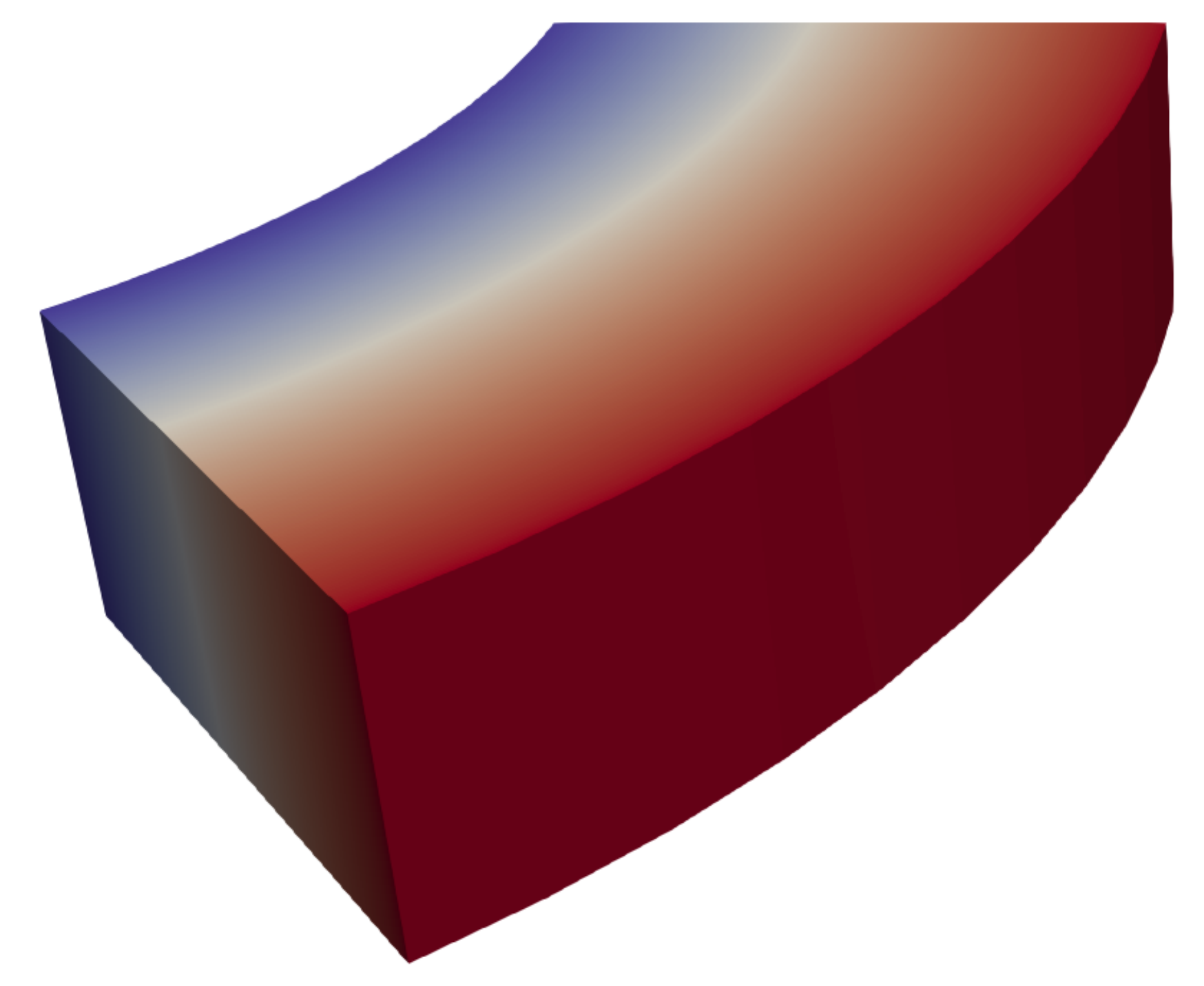}   
		\caption{}\label{s1}
	\end{subfigure}%
	\hspace{0.5cm}
	\begin{subfigure}[b]{5.0cm}
		\centering\includegraphics[height=5cm,width=5.0cm]{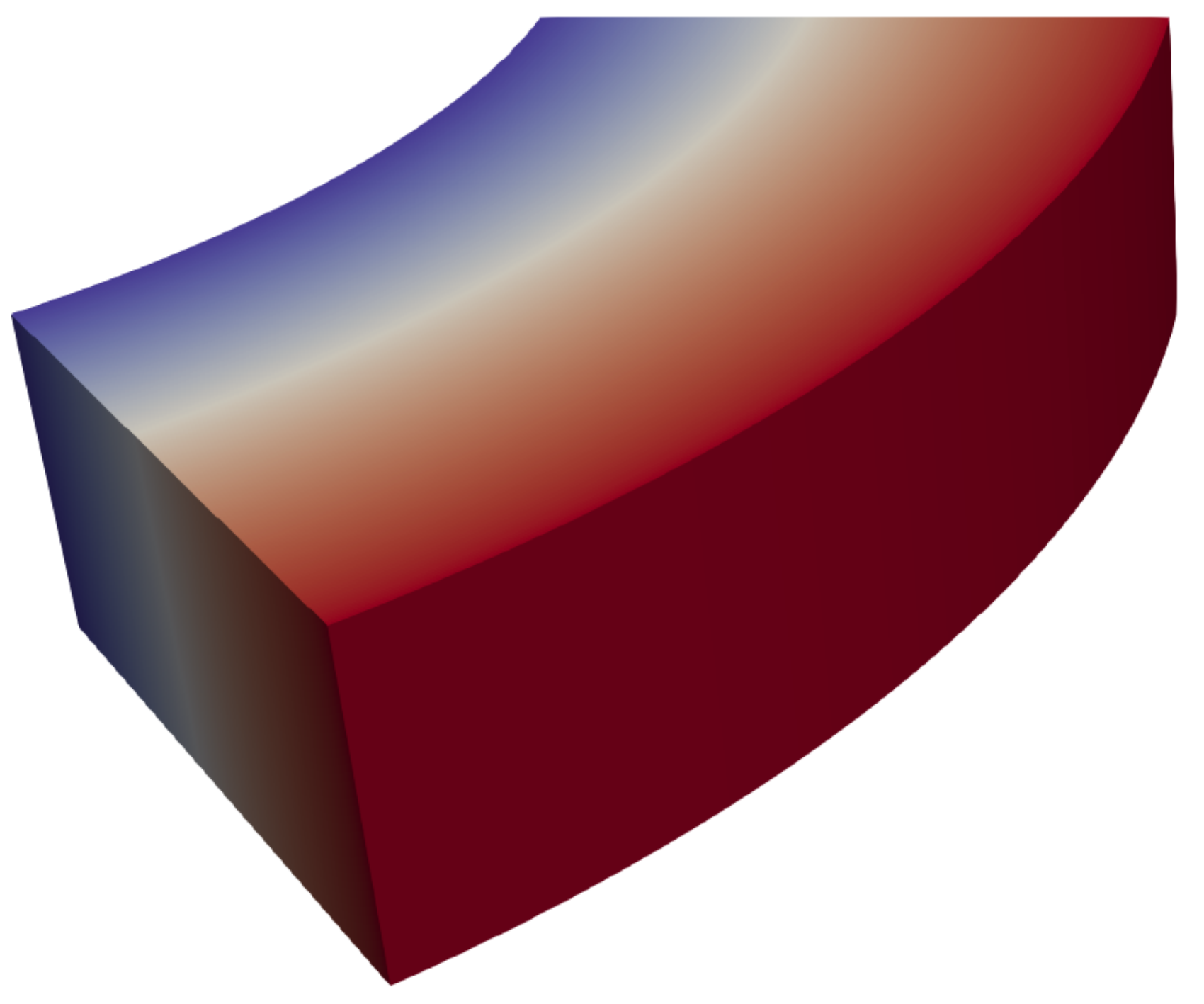}   
		\caption{}\label{s2}
	\end{subfigure}% 
	\caption{ Temperature distribution for irregular shaped FGMs obtained by $\left(a\right)$ deep collocation method; and $\left(b\right)$ ABAQUS}
	\label{rotor-temp}
\end{figure}
\begin{figure}[!htb]
	\captionsetup{width=0.9\columnwidth}
	\centering
	\begin{subfigure}[b]{5.0cm}
		\centering\includegraphics[height=5cm,width=5.0cm]{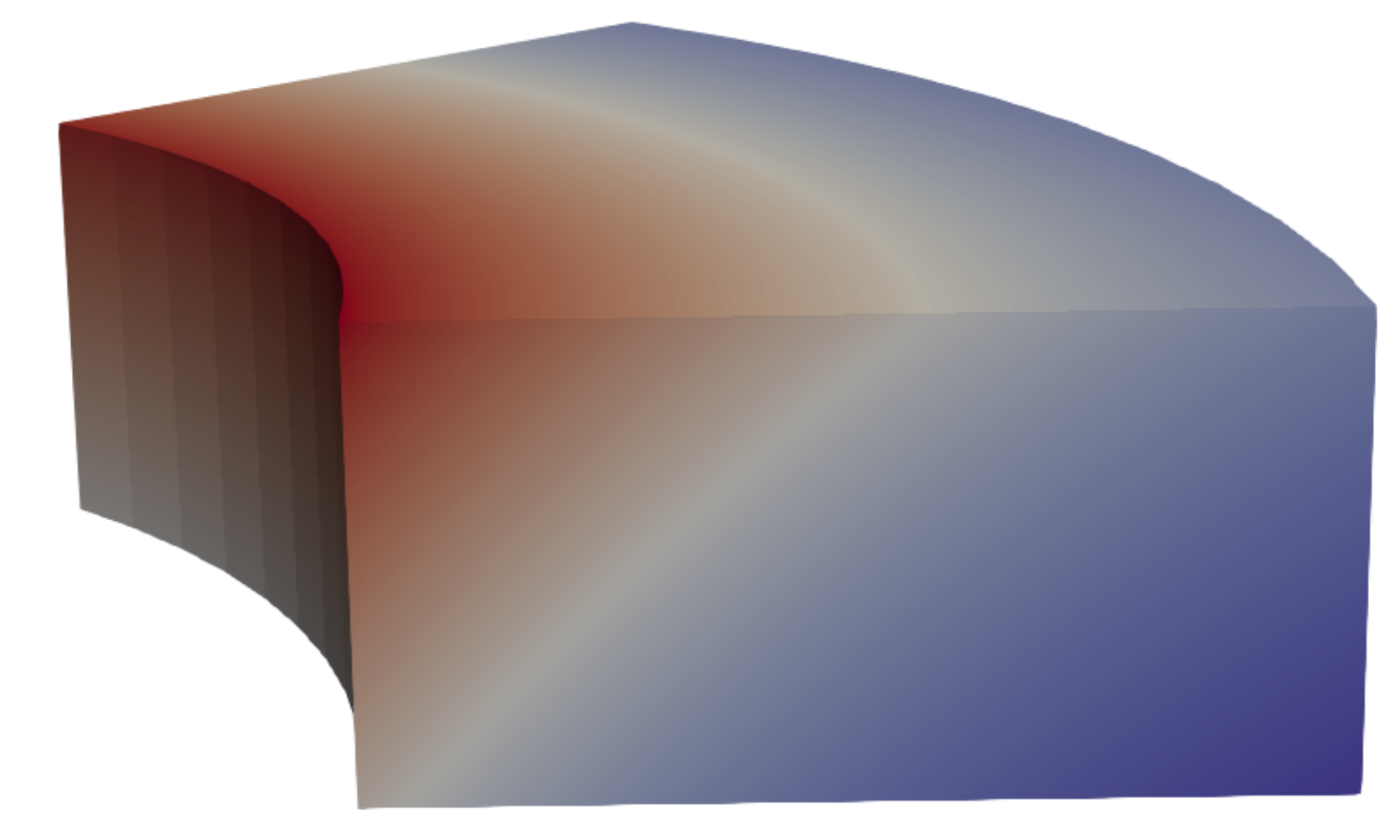}   
		\caption{}\label{s3}
	\end{subfigure}%
	\hspace{0.5cm}
	\begin{subfigure}[b]{5.0cm}
		\centering\includegraphics[height=5cm,width=5.0cm]{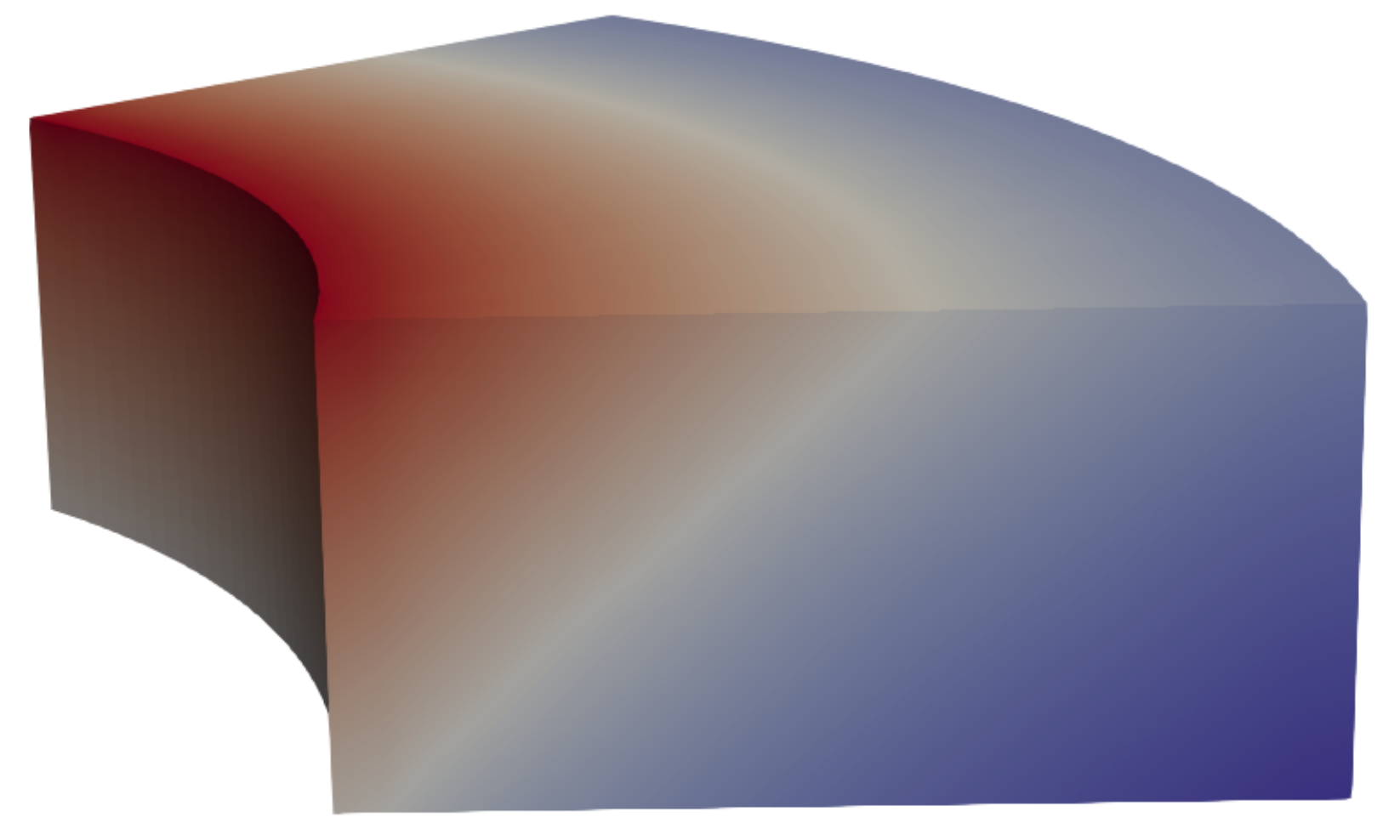}   
		\caption{}\label{s4}
	\end{subfigure}% 
	\caption{Flux distribution for irregular shaped FGMs obtained by $\left(a\right)$ deep collocation method; and $\left(b\right)$ ABAQUS }
	\label{rotor-flux}
\end{figure}
The temperature along the radial direction at the edge is plotted and compared with results obtained by ABAQUS in Figure~\ref{temperature line rotor}.
\begin{figure}[!htb]
	\captionsetup{width=0.9\columnwidth}
	\centering\includegraphics[height=6cm,width=8cm]{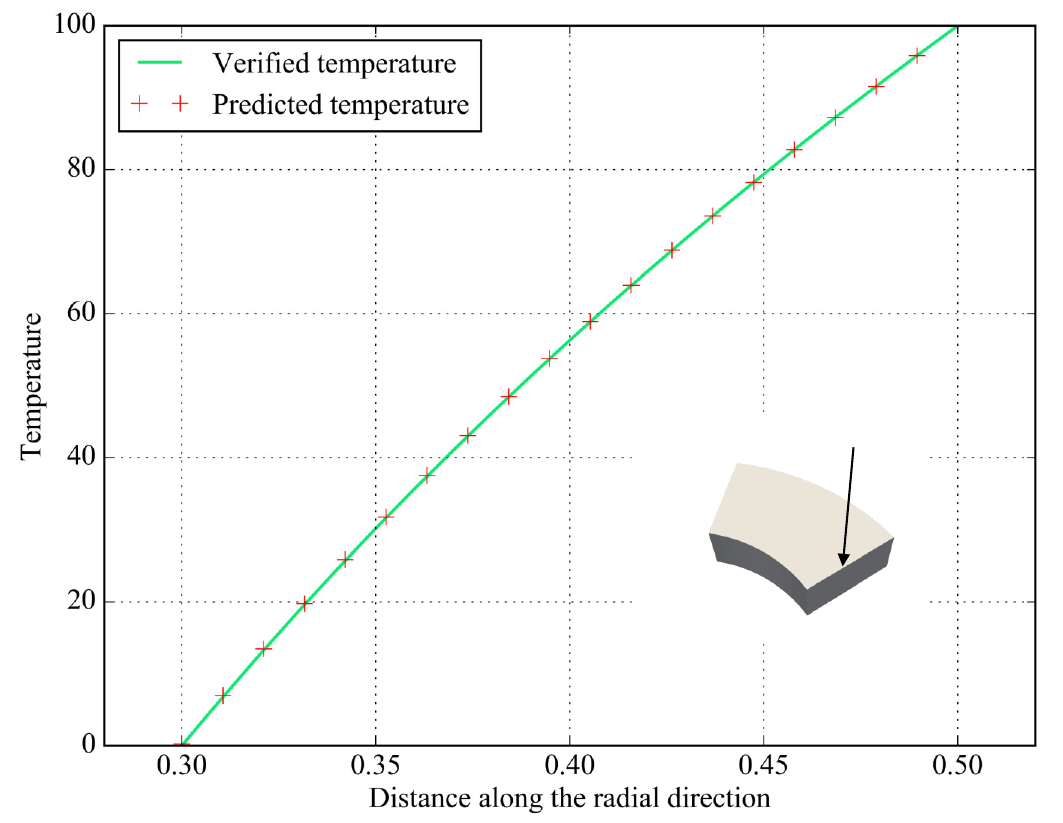}   
	\caption{Temperature comparison along the right top edge (indicated by the arrow)}
	\label{temperature line rotor}
\end{figure}

\section{Conclusion}
We presented a transfer learning based deep collocation method (DCM) for solving the problems of potential in non-homogeneous media. It avoids classical discretization methods such as FEM and treats the problem as minimzation problem, minimizing a loss function which is related to the underlying governing equation. Thanks to the nonlinear activation function, the approach enables us to discover complex nonlinear pattern.  The DCM requires sampling inside the physical domain. Therefore, we obtained a suitable sampling method for selected problems. In order to find the most favorable configuration of the neural network for specific problems, we carried out a sensitivity analysis quantifying the influence of algorithm-specific parameters on specific outputs such as the relative error in the L2 norm. For different material variation forms and material parameters, a material transfer learning is embedded into the framework to enhance the robustness and generality of this deep collocation method. To demonstrate the performance of the proposed DCM, various benchmark problems including the heat transfer and a representative potential problem are studied, which has been verified with effectiveness and accuracy. Even though DCM is quite promising in solving partial differential equations, there are limitations remains to be solved, such as lack of a more systematic procedure in preventing overfitting issues for physics informed deep learning model, the local minima issues and DCM still remains to be enhanced in the learning ability and efficiency for complex multi-physics and multi-scale modelling, which will be our major research topics in the future.

\bmhead{Acknowledgments}\label{Acknowledgement}
The authors extend their appreciation to the Distinguished Scientist Fellowship Program (DSFP) at King Saud University for funding this work.

\begin{appendices}

\section{Data flow for this study}\label{appendix a}

The data flow inside all modules in this investigation are depicted in Figure \ref{fig:Dataflow}, which is mainly classified into three modules. First the pivotal part of this study is the physics-informed deep learning based collocation methods, with collocation points generated for modelling, which makes the deep collocation method truly `meshfree'. The the physics-informed deep learning is built to incorporated the physics information into the state-of-the-art deep learning model. After that, we did a parametric study on the influence of algorithm-specific parameters including numbers of collocation points and parameters for deep learning configurations on predictive accuracy, in hope to give general guidance to further applications of deep collocation method in computational mechanics. For the parametric analysis, a two-step Morris screening method and EFAST method are adapted to give a qualitative and quantitive measure of importance and interaction of DCM-specific parameters. To facilitate and improve the generality and robustness of the presented model, data is finally imported into a material transfer learning model to transfer and expand learnt knowledge between different material variations.

\begin{figure}[h]%
\centering
\includegraphics[width=1\textwidth,height=0.65\textwidth]{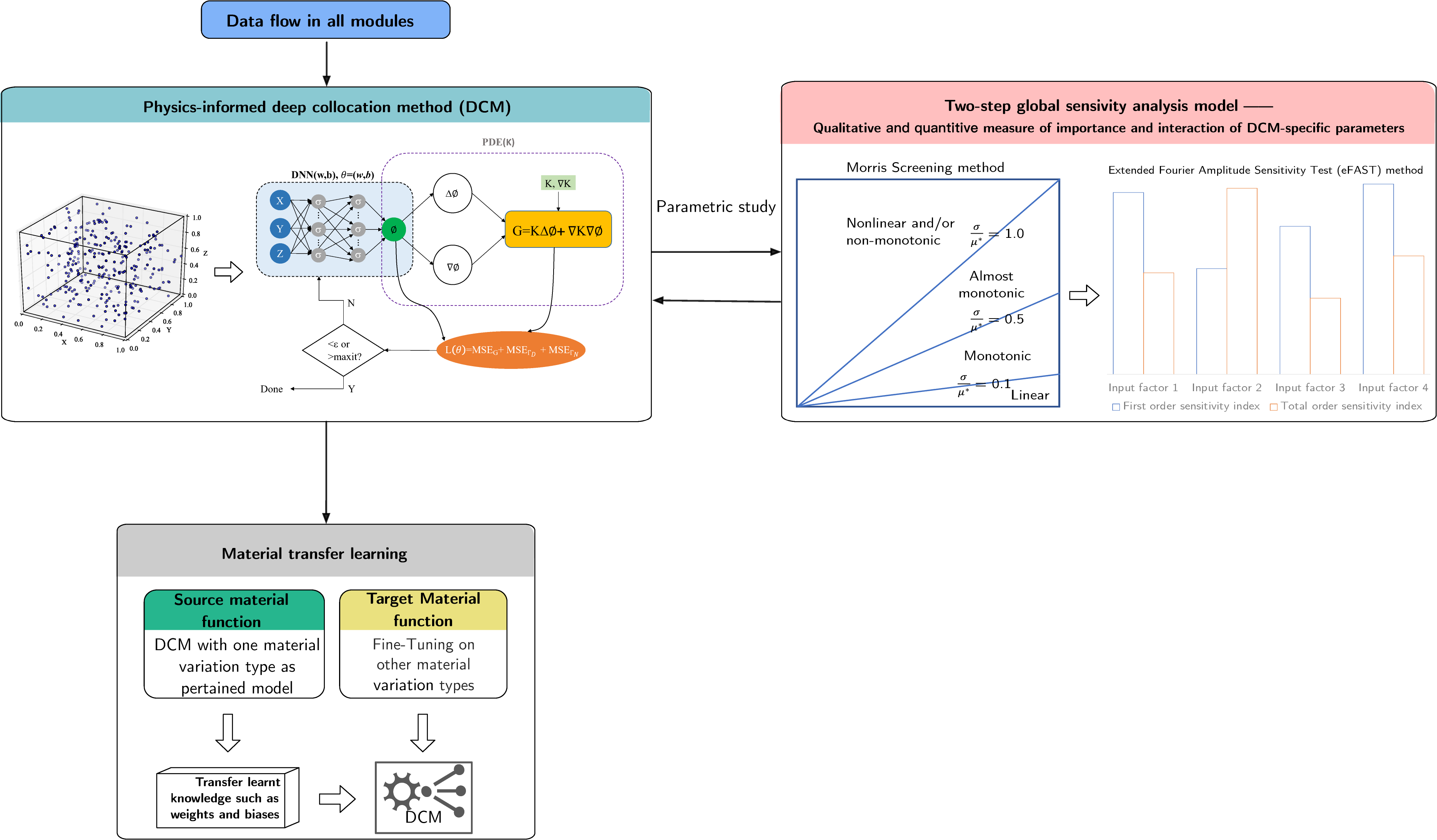}
\centering
\caption{Data flow in physics-informed deep collocation method and sensitivity analysis}\label{fig:Dataflow}
\end{figure}

\section{Activation function and sampling method for comparison}\label{appendix b}
The following table shows a list of classical activation functions and its graphs that studied in this application, which will help to choose a suitable activation for physics-informed neural networks.
%%% complete a table of activation function and its figure
\begin{table}[!h] % Add the following just after the closing bracket on this line to specify a position for the table on the page: [!htb], [t], [b] or [p] - these mean: here, top, bottom and on a separate page, respectively
\captionsetup{width=0.9\columnwidth}
\caption{Activation function} % Table caption, can be commented out if no caption is required
\vspace{-0.1cm}
\centering % Centres the table on the page, comment out to left-justify
\resizebox{0.9\columnwidth}{!}{%
\begin{tabular}{c |c | m{5cm}<{\centering} | m{7cm}<{\centering}} % The final bracket specifies the number of columns in the table along with left and right borders which are specified using vertical pipes (|); each column can be left, right or center-justified using l, r or c. Columns will widen to hold the content in them by default, to specify a precise width, use p{width}, e.g. p{5cm}
\toprule % Top horizontal line
\toprule % Top horizontal line
Activation function & Explicit function form& Function figure& Derivatives of function figure\\ % Column names row
\midrule % In-table horizontal line
Tanh & $f(x)=\frac{e^{2x}-1}{e^{2x}+1}$ & \includegraphics[height=4cm,width=5cm]{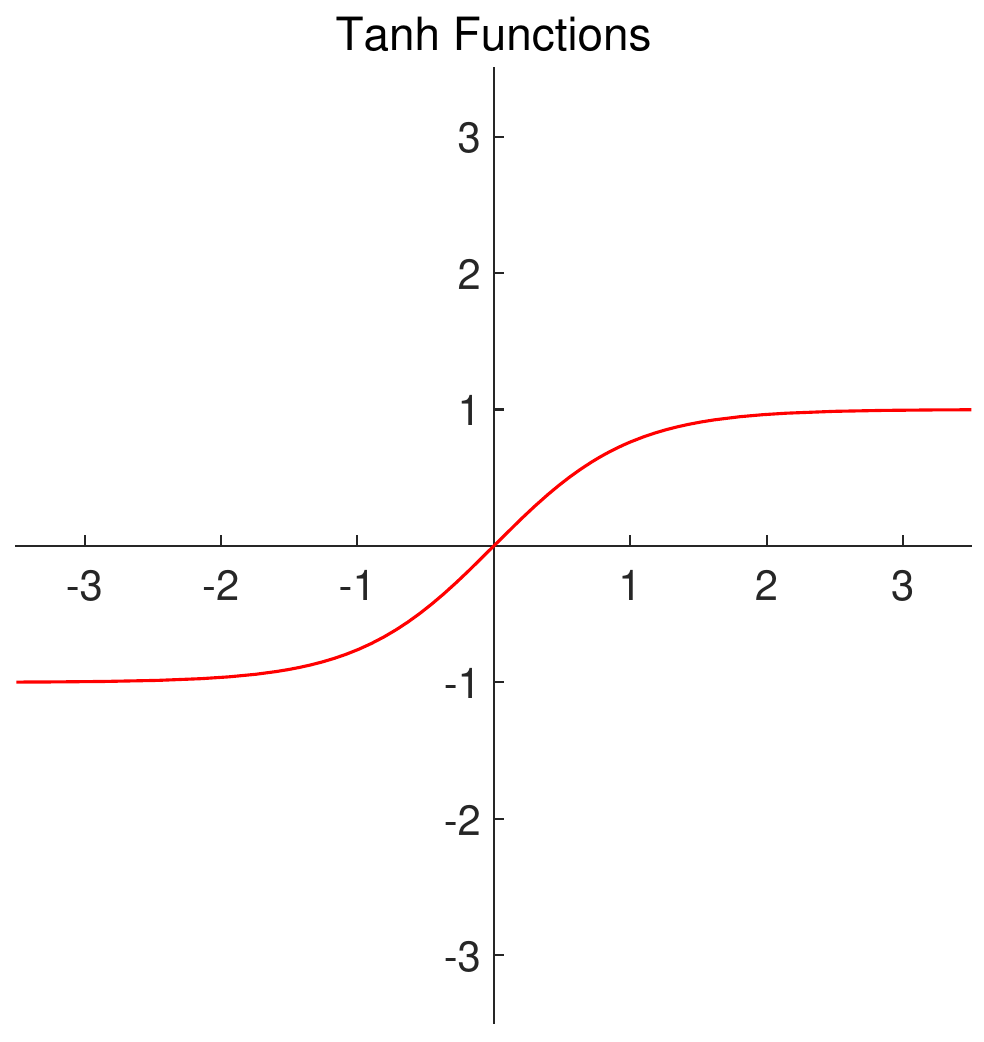} & \includegraphics[height=4cm,width=5cm]{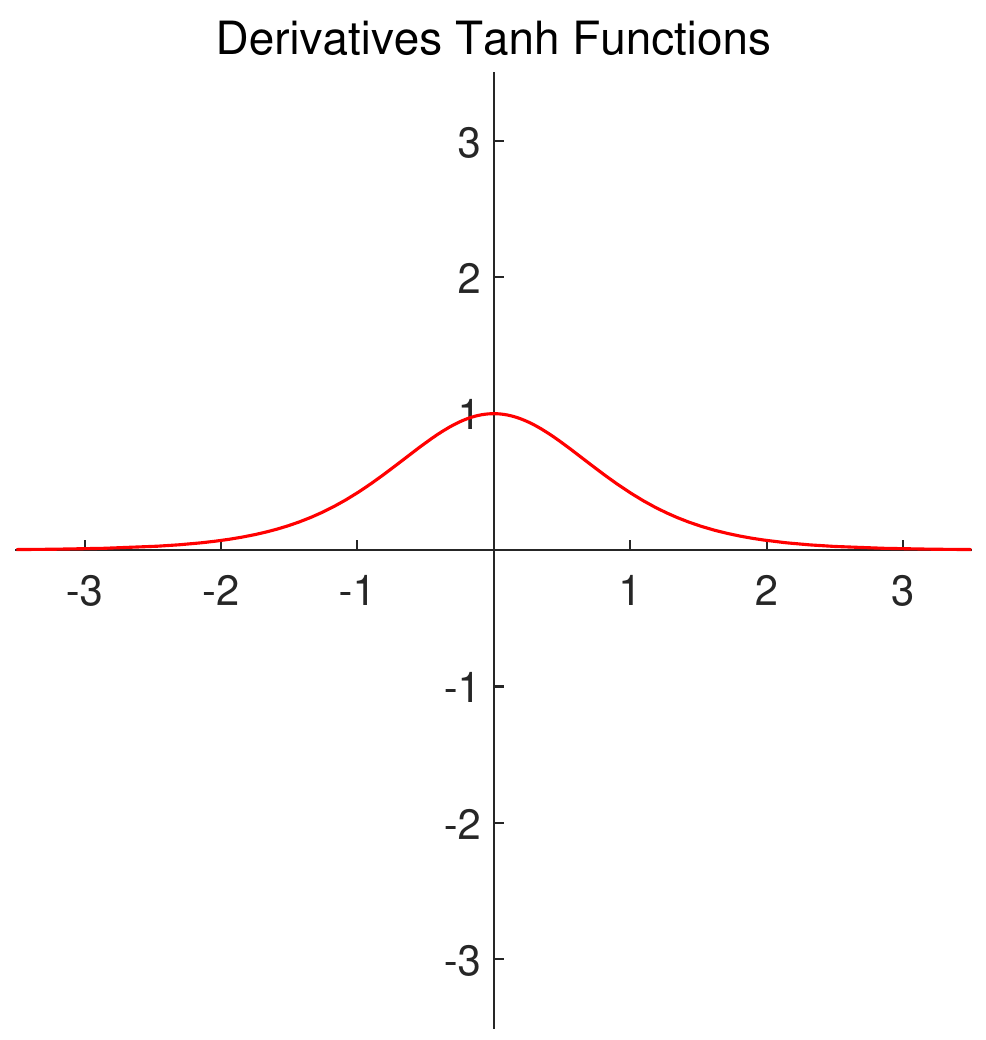}\\ % Content row 1
\midrule
Sigmoid & $f(x)=\frac { 1 }{ 1+{ e }^{ (-x) } } $ & \includegraphics[height=4cm,width=5cm]{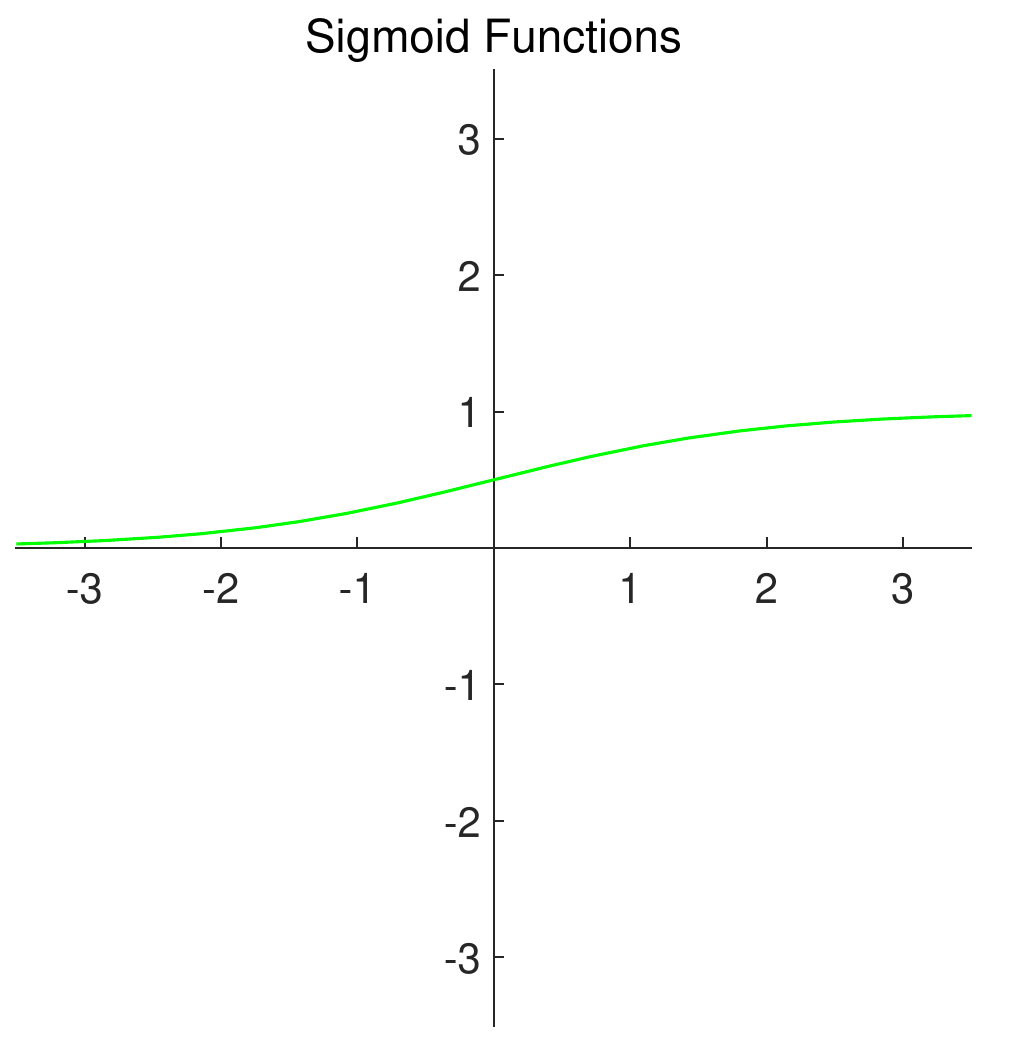}  & \includegraphics[height=4cm,width=5cm]{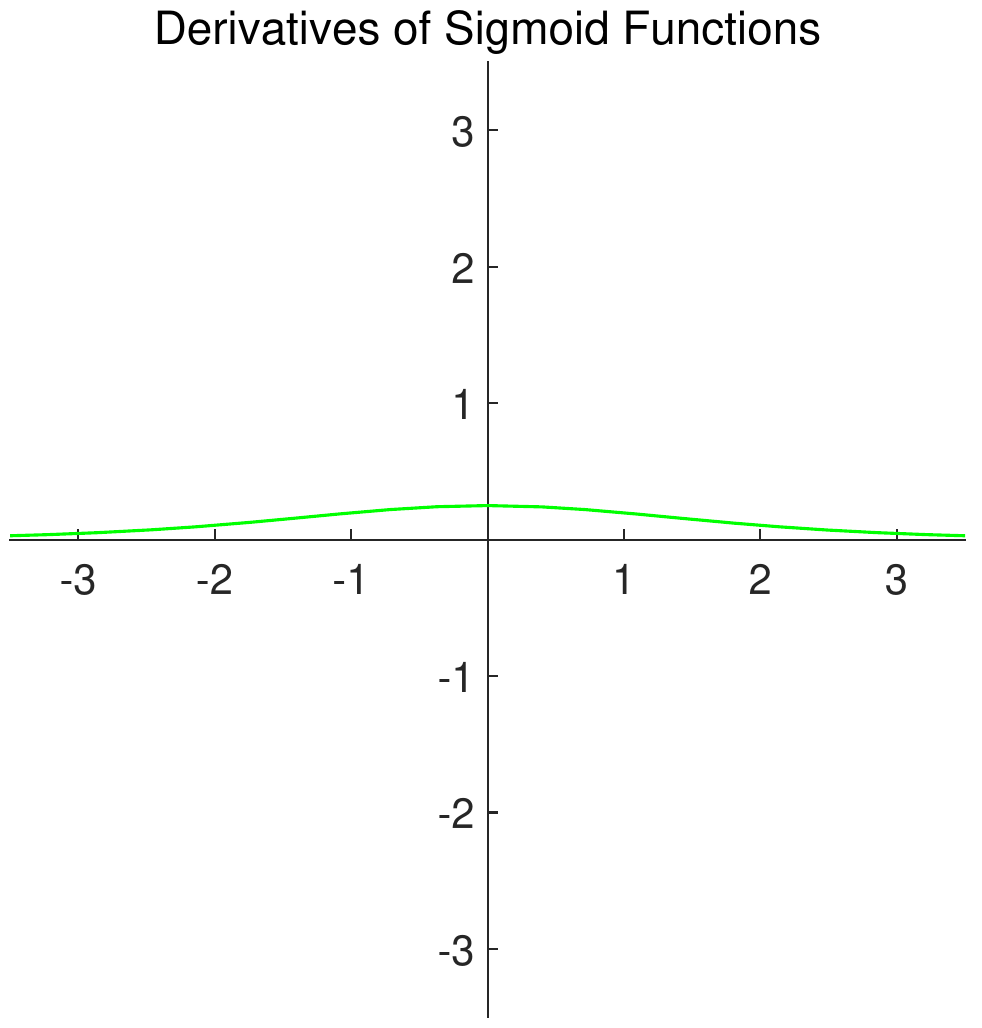} \\ % Content row 2
\midrule
Swish & $f(x)=\frac { x }{ 1+{ e }^{ (-\beta x) } }  $ & \includegraphics[height=4cm,width=5cm]{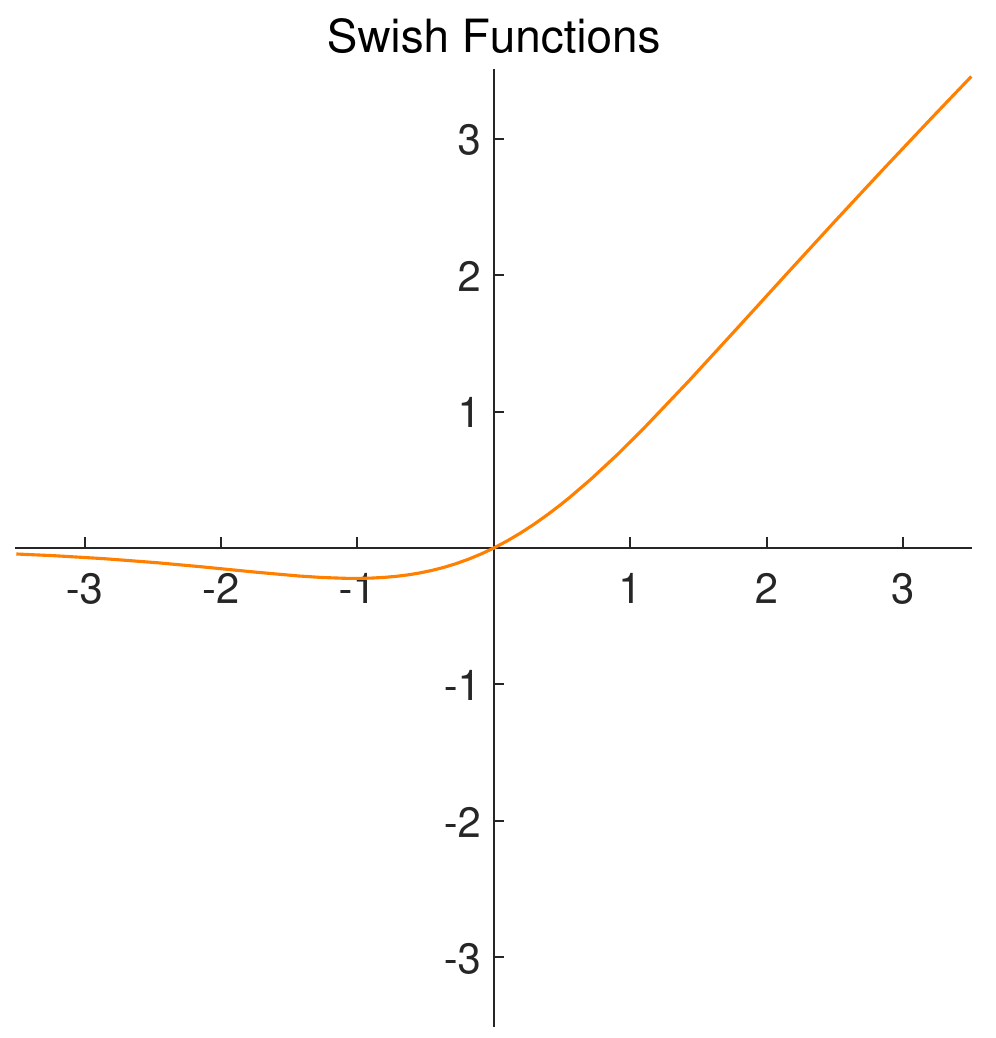}  & \includegraphics[height=4cm,width=5cm]{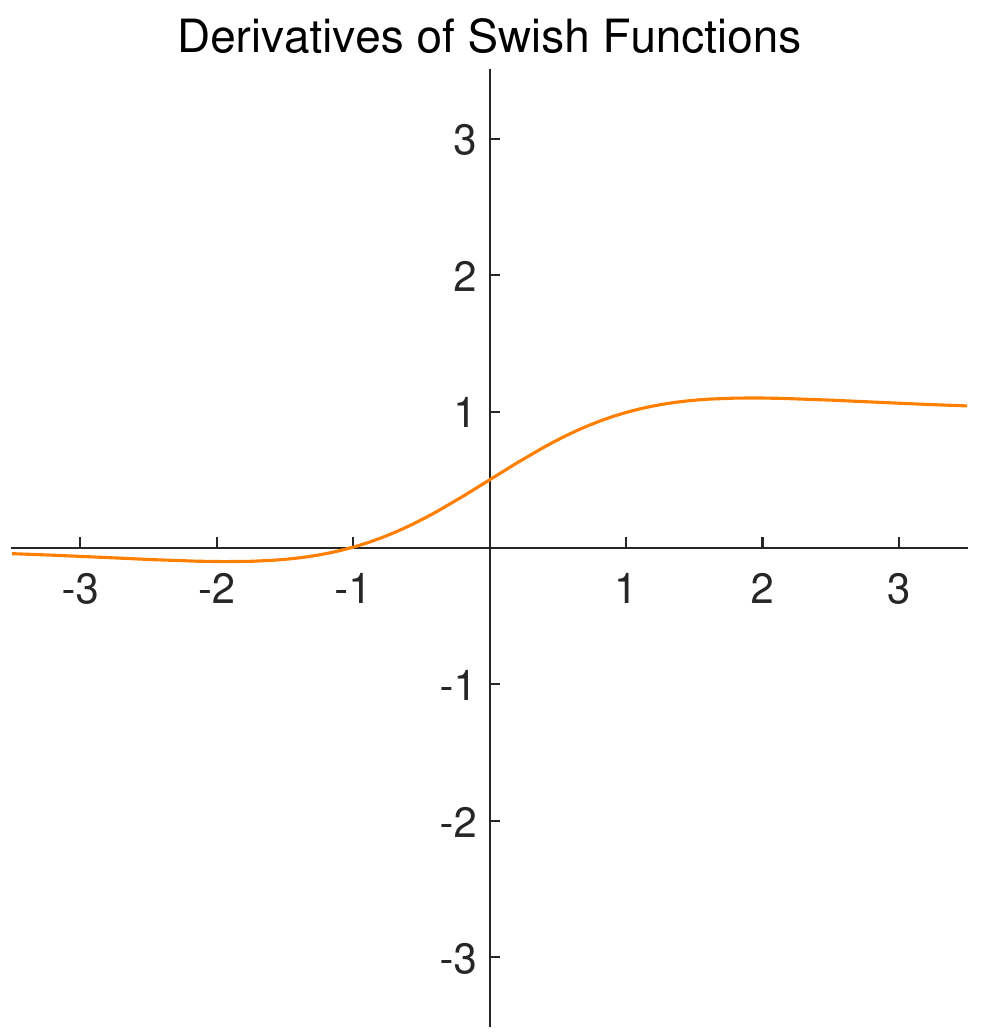} \\ % Content row 3
\midrule
LeCuns Tanh & $f(x)=1.7159 \times tanh(\frac { 2 }{ 3 } x)$ & \includegraphics[height=4cm,width=5cm]{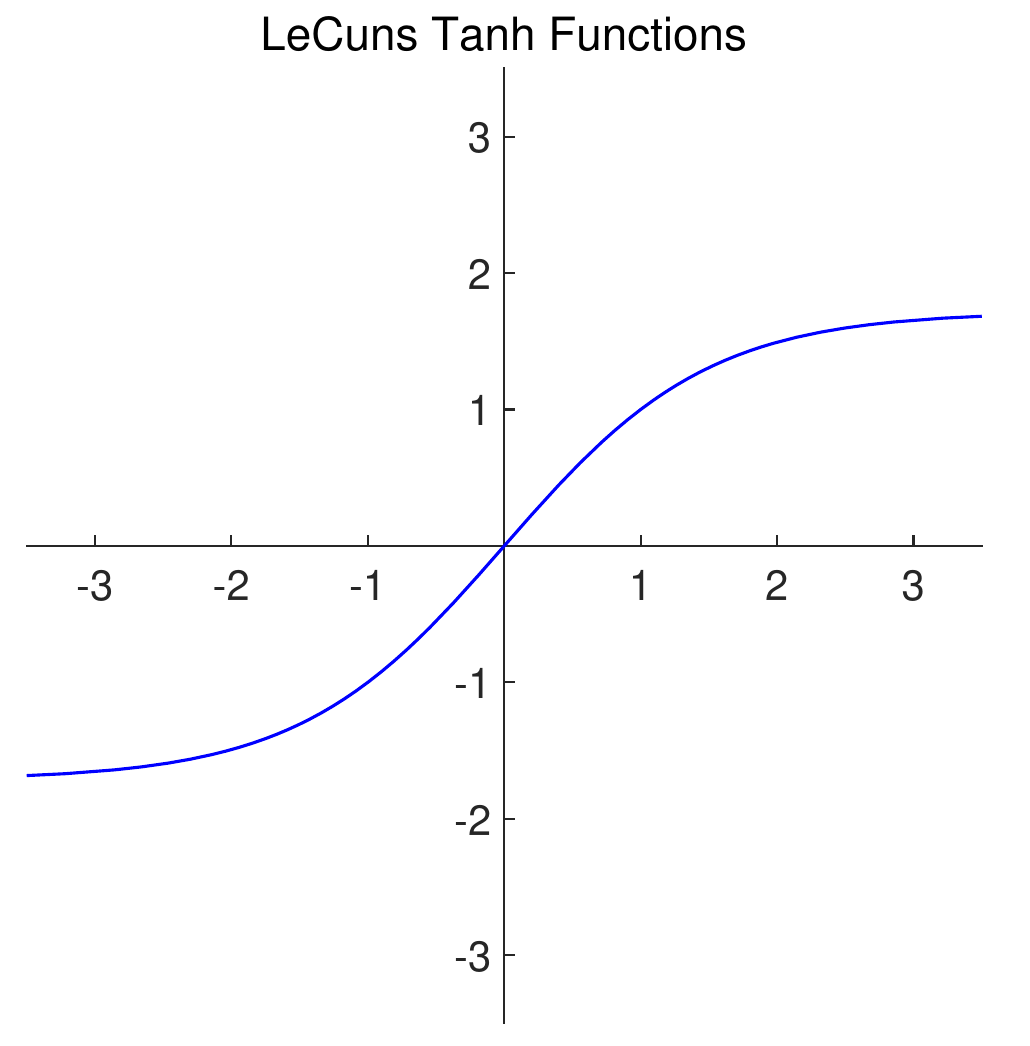} &\includegraphics[height=4cm,width=5cm]{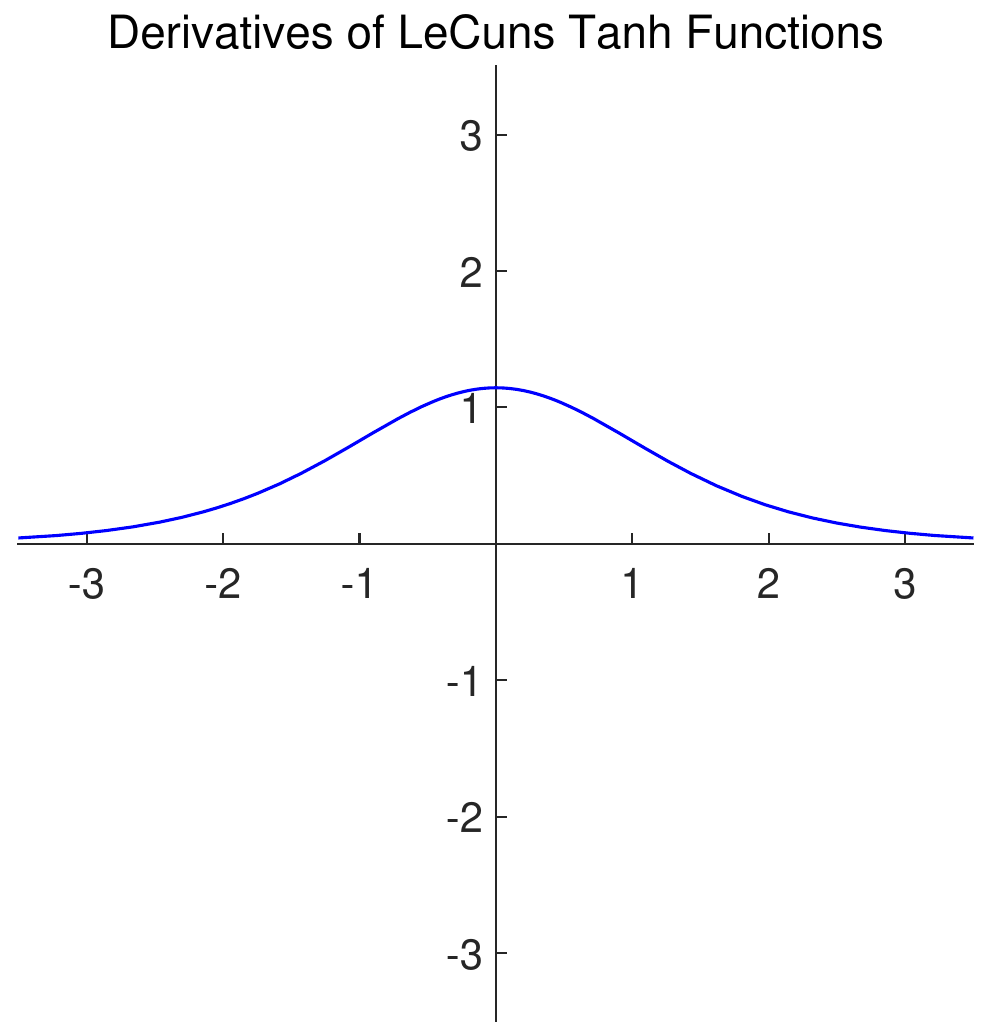} \\ % Content row 4
\midrule
Bipolar sigmoid&  $f(x)=\frac { { e }^{ x }-1 }{ { e }^{ x }+1 }  $  &\includegraphics[height=4cm,width=5cm]{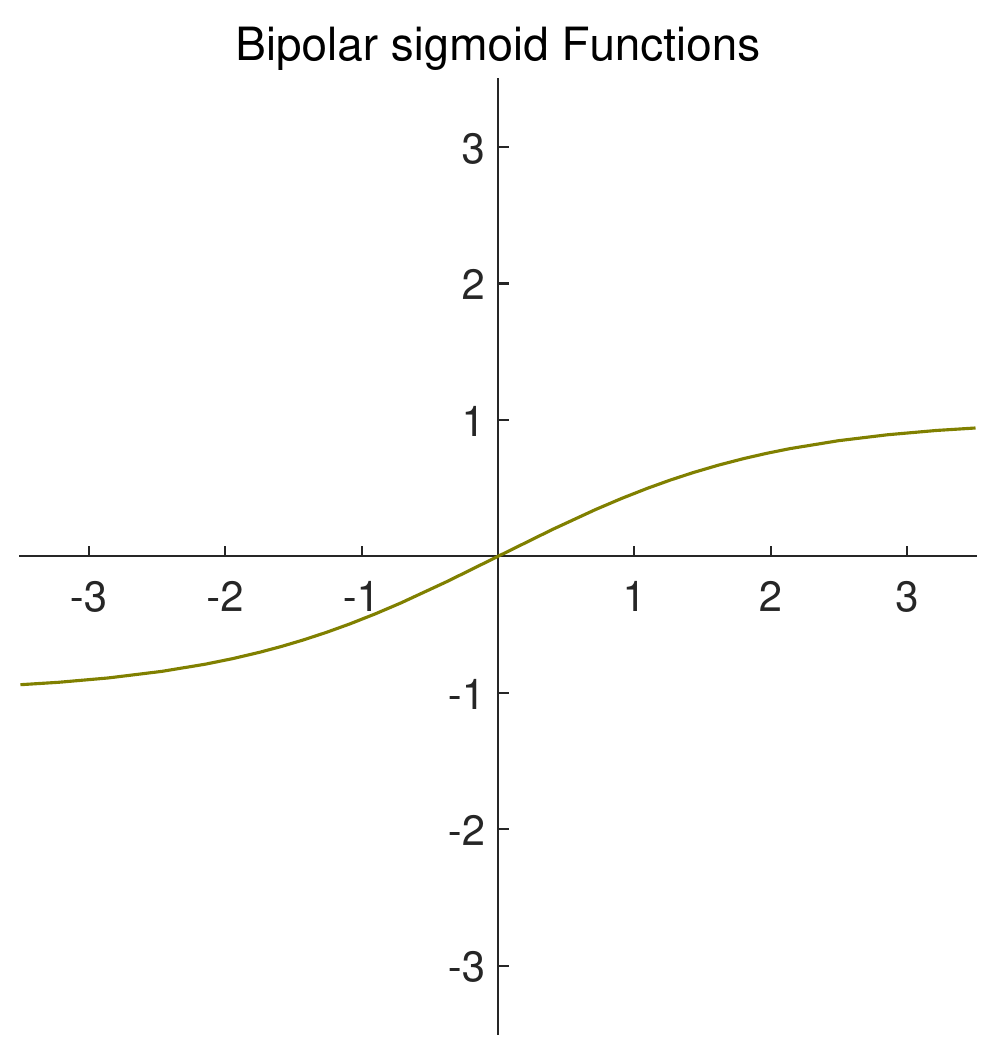} &\includegraphics[height=4cm,width=5cm]{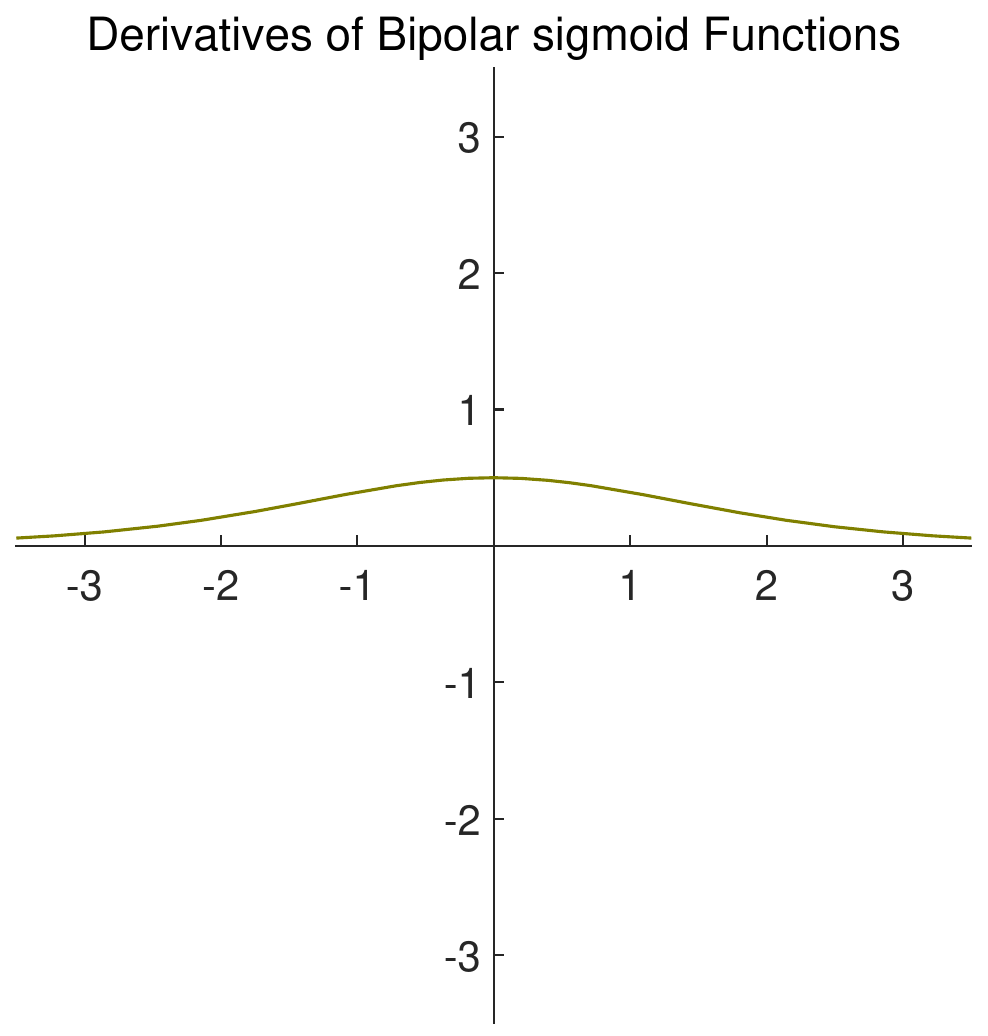} \\ %Content row 5
\midrule
Mish& $f(x)=x\times tanh(ln(1+{ e }^{ x })) $ &\includegraphics[height=4cm,width=5cm]{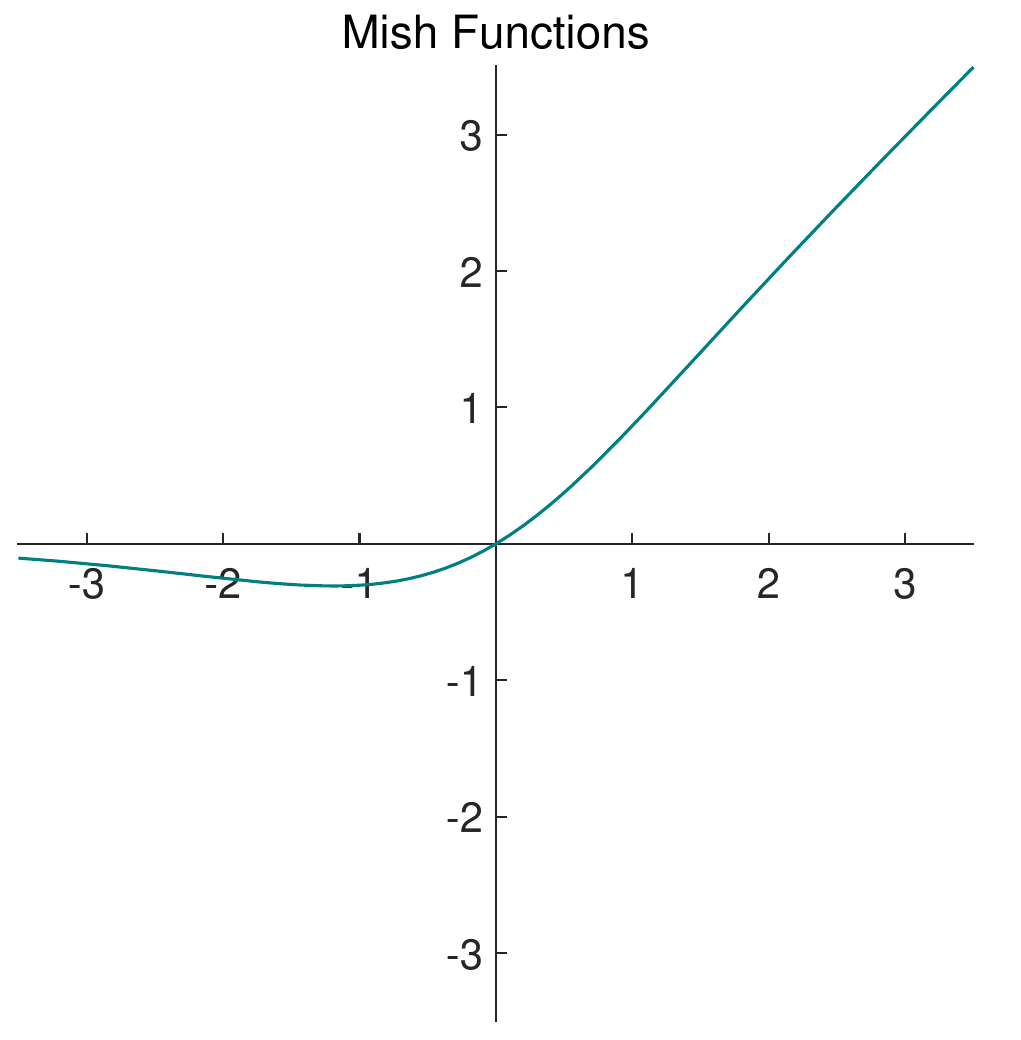} &\includegraphics[height=4cm,width=5cm]{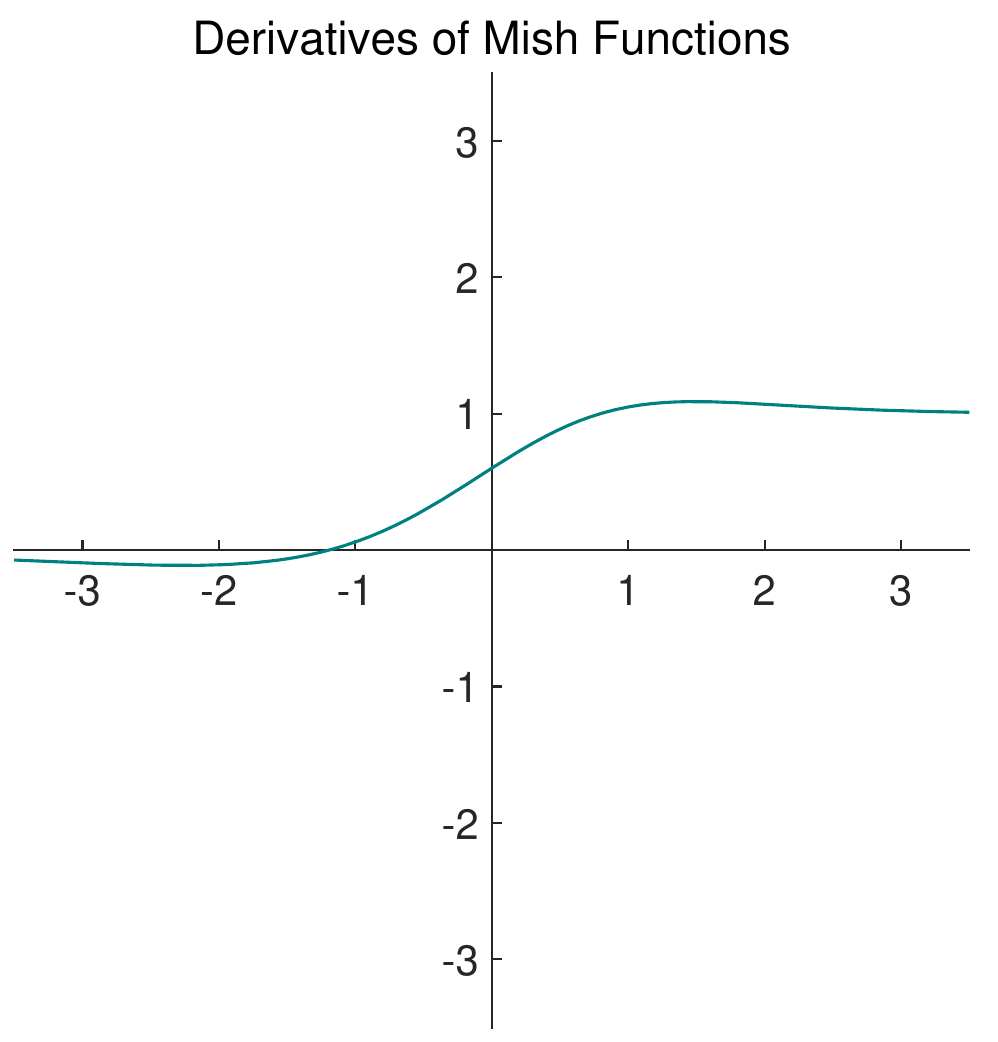} \\ % Content row 6
\midrule
Arctan &  $f(x)={ tan }^{ -1 }(x) $  &\includegraphics[height=4cm,width=5cm]{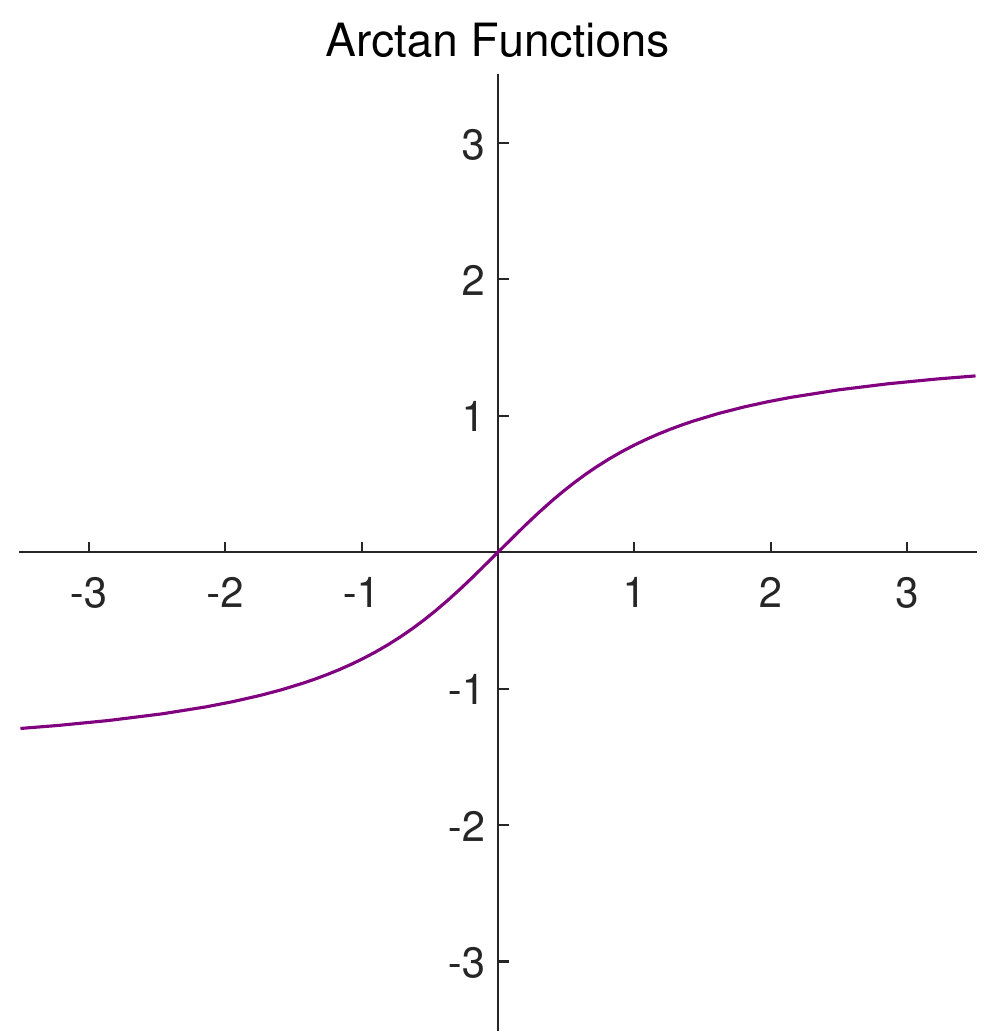} &\includegraphics[height=4cm,width=5cm]{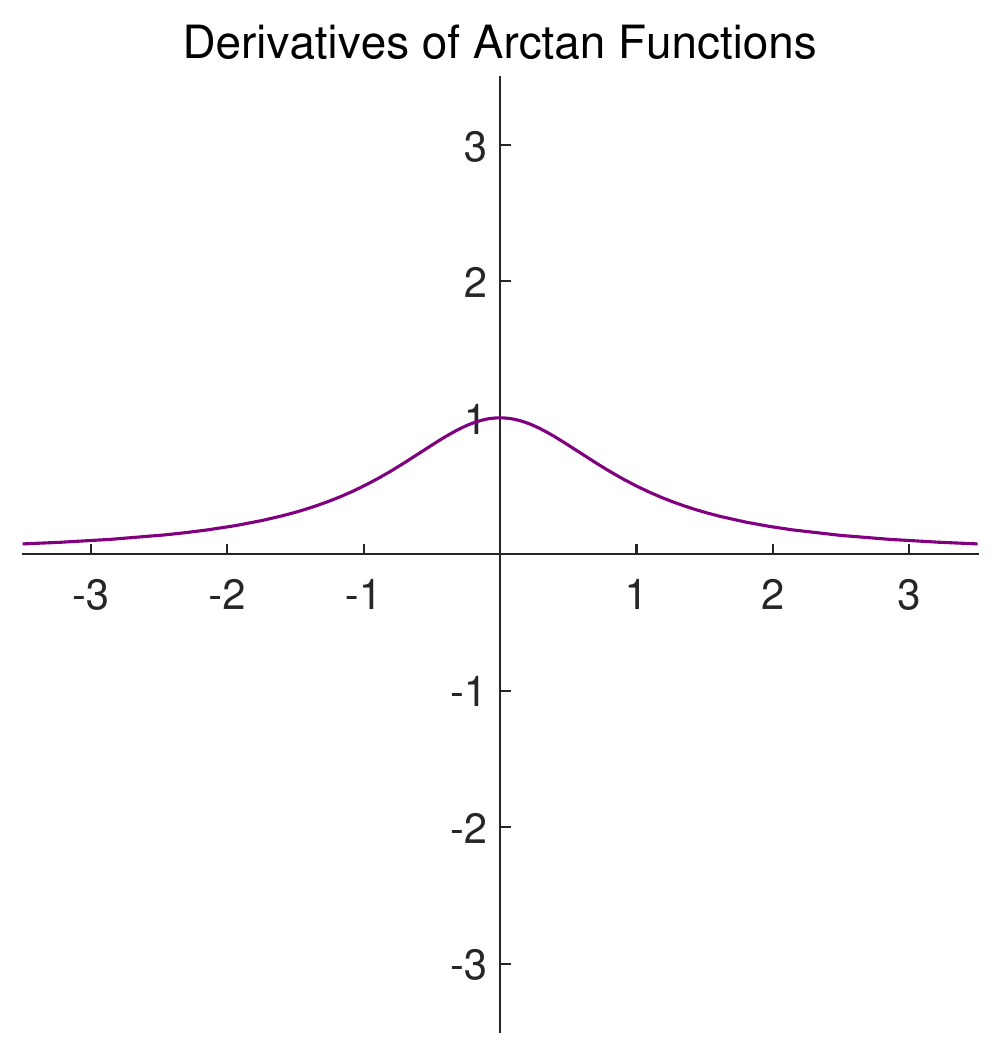}\\ % Content row 7
\midrule
Silu  & $f(x)=x\times sigmoid(x) $  &\includegraphics[height=4cm,width=5cm]{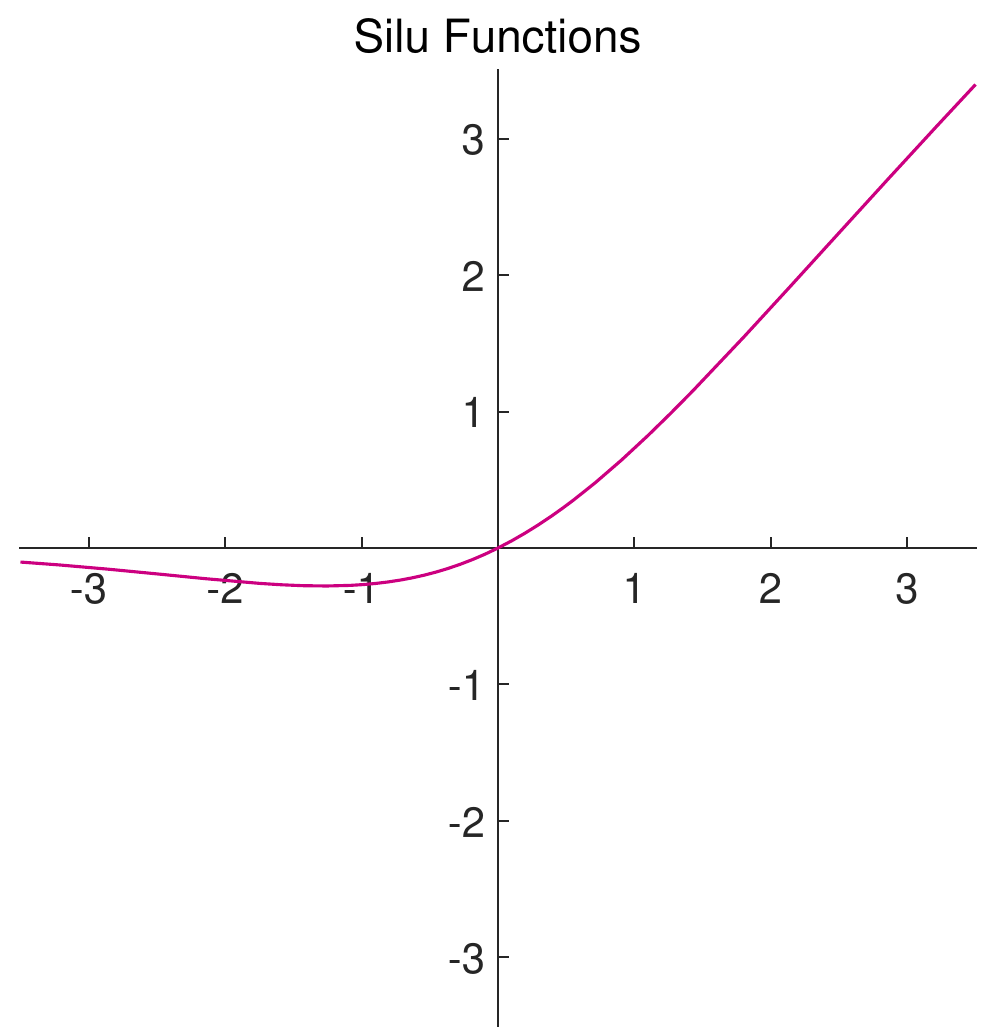} &\includegraphics[height=4cm,width=5cm]{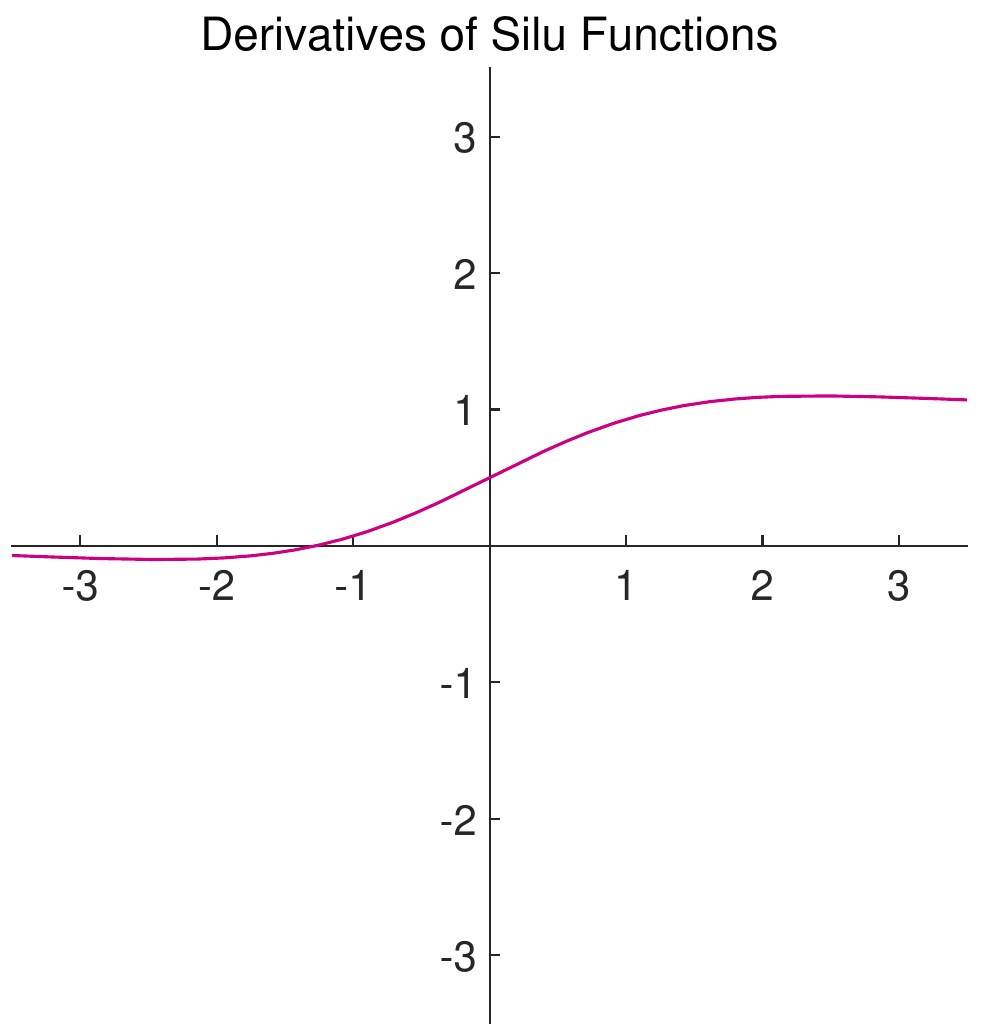}\\ % Content row 8
\bottomrule % Bottom horizontal line
\end{tabular}
}
\label{tab:Tableac} % A label for referencing this table elsewhere, references are used in text as \ref{label}
\end{table}

Various sampling methods are used to generate sequence of points within a cube. The purpose of sampling method is to generated training dataset for DCM and make the network better trained and proper sampling will help in case that the neural networks are only trained on fixed points and prevent a biased trained model, which may have a better prediction on random new data.
\begin{table}[!htb]
\captionsetup{width=0.9\columnwidth}
\caption{Sampling method}
\vspace{-0.1cm}
\centering 
\resizebox{0.8\columnwidth}{!}{%
\begin{tabular}{c  | m{5cm}<{\centering} || c | m{5cm}<{\centering}}
\toprule % Top horizontal line
\toprule % Top horizontal line
Sampling method & points figure&Sampling method & points figure\\
\midrule % In-table horizontal line
Latin hypercube & \includegraphics[height=4cm,width=4cm]{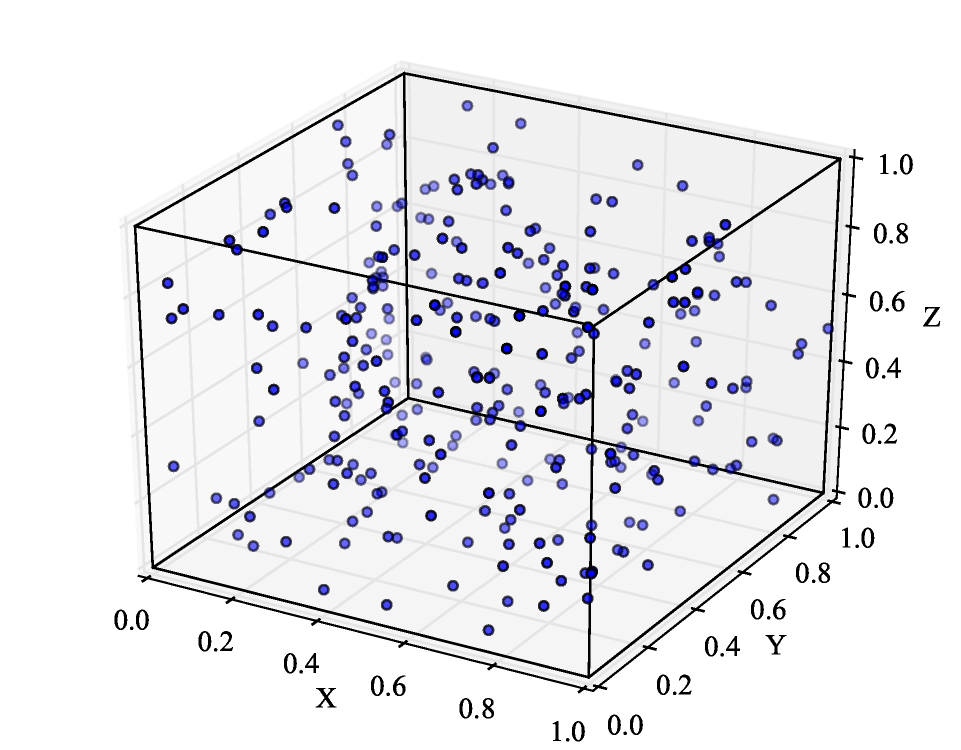}&Monte Carlo & \includegraphics[height=4cm,width=4cm]{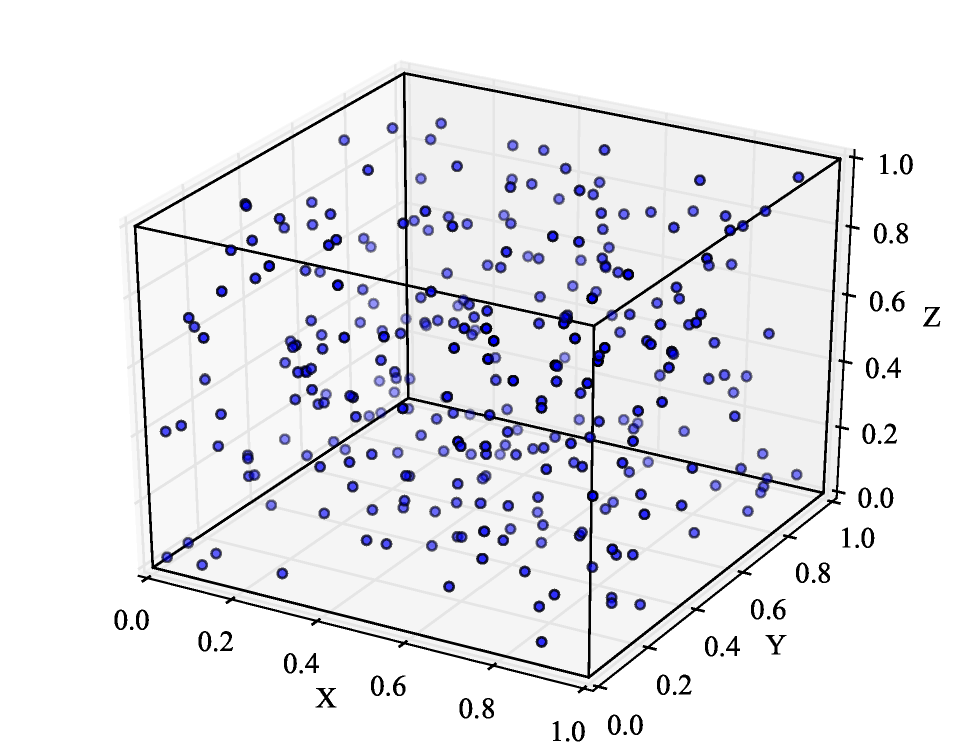}\\
\midrule
Random & \includegraphics[height=4cm,width=4cm]{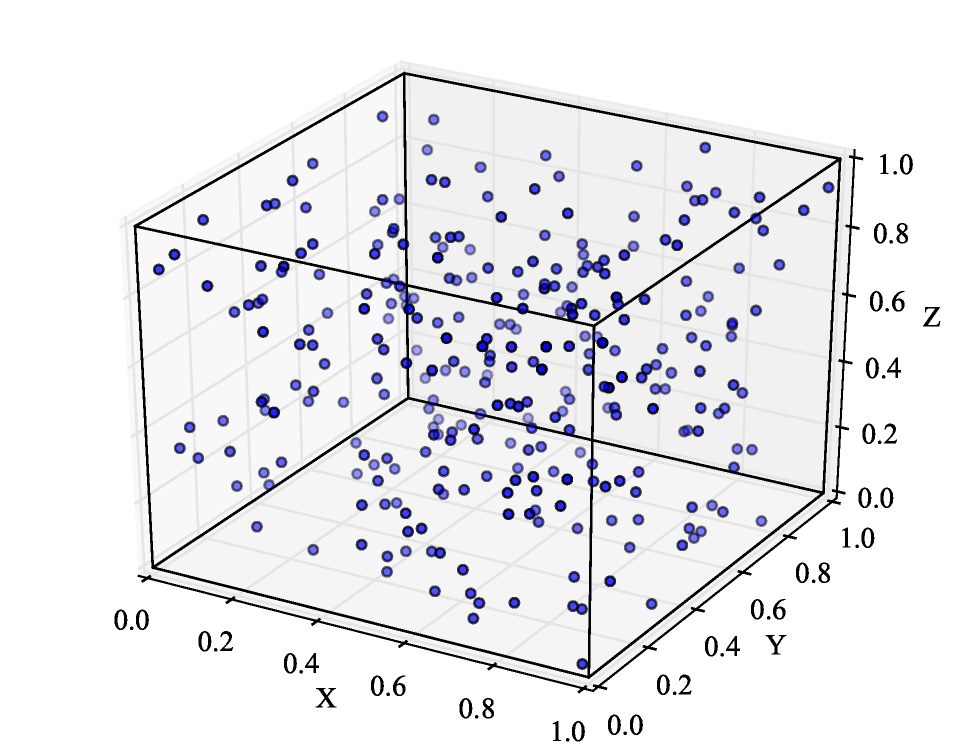}&Halton Sequences &\includegraphics[height=4cm,width=4cm]{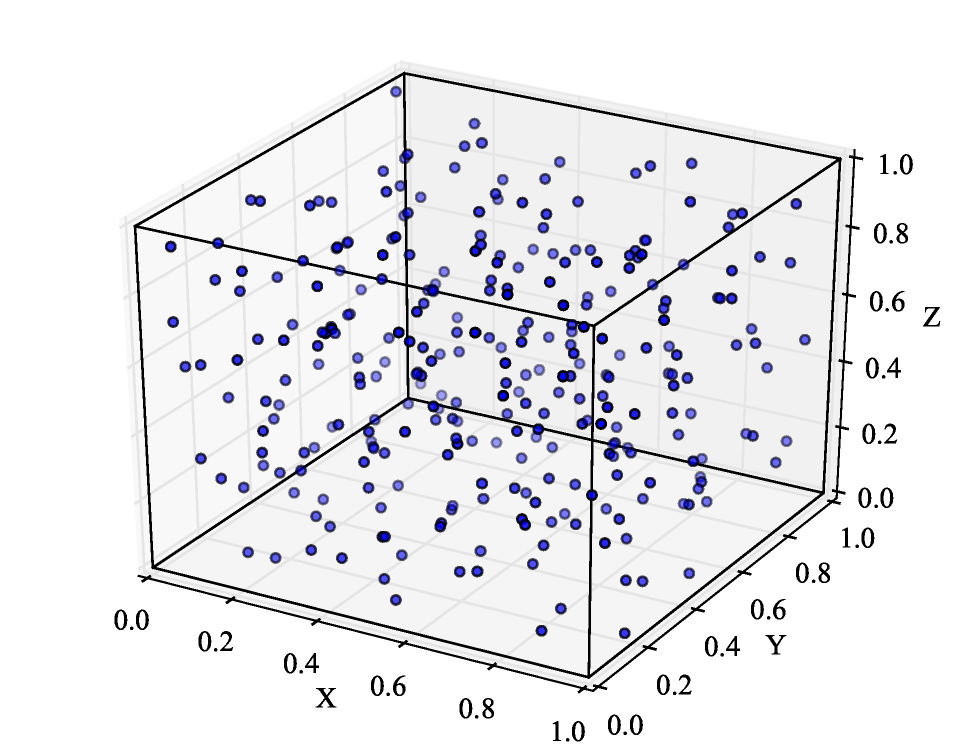} \\ % Content row 3
\midrule
Hammersley Sequence&\includegraphics[height=4cm,width=4cm]{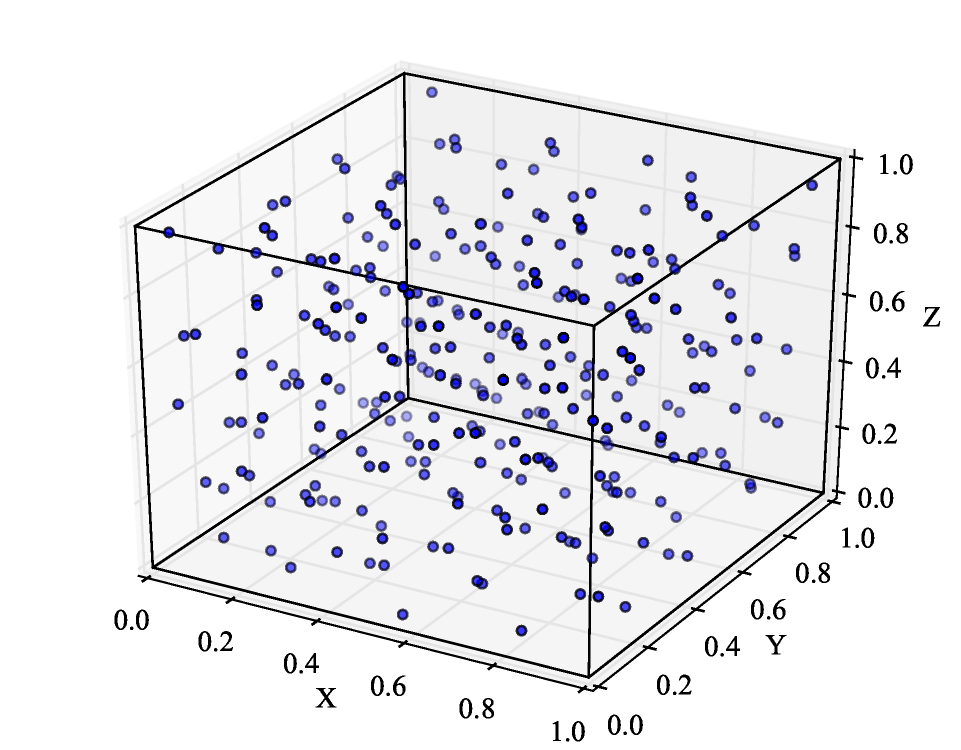}&Korobov Lattice&\includegraphics[height=4cm,width=4cm]{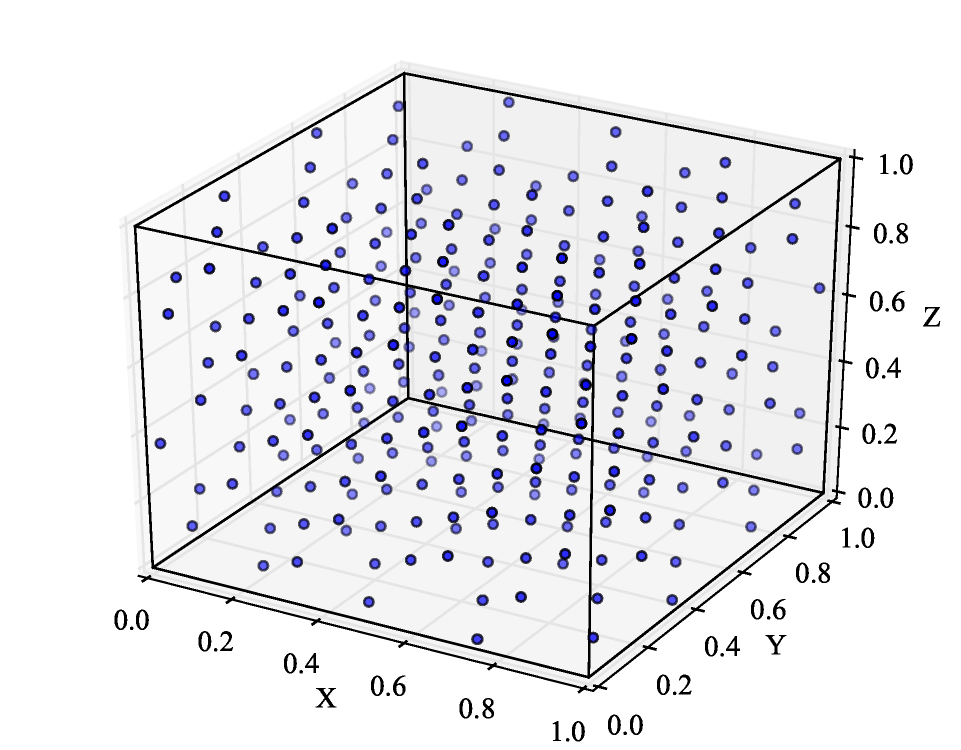} \\  
\midrule
Sobol Sequence & \includegraphics[height=4cm,width=4cm]{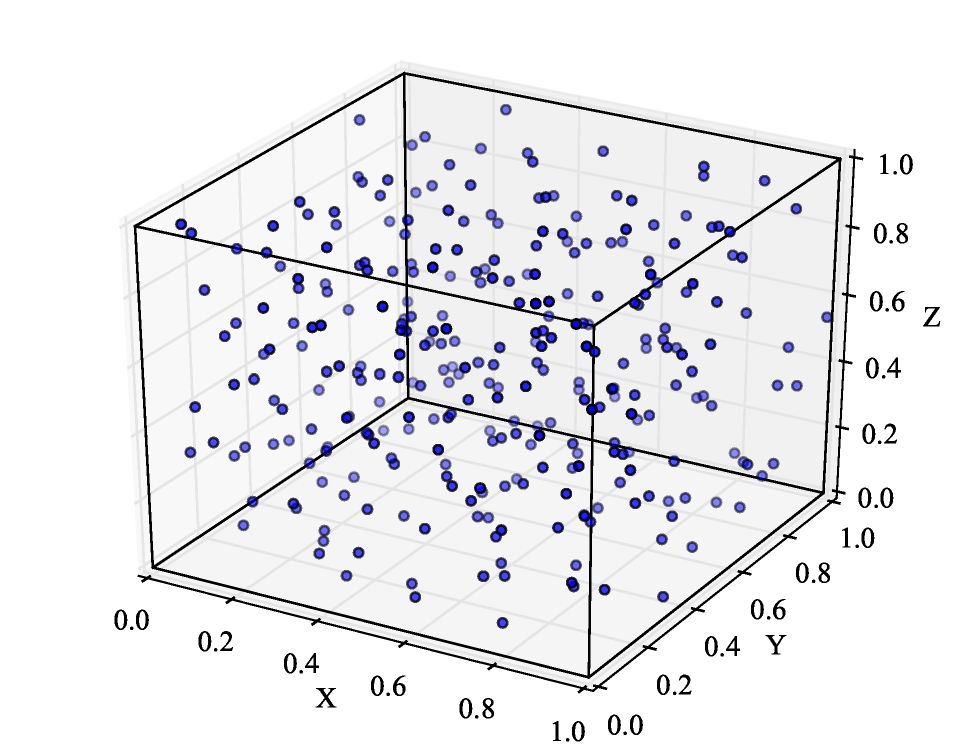}\\  
\bottomrule 
\end{tabular}
}
\label{tab:Tablesp}
\end{table}
\end{appendices}

\clearpage
\bibliography{ref.bib}

\begin{thebibliography}{10}

\bibitem{hinton2006reducing}
Geoffrey~E Hinton and Ruslan~R Salakhutdinov.
\newblock Reducing the dimensionality of data with neural networks.
\newblock {\em science}, 313(5786):504--507, 2006.

\bibitem{hinton2006fast}
Geoffrey~E Hinton, Simon Osindero, and Yee-Whye Teh.
\newblock A fast learning algorithm for deep belief nets.
\newblock {\em Neural computation}, 18(7):1527--1554, 2006.

\bibitem{bengio2007greedy}
Yoshua Bengio, Pascal Lamblin, Dan Popovici, and Hugo Larochelle.
\newblock Greedy layer-wise training of deep networks.
\newblock In {\em Advances in neural information processing systems}, pages
  153--160, 2007.

\bibitem{goodfellow2016deep}
Ian Goodfellow, Yoshua Bengio, and Aaron Courville.
\newblock {\em Deep learning}.
\newblock MIT press, 2016.

\bibitem{patterson2017deep}
Josh Patterson and Adam Gibson.
\newblock {\em Deep learning: A practitioner's approach}.
\newblock " O'Reilly Media, Inc.", 2017.

\bibitem{yang2018visually}
Liping Yang, Alan MacEachren, Prasenjit Mitra, and Teresa Onorati.
\newblock Visually-enabled active deep learning for (geo) text and image
  classification: a review.
\newblock {\em ISPRS International Journal of Geo-Information}, 7(2):65, 2018.

\bibitem{kermany2018identifying}
Daniel~S Kermany, Michael Goldbaum, Wenjia Cai, Carolina~CS Valentim, Huiying
  Liang, Sally~L Baxter, Alex McKeown, Ge~Yang, Xiaokang Wu, Fangbing Yan,
  et~al.
\newblock Identifying medical diagnoses and treatable diseases by image-based
  deep learning.
\newblock {\em Cell}, 172(5):1122--1131, 2018.

\bibitem{ouyang2015deepid}
Wanli Ouyang, Xiaogang Wang, Xingyu Zeng, Shi Qiu, Ping Luo, Yonglong Tian,
  Hongsheng Li, Shuo Yang, Zhe Wang, Chen-Change Loy, et~al.
\newblock Deepid-net: Deformable deep convolutional neural networks for object
  detection.
\newblock In {\em Proceedings of the IEEE conference on computer vision and
  pattern recognition}, pages 2403--2412, 2015.

\bibitem{zhao2019object}
Zhong-Qiu Zhao, Peng Zheng, Shoutao Xu, and Xindong Wu.
\newblock Object detection with deep learning: A review.
\newblock {\em IEEE transactions on neural networks and learning systems},
  2019.

\bibitem{amodei2016deep}
Dario Amodei, Sundaram Ananthanarayanan, Rishita Anubhai, Jingliang Bai, Eric
  Battenberg, Carl Case, Jared Casper, Bryan Catanzaro, Qiang Cheng, Guoliang
  Chen, et~al.
\newblock Deep speech 2: End-to-end speech recognition in english and mandarin.
\newblock In {\em International conference on machine learning}, pages
  173--182, 2016.

\bibitem{nassif2019speech}
Ali~Bou Nassif, Ismail Shahin, Imtinan Attili, Mohammad Azzeh, and Khaled
  Shaalan.
\newblock Speech recognition using deep neural networks: a systematic review.
\newblock {\em IEEE Access}, 2019.

\bibitem{yue2018deep}
Tianwei Yue and Haohan Wang.
\newblock Deep learning for genomics: A concise overview.
\newblock {\em arXiv preprint arXiv:1802.00810}, 2018.

\bibitem{ching2018opportunities}
Travers Ching, Daniel~S Himmelstein, Brett~K Beaulieu-Jones, Alexandr~A
  Kalinin, Brian~T Do, Gregory~P Way, Enrico Ferrero, Paul-Michael Agapow,
  Michael Zietz, Michael~M Hoffman, et~al.
\newblock Opportunities and obstacles for deep learning in biology and
  medicine.
\newblock {\em Journal of The Royal Society Interface}, 15(141):20170387, 2018.

\bibitem{heaton2017deep}
JB~Heaton, NG~Polson, and Jan~Hendrik Witte.
\newblock Deep learning for finance: deep portfolios.
\newblock {\em Applied Stochastic Models in Business and Industry},
  33(1):3--12, 2017.

\bibitem{fischer2018deep}
Thomas Fischer and Christopher Krauss.
\newblock Deep learning with long short-term memory networks for financial
  market predictions.
\newblock {\em European Journal of Operational Research}, 270(2):654--669,
  2018.

\bibitem{gyryamachine}
Vitaliy Gyrya, Mikhail~Jurievich Shashkov, Alexei~N Skurikhin, and Svetlana
  Tokareva.
\newblock Machine learning approaches for the solution of the riemann problem
  in fluid dynamics: a case study.

\bibitem{mcculloch1943logical}
Warren~S McCulloch and Walter Pitts.
\newblock A logical calculus of the ideas immanent in nervous activity.
\newblock {\em The bulletin of mathematical biophysics}, 5(4):115--133, 1943.

\bibitem{dias2004artificial}
Fernando~Morgado Dias, Ana Antunes, and Alexandre~Manuel Mota.
\newblock Artificial neural networks: a review of commercial hardware.
\newblock {\em Engineering Applications of Artificial Intelligence},
  17(8):945--952, 2004.

\bibitem{lagaris1998artificial}
Isaac~E Lagaris, Aristidis Likas, and Dimitrios~I Fotiadis.
\newblock Artificial neural networks for solving ordinary and partial
  differential equations.
\newblock {\em IEEE transactions on neural networks}, 9(5):987--1000, 1998.

\bibitem{lagaris2000neural}
Isaac~E Lagaris, Aristidis~C Likas, and Dimitris~G Papageorgiou.
\newblock Neural-network methods for boundary value problems with irregular
  boundaries.
\newblock {\em IEEE Transactions on Neural Networks}, 11(5):1041--1049, 2000.

\bibitem{mcfall2009artificial}
Kevin~Stanley McFall and James~Robert Mahan.
\newblock Artificial neural network method for solution of boundary value
  problems with exact satisfaction of arbitrary boundary conditions.
\newblock {\em IEEE Transactions on Neural Networks}, 20(8):1221--1233, 2009.

\bibitem{FUNAHASHI1989183}
Ken-Ichi Funahashi.
\newblock On the approximate realization of continuous mappings by neural
  networks.
\newblock {\em Neural Networks}, 2(3):183 -- 192, 1989.

\bibitem{HORNIK1989359}
Kurt Hornik, Maxwell Stinchcombe, and Halbert White.
\newblock Multilayer feedforward networks are universal approximators.
\newblock {\em Neural Networks}, 2(5):359 -- 366, 1989.

\bibitem{mhaskar2016deep}
Hrushikesh~N Mhaskar and Tomaso Poggio.
\newblock Deep vs. shallow networks: An approximation theory perspective.
\newblock {\em Analysis and Applications}, 14(06):829--848, 2016.

\bibitem{weinan2017deep}
E~Weinan, Jiequn Han, and Arnulf Jentzen.
\newblock Deep learning-based numerical methods for high-dimensional parabolic
  partial differential equations and backward stochastic differential
  equations.
\newblock {\em Communications in Mathematics and Statistics}, 5(4):349--380,
  2017.

\bibitem{han2018solving}
Jiequn Han, Arnulf Jentzen, and E~Weinan.
\newblock Solving high-dimensional partial differential equations using deep
  learning.
\newblock {\em Proceedings of the National Academy of Sciences},
  115(34):8505--8510, 2018.

\bibitem{RAISSI2019686}
M.~Raissi, P.~Perdikaris, and G.E. Karniadakis.
\newblock Physics-informed neural networks: A deep learning framework for
  solving forward and inverse problems involving nonlinear partial differential
  equations.
\newblock {\em Journal of Computational Physics}, 378:686 -- 707, 2019.

\bibitem{Beck_2019}
Christian Beck, Weinan E, and Arnulf Jentzen.
\newblock Machine learning approximation algorithms for high-dimensional fully
  nonlinear partial differential equations and second-order backward stochastic
  differential equations.
\newblock {\em Journal of Nonlinear Science}, Jan 2019.

\bibitem{sirignano2018dgm}
Justin Sirignano and Konstantinos Spiliopoulos.
\newblock Dgm: A deep learning algorithm for solving partial differential
  equations.
\newblock {\em Journal of Computational Physics}, 375:1339--1364, 2018.

\bibitem{karniadakis2021physics}
George~Em Karniadakis, Ioannis~G Kevrekidis, Lu~Lu, Paris Perdikaris, Sifan
  Wang, and Liu Yang.
\newblock Physics-informed machine learning.
\newblock {\em Nature Reviews Physics}, 3(6):422--440, 2021.

\bibitem{anitescu2019artificial}
Cosmin Anitescu, Elena Atroshchenko, Naif Alajlan, and Timon Rabczuk.
\newblock Artificial neural network methods for the solution of second order
  boundary value problems.
\newblock {\em Computers, Materials \& Continua}, 59(1):345--359, 2019.

\bibitem{guo2019deep}
Hongwei Guo, Xiaoying Zhuang, and Timon Rabczuk.
\newblock A deep collocation method for the bending analysis of kirchhoff
  plate.
\newblock {\em CMC-COMPUTERS MATERIALS \& CONTINUA}, 59(2):433--456, 2019.

\bibitem{guo2022stochastic}
Hongwei Guo, Xiaoying Zhuang, Pengwan Chen, Naif Alajlan, and Timon Rabczuk.
\newblock Stochastic deep collocation method based on neural architecture
  search and transfer learning for heterogeneous porous media.
\newblock {\em Engineering with Computers}, pages 1--26, 2022.

\bibitem{samaniego2020energy}
Esteban Samaniego, Cosmin Anitescu, Somdatta Goswami, Vien~Minh Nguyen-Thanh,
  Hongwei Guo, Khader Hamdia, X~Zhuang, and T~Rabczuk.
\newblock An energy approach to the solution of partial differential equations
  in computational mechanics via machine learning: Concepts, implementation and
  applications.
\newblock {\em Computer Methods in Applied Mechanics and Engineering},
  362:112790, 2020.

\bibitem{nguyen2020deep}
Vien~Minh Nguyen-Thanh, Xiaoying Zhuang, and Timon Rabczuk.
\newblock A deep energy method for finite deformation hyperelasticity.
\newblock {\em European Journal of Mechanics-A/Solids}, 80:103874, 2020.

\bibitem{goswami2020transfer}
Somdatta Goswami, Cosmin Anitescu, Souvik Chakraborty, and Timon Rabczuk.
\newblock Transfer learning enhanced physics informed neural network for
  phase-field modeling of fracture.
\newblock {\em Theoretical and Applied Fracture Mechanics}, 106:102447, 2020.

\bibitem{zhuang2021deep}
Xiaoying Zhuang, Hongwei Guo, Naif Alajlan, Hehua Zhu, and Timon Rabczuk.
\newblock Deep autoencoder based energy method for the bending, vibration, and
  buckling analysis of kirchhoff plates with transfer learning.
\newblock {\em European Journal of Mechanics-A/Solids}, 87:104225, 2021.

\bibitem{qu2015solutions}
Wenzhen Qu, Wen Chen, and Zhuojia Fu.
\newblock Solutions of 2d and 3d non-homogeneous potential problems by using a
  boundary element-collocation method.
\newblock {\em Engineering Analysis with Boundary Elements}, 60:2--9, 2015.

\bibitem{alves2005new}
CJS Alves and CS~Chen.
\newblock A new method of fundamental solutions applied to nonhomogeneous
  elliptic problems.
\newblock {\em Advances in Computational Mathematics}, 23(1-2):125--142, 2005.

\bibitem{paris1997boundary}
Federico Paris and Jose Canas.
\newblock {\em Boundary element method: fundamentals and applications},
  volume~1.
\newblock Oxford University Press, USA, 1997.

\bibitem{dhingraactivation}
Anshul Dhingra.
\newblock Activation functions in neural networks.

\bibitem{misra2019mish}
Diganta Misra.
\newblock Mish: A self regularized non-monotonic neural activation function.
\newblock {\em arXiv preprint arXiv:1908.08681}, 2019.

\bibitem{zhang2018efficient}
Huan Zhang, Tsui-Wei Weng, Pin-Yu Chen, Cho-Jui Hsieh, and Luca Daniel.
\newblock Efficient neural network robustness certification with general
  activation functions.
\newblock In {\em Advances in neural information processing systems}, pages
  4939--4948, 2018.

\bibitem{hornik1991approximation}
Kurt Hornik.
\newblock Approximation capabilities of multilayer feedforward networks.
\newblock {\em Neural networks}, 4(2):251--257, 1991.

\bibitem{raissi2017physics}
Maziar Raissi, Paris Perdikaris, and George~Em Karniadakis.
\newblock Physics informed deep learning (part i): Data-driven solutions of
  nonlinear partial differential equations.
\newblock {\em arXiv preprint arXiv:1711.10561}, 2017.

\bibitem{rafajlowicz2006halton}
Ewaryst Rafaj{\l}owicz and Rainer Schwabe.
\newblock Halton and hammersley sequences in multivariate nonparametric
  regression.
\newblock {\em Statistics \& Probability Letters}, 76(8):803--812, 2006.

\bibitem{wang2004korobov}
Xiaoqun Wang, Ian~H Sloan, and Josef Dick.
\newblock On korobov lattice rules in weighted spaces.
\newblock {\em SIAM journal on numerical analysis}, 42(4):1760--1779, 2004.

\bibitem{dick2007construction}
Josef Dick, Friedrich Pillichshammer, and Benjamin~J Waterhouse.
\newblock The construction of good extensible korobov rules.
\newblock {\em Computing}, 79(1):79--91, 2007.

\bibitem{shields2016generalization}
Michael~D Shields and Jiaxin Zhang.
\newblock The generalization of latin hypercube sampling.
\newblock {\em Reliability Engineering \& System Safety}, 148:96--108, 2016.

\bibitem{shapiro2003monte}
Alexander Shapiro.
\newblock Monte carlo sampling methods.
\newblock {\em Handbooks in operations research and management science},
  10:353--425, 2003.

\bibitem{iooss2015review}
Bertrand Iooss and Paul Lema{\^\i}tre.
\newblock A review on global sensitivity analysis methods.
\newblock In {\em Uncertainty management in simulation-optimization of complex
  systems}, pages 101--122. Springer, 2015.

\bibitem{sobol2001global}
Ilya~M Sobol.
\newblock Global sensitivity indices for nonlinear mathematical models and
  their monte carlo estimates.
\newblock {\em Mathematics and computers in simulation}, 55(1-3):271--280,
  2001.

\bibitem{cukier1973study}
RI~Cukier, CM~Fortuin, Kurt~E Shuler, AG~Petschek, and JH~Schaibly.
\newblock Study of the sensitivity of coupled reaction systems to uncertainties
  in rate coefficients. i theory.
\newblock {\em The Journal of chemical physics}, 59(8):3873--3878, 1973.

\bibitem{saltelli1999quantitative}
Andrea Saltelli, Stefano Tarantola, and KP-S Chan.
\newblock A quantitative model-independent method for global sensitivity
  analysis of model output.
\newblock {\em Technometrics}, 41(1):39--56, 1999.

\bibitem{herman2013method}
JD~Herman, JB~Kollat, PM~Reed, and T~Wagener.
\newblock Method of morris effectively reduces the computational demands of
  global sensitivity analysis for distributed watershed models.
\newblock {\em Hydrology \& Earth System Sciences Discussions}, 10(4), 2013.

\bibitem{morris1991factorial}
Max~D Morris.
\newblock Factorial sampling plans for preliminary computational experiments.
\newblock {\em Technometrics}, 33(2):161--174, 1991.

\bibitem{sanchez2014application}
D~Garcia Sanchez, Bruno Lacarri{\`e}re, Marjorie Musy, and Bernard Bourges.
\newblock Application of sensitivity analysis in building energy simulations:
  Combining first-and second-order elementary effects methods.
\newblock {\em Energy and Buildings}, 68:741--750, 2014.

\bibitem{sutradhar2004simple}
Alok Sutradhar and Glaucio~H Paulino.
\newblock A simple boundary element method for problems of potential in
  non-homogeneous media.
\newblock {\em International Journal for Numerical Methods in Engineering},
  60(13):2203--2230, 2004.

\end{thebibliography}
%
% --------------------------------------------------------------
%     You don't have to mess with anything below this line.
% --------------------------------------------------------------
\end{document}